\documentclass{article}
\usepackage{amsmath, amssymb, bm, amsthm, bbm, color, booktabs, hyperref, url, authblk}
\usepackage{graphicx}
\graphicspath{{./images/}}

\usepackage{geometry}
\usepackage{color}

\theoremstyle{plain}
\newtheorem{theorem}{Theorem}[section]
\newtheorem{proposition}[theorem]{Proposition}
\newtheorem{lemma}[theorem]{Lemma}
\newtheorem{corollary}[theorem]{Corollary}
\theoremstyle{definition}

\newtheorem{assumption}[theorem]{Assumption}

\theoremstyle{remark}

\begin{document}

\title{Robust and Sparse Estimation of Unbounded Density Ratio under Heavy Contamination}
\author[1]{Ryosuke Nagumo}
\author[2,3]{Hironori Fujisawa}
\affil[1]{Panasonic Holdings Corporation}
\affil[2]{The Institute of Statistical Mathematics}
\affil[3]{The Graduate University for Advanced Studies (SOKENDAI)}
\date{}
\maketitle

\begin{abstract}
    We examine the non-asymptotic properties of robust density ratio estimation (DRE) in contaminated settings.
    Weighted DRE is the most promising among existing methods, exhibiting doubly strong robustness from an asymptotic perspective.
    This study demonstrates that Weighted DRE achieves sparse consistency even under heavy contamination within a non-asymptotic framework.
    This method addresses two significant challenges in density ratio estimation and robust estimation.
    For density ratio estimation, we provide the non-asymptotic properties of estimating unbounded density ratios under the assumption that the weighted density ratio function is bounded.
    For robust estimation, we introduce a non-asymptotic framework for doubly strong robustness under heavy contamination, assuming that at least one of the following conditions holds: (i) contamination ratios are small, and (ii) outliers have small weighted values.
    This work provides the first non-asymptotic analysis of strong robustness under heavy contamination.
\end{abstract}

\section{Introduction}
\label{introduction}

Density ratio estimation (DRE) is a statistical method to directly estimate the ratio of two probability density functions without estimating each function \cite{Nguyen_nips2007,sugiyama_dre2012}.
This method has wide-ranging applications, including change detection \cite{kawahara_icdm2009,songliu_neuralnet2013}, outlier detection \cite{hido_kis2011}, covariate shift adaptation \cite{shimodaira_2000,zhang_2023}, and two-sample tests \cite{wornowizki_comp_stat_2016,byol_jrssb2021}.
Its parametric formulation, often referred to as the differential graphical model \cite{songliu_neuralcomp2014,songliu_behaviormetrika2017,songliu_annals_of_stats2017}, is also used in various fields such as protein and genetic interaction mapping \cite{ideker_MolSysBio2012} and brain imaging \cite{na_biometrika2020}.

Despite its versatility, density ratio estimation is not robust when data exist in a region with small density function values \cite{smola_icais2009,yamada_NeuralComp2011,sugiyama_2012,songliu_neurips2017}.
Consider the density ratio $r(x)=p(x)/q(x)$, where we call $p(x)$ a reference distribution and $q(x)$ a target distribution.
When data exist in regions where  $p(x)$ or $q(x)$ is small, the estimation of $r(x)$ often becomes unstable \cite{rhodes2020,kato2021,choi_2021,choi_2022,srivastava_2023}.
This issue is prevalent when outliers contaminate the main distributions \cite{maronna_martin_yohai_robust_statistics,hampel2011robust}.

Robust DRE methods, including Weighted DRE and $\gamma$-DRE, have been proposed to address contamination scenarios \cite{nagumo_2024}.
Both methods exhibit doubly strong robustness, where ``doubly'' refers to handling contamination in the reference and target datasets, and ``strong'' signifies robustness even under heavy contamination.
Weighted DRE incorporates a strictly positive weight function $w(x)$ into the Unnormalized Kullback-Leibler divergence, resulting in a convex optimization problem.
$\gamma$-DRE introduces the weight function into the $\gamma$-divergence \cite{fujisawa_mult_anal_2008}, leading to a Difference of Convex functions (DC) optimization problem.
Consistency analyses theoretically confirm that both methods achieve doubly strong robustness in heavily contaminated settings.

The non-asymptotic analysis of the estimator of Weighted DRE is of interest, enabled by the convexity of its optimization problem.
The previous consistency analysis only provides robustness of the optimal parameter under contamination at the population level \cite{nagumo_2024}, whereas non-asymptotic analyses can provide the error order between the estimator and the true parameter at the sample level.
Previous studies have investigated sparse consistency and estimation error of conventional DRE methods \cite{songliu_annals_of_stats2017,songliu_neurips2017,kato2021}.
For instance, DRE can achieve sparse consistency with high probability when the regularization parameter is appropriately chosen \cite{songliu_annals_of_stats2017}, similar to Lasso \cite{wainwright2009,buhlmann_2011,Wainwright_2019}.
Furthermore, previous research has demonstrated that the estimation error of DRE is bounded by $O(\sqrt{1/\min\{n_p,n_q\}})$ \cite{yamada_NeuralComp2011,songliu_annals_of_stats2017,kato2021}, where $n_p$ and $n_q$ denote the dataset sizes of the reference and target datasets, respectively.

Two key challenges arise when conducting a non-asymptotic analysis of Weighted DRE under contamination.
First, previous non-asymptotic analyses often assume that the density ratio is somehow bounded, an assumption that is restrictive and often violated in practical settings.
For instance, some studies enforce a constant bound on the density ratio \cite{songliu_annals_of_stats2017,kato2021}, while others assume the density ratio follows a sub-Gaussian distribution \cite{songliu_annals_of_stats2017,songliu_neurips2017}, which implicitly imposes the boundedness by some function.
An alternative approach modifies the density ratio $p(x)/q(x)$ to the relative density ratio $p(x)/(\alpha p(x)+(1-\alpha)q(x))$, where $0\le\alpha<1$, to achieve the boundedness \cite{yamada_NeuralComp2011}.
The non-asymptotic properties of estimating unbounded density ratios remain unexplored.

Second, robust estimation typically assumes a small contamination ratio.
Many non-asymptotic analyses impose this assumption under the Huber's contamination \cite{lai_2016,chen_2018,gao_2020} or the adversarial contamination \cite{nguyen_2013,chen_2013,dalalyan_2019,diakonikolas_2019,liu_2020,oliverira_2025}.
For example, when the number of outliers $o$ is much smaller than the dataset size $n$, the convergence rate of the Lasso estimator degrades by $O(\sqrt{(o \log n) /n})$ \cite{nguyen_2013}.
However, this assumption is unsuitable for analyzing Weighted DRE, as its doubly strong robustness is designed to handle scenarios under heavy contamination, where the number of outliers is relatively large compared to the dataset size.

This study addresses these challenges by presenting a non-asymptotic analysis of Weighted DRE that demonstrates sparse consistency and estimation error for unbounded density ratios under heavy contamination.
By requiring that the weighted density ratio $r(x)w(x)$ is bounded while allowing $r(x)$ to remain unbounded, we significantly relax previous assumptions.
This assumption can be satisfied if the weight function decays faster to zero than the density ratio.
Besides, our analysis introduces a distinctive framework for strong robustness under heavy contamination, assuming at least one of the following conditions: (i) the contamination ratio is small, as commonly addressed in non-asymptotic analyses \cite{nguyen_2013,liu_2020}, and (ii) outliers have small weighted values, a typical assumption in asymptotic analyses \cite{fujisawa_mult_anal_2008,nagumo_2024}.
This mild condition realizes the non-asymptotic analysis of doubly strong robustness, complementing the prior asymptotic result \cite{nagumo_2024}.
This approach offers the first non-asymptotic result in the strong robustness literature under heavy contamination.

This paper is organized as follows:
Section \ref{section previous research} introduces Weighted DRE, the focus of our non-asymptotic analysis.
Section \ref{section non-asymptotic analysis} outlines the assumptions and key results of the analysis, including conditions related to the unboundedness of the density ratio and the presence of heavy contamination by outliers.
Section \ref{section proof outline} provides a proof outline of the main theorem.
In Section \ref{section experiments}, numerical experiments demonstrate that Weighted DRE achieves sparse consistency for unbounded density ratios under heavy contamination.

\section{Density Ratio Estimation}
\label{section previous research}

\subsection{Non-robust Density Ratio Estimation}
\label{subsection non-robust dre}

The density ratio is defined as the ratio of two density functions.
Let $p(\bm{x})$ and $q(\bm{x})$ be strictly positive density functions of the reference and target datasets for $\bm{x} \in \mathbb{R}^m$, respectively.
The true density ratio can be written as $r(\bm{x}) = p(\bm{x})/q(\bm{x}): \mathbb{R}^m \mapsto \mathbb{R}$.
To estimate the density ratio, we employ the density ratio function $r_{\bm{\beta}}(\bm{x})$, where $\bm{\beta}$ is a parameter, and then measure the discrepancy between the true density ratio $r(\bm{x})$ and the density ratio function $r_{\bm{\beta}}(\bm{x})$.
The choice of the density ratio function and the discrepancy measure realizes various DRE methods \cite{kato2021}.

The parametric model is a practical choice for the density ratio function in high-dimensional settings \cite{songliu_neuralcomp2014,songliu_annals_of_stats2017}.
The parametric function is defined as
\begin{align}
\label{eq parametric dre}
    r_{\bm{\theta}, C}(\bm{x})=Cr_{\bm{\theta}}(\bm{x})= C \exp \left( \bm{\theta}^T h(\bm{x}) \right),
\end{align}
where $\bm{\theta}\in \mathbb{R}^d$ is the difference parameter, $C\in\mathbb{R}$ is the normalizing term, and $h(\bm{x}):\mathbb{R}^m\mapsto\mathbb{R}^d$ is the feature transform function.

Statistical divergences are effective tools for quantifying the discrepancy between the true density ratio $r(\bm{x})$ and the density ratio function $r_{\bm{\beta}}(\bm{x})$.
Among them, the Bregman (BR) divergence is one of the most widely used \cite{bregman_1967,sugiyama_2012,kato2021}.
Let $f$ be a differentiable and strictly convex function with the derivative $\partial f$.
Then, we quantify the discrepancy of $r(\bm{x})$ and $r_{\bm{\beta}}(\bm{x})$ as
\begin{align*}
    D_{\rm BR}(r, r_{\bm{\beta}})
    &=\mathbb{E}_q\left[f\left(r(X)\right)-f\left(r_{\bm{\beta}}(X)\right) -\partial f\left(r_{\bm{\beta}}(X)\right)\left(r(X)-r_{\bm{\beta}}(X)\right)\right] \\
    &=\mathbb{E}_q \left[\partial f\left(r_{\bm{\beta}}(X)\right)r_{\bm{\beta}}(X)-f\left(r_{\bm{\beta}}(X)\right)\right]
        -\mathbb{E}_p \left[\partial f\left(r_{\bm{\beta}}(X)\right)\right] + {\rm const}.
\end{align*}
The Bregman divergence serves as a general framework, with the choice of $f$ determining specific methods.
For example, the UKL (Unnormalized Kullback-Leibler) divergence \cite{Nguyen_nips2007} and KLIEP (Kullback-Leibler Importance Estimation Procedure) \cite{sugiyama2008} adopt $f(t)=t\log t-t$, LSIF (Least-Squares Importance Fitting) \cite{kanamori_2009} and KMM (Kernel Mean Matching) \cite{gretton_2009} adopt $f(t)=(t-1)^2/2$, and the BKL (Binary Kullback-Leibler) divergence \cite{hastie_2001} adopts $f(t)=t\log t-(1+t)\log(1+t)$.

When using the parametric density ratio function in \eqref{eq parametric dre} with $f(t)=t\log t-t$, the Bregman divergence reduces to the UKL divergence, which is formulated as
\begin{equation}
\label{eq ukl div}
\begin{split}
    D_{\rm UKL}(r, r_{\bm{\theta},C})
    &=\mathbb{E}_q\left[r_{\bm{\theta},C}(X)\right]-\mathbb{E}_p\left[\log r_{\bm{\theta},C}(X)\right] +{\rm const} \\
    &=C\mathbb{E}_q\left[\exp\left(\bm{\theta}^Th(X)\right)\right]-\mathbb{E}_p\left[\left(\bm{\theta}^Th(X)+\log C\right)\right] +{\rm const}.
\end{split}
\end{equation}
Because \eqref{eq ukl div} is convex about $C$, the optimal normalizing term is
\begin{equation}
\label{eq conventional normalizing term}
    C^\sharp_{\bm{\theta}}=\frac{1}{\mathbb{E}_q\left[\exp(\bm{\theta}^T h(X))\right]}.
\end{equation}
By substituting \eqref{eq conventional normalizing term} to \eqref{eq ukl div}, we have
\begin{align*}
    D_{\rm{UKL}}(r, r_{\bm{\theta}})=D_{\rm{UKL}}(r, r_{\bm{\theta},C^\sharp_{\bm{\theta}}})
    =-\mathbb{E}_p\left[\bm{\theta}^T h(X) \right]
        + \log \mathbb{E}_q\left[\exp(\bm{\theta}^T h(X))\right] + {\rm const}.
\end{align*}
The UKL divergence is empirically approximated, excluding the constant term, using two datasets $\{\bm{x}^{(p)}_n\}_{n=1}^{n_p}$ and $\{\bm{x}^{(q)}_n\}_{n=1}^{n_q}$:
\begin{align}
\label{eq objective of conventional dre}
    \hat{D}_{\rm{UKL}}(r, r_{\bm{\theta}})
    =-\hat{\mathbb{E}}_p\left[\bm{\theta}^T h(X) \right] + \log \hat{\mathbb{E}}_q\left[\exp(\bm{\theta}^T h(X))\right],
\end{align}
where $\hat{\mathbb{E}}_f[g(X)]=\frac{1}{n_f}\sum_{n=1}^{n_f}g(\bm{x}_n^{(f)})$.
This objective function is convex with respect to $\bm{\theta}$ and can be efficiently minimized using gradient descent.

\subsection{Robust Density Ratio Estimation}

Robust estimation is achieved by introducing a weight function $w(\bm{x}): \mathbb{R}^m \mapsto \mathbb{R}_+$, which mitigates the adverse effects of outliers in the estimator \cite{maronna_martin_yohai_robust_statistics}.
The Bregman divergence can include the weight function as the base measure:
\begin{align*}
    D_{\rm BR}(r, r_{\bm{\beta}}; w)
    &=\mathbb{E}_{wq} \left[\partial f\left(r_{\bm{\beta}}(X)\right)r_{\bm{\beta}}(X)-f\left(r_{\bm{\beta}}(X)\right)\right]
    -\mathbb{E}_{wp} \left[\partial f\left(r_{\bm{\beta}}(X)\right)\right] + {\rm const}.
\end{align*}
With the base measure $w(\bm{x})d\bm{x}$, the Bregman divergence retains the following properties: (i) $D_{\rm{BR}}(r, r_{\bm{\beta}}; w)\ge0$, and (ii) $D_{\rm{BR}}(r, r_{\bm{\beta}}; w)=0\Leftrightarrow r=r_{\bm{\beta}}$.

We propose the estimator of Weighted DRE to enable robust and sparse estimation \cite{nagumo_2024}.
Similar to Section \ref{subsection non-robust dre}, the formulation of the UKL divergence can be given as
\begin{align*}
    D_{\rm UKL}(r, r_{\bm{\theta},C};w)
    &=\mathbb{E}_q\left[r_{\bm{\theta},C}(X)w(X)\right]-\mathbb{E}_p\left[\log r_{\bm{\theta},C}(X)w(X)\right] +{\rm const} \\
    &=C\mathbb{E}_q\left[\exp\left(\bm{\theta}^Th(X)\right)w(X)\right]-\mathbb{E}_p\left[\left(\bm{\theta}^Th(X)+\log C\right)w(X)\right] +{\rm const}.
\end{align*}
Because the optimal normalizing term is
\begin{equation}
\label{eq weighted normalizing term}
    C^\circ_{\bm{\theta}}=\frac{\mathbb{E}_p\left[w(X)\right]}{\mathbb{E}_q\left[\exp(\bm{\theta}^T h(X))w(X)\right]}
\end{equation}
due to the convexity, we have
\begin{align*}
    &D_{\rm{UKL}}(r, r_{\bm{\theta}}; w)=D_{\rm{UKL}}(r, r_{\bm{\theta},C^\circ_{\bm{\theta}}}; w) \\
    &=-\mathbb{E}_p\left[\bm{\theta}^T h(X) w(X)\right]
        + \mathbb{E}_p\left[w(X)\right]\times \log \mathbb{E}_q\left[\exp(\bm{\theta}^T h(X))w(X)\right] + {\rm const}.
\end{align*}
The empirical UKL divergence without the constant term is
\begin{align*}
    \mathcal{L}(\bm{\theta})
    =\hat{D}_{\rm{UKL}}(r, r_{\bm{\theta}}; w)
    =-\hat{\mathbb{E}}_p\left[\bm{\theta}^T h(X) w(X)\right] + \hat{\mathbb{E}}_p\left[w(X)\right]\times \log \hat{\mathbb{E}}_q\left[\exp(\bm{\theta}^T h(X))w(X)\right].
\end{align*}
Then, the naive estimator of Weighted DRE is defined as
\begin{align*}
    \hat{\bm{\theta}}^\circ = \underset{\bm{\theta}}{\operatorname{argmin}}\, \mathcal{L}(\bm{\theta}) + \lambda \|\bm{\theta}\|_1,
\end{align*}
where $\lambda$ is a regularization parameter.
This objective function is also convex about $\bm{\theta}$.
Weighted DRE demonstrates doubly strong robustness, maintaining consistency of the optimal parameter even under high contamination settings \cite{nagumo_2024}.
The previously proposed robust DRE method, Trimmed DRE \cite{songliu_neurips2017}, does not have doubly strong robustness.

\section{Non-asymptotic Analysis of Weighted DRE}
\label{section non-asymptotic analysis}

This section presents a non-asymptotic analysis of Weighted DRE.
Section \ref{subsection problem setting} outlines the problem setting involving heavy contamination.
Sections \ref{subsection assumptions for boundedness}, \ref{subsection assumption for robustness}, and \ref{subsection assumptions for sparse estimation} detail the assumptions for the boundedness of the weight function, robustness under heavy contamination, and sparse estimation, respectively.
Finally, Section \ref{subsection main theorem} establishes sparse consistency and estimation error bound for unbounded density ratio under heavy contamination.

\subsection{Problem Setting}
\label{subsection problem setting}

We describe the problem setting of density ratio estimation under heavy contamination.
Suppose that the reference and target datasets are contaminated by outliers, more precisely, drawn from the contaminated distributions \cite{huber2004robust,maronna_martin_yohai_robust_statistics,hampel2011robust,nagumo_2024} given by
\begin{align}
\label{eq setting of p and q}
    p^{\dagger}(\bm{x}) = (1 - \varepsilon_p) p^*(\bm{x}) + \varepsilon_p \delta_p(\bm{x}), \quad
    q^{\dagger}(\bm{x}) = (1 - \varepsilon_q) q^*(\bm{x}) + \varepsilon_q \delta_q(\bm{x}),
\end{align}
respectively. Here, $p^*(\bm{x})$ and $q^*(\bm{x})$ represent the true reference and target distributions, $\delta_p(\bm{x})$ and $\delta_q(\bm{x})$ correspond to the outlier distributions, and $\varepsilon_p,\varepsilon_q \in [0, 1)$ denote the contamination ratios.
We assume that the given dataset is composed of a combination of inliers, $\bm{x}^{(p^*)}$ or $\bm{x}^{(q^*)}$, and outliers, $\bm{x}^{(\delta_p)}$ or $\bm{x}^{(\delta_q)}$:
\begin{equation}
\label{eq dataset setting}
\begin{split}
    \{\bm{x}_n^{(p^{\dagger})}\}_{n=1}^{n_p} = \{\bm{x}_n^{(p^*)}\}_{n=1}^{n_p^*} \cup \{\bm{x}_n^{(\delta_p)}\}_{n=1}^{\varepsilon_p n_p}, \\
    \{\bm{x}_n^{(q^{\dagger})}\}_{n=1}^{n_q} =\{\bm{x}_n^{(q^*)}\}_{n=1}^{n_q^*} \cup \{\bm{x}_n^{(\delta_q)}\}_{n=1}^{\varepsilon_q n_q},
\end{split}
\end{equation}
where $n_p^*= (1-\varepsilon_p)n_p$ and $n_q^*=(1-\varepsilon_q)n_q$.
This setting is the same as the previous non-asymptotic analysis of density ratio estimation \cite{songliu_neurips2017}.
For simplicity, we assume that $\varepsilon_pn_p$ and $\varepsilon_qn_q$ are integers.
We define $n_{p,q}^*=\min\left\{n_p^*, n_q^*\right\}$.

We only consider the case where the parameter $\bm{\theta}$ exists in a compact convex set $\Theta$ \cite{eunho_2012,songliu_annals_of_stats2017,songliu_neurips2017,nagumo_2024}.
The normalizing term can be written as $C^*_{\bm{\theta}}=\mathbb{E}_{p^*}[w(X)]/\mathbb{E}_{q^*}[\exp(\bm{\theta}^{T}X)w(X)]$.
We suppose that there exists an interior point $\bm{\theta}^* \in \Theta$ which satisfies $p^*(\bm{x})=r(\bm{x};\bm{\theta}^*,C^*_{\bm{\theta}^*})q^*(\bm{x})$, where $r(\bm{x};\bm{\theta}^*,C^*_{\bm{\theta}^*})=r_{\bm{\theta}^*,C^*_{\bm{\theta}^*}}(\bm{x})$ \cite{songliu_annals_of_stats2017,songliu_neurips2017}.
Given the index set $\mathcal{E}$ with $|\mathcal{E}|=d$, let us define two sets of indices for the true parameter, the active set $S=\{t\in \mathcal{E}| \theta^*_{t} \neq 0\}$ and the non-active set $S^c=\{t\in \mathcal{E}| \theta^*_{t} = 0\}$.
The number of the non-zero parameters is set to $k = |S|$.

The non-asymptotic analysis aims to show sparse consistency and estimation error of the estimator of Weighted DRE in the contaminated setting.
Let us define the contaminated objective function $\mathcal{L}^\dagger(\bm\theta)$ and the uncontaminated one $\mathcal{L}^*(\bm\theta)$:
\begin{equation}
\label{eq objective function}
\begin{split}
    \mathcal{L}^\dagger(\bm{\theta})
    &=-\hat{\mathbb{E}}_{p^\dagger}\left[\bm{\theta}^T h(X)w(X)\right]+\hat{\mathbb{E}}_{p^\dagger}\left[w(X)\right]\times \log \hat{\mathbb{E}}_{q^\dagger}\left[\exp(\bm{\theta}^T h(X))w(X)\right], \\
    \mathcal{L}^*(\bm{\theta})
    &=-\hat{\mathbb{E}}_{p^*}\left[\bm{\theta}^Th(X)w(X)\right]+\hat{\mathbb{E}}_{p^*}\left[w(X)\right]\times \log \hat{\mathbb{E}}_{q^*}\left[\exp(\bm{\theta}^T h(X))w(X)\right].
\end{split}
\end{equation}
Then, the estimator of Weighted DRE in the contaminated setting is defined as
\begin{align}
\label{eq-9}
    \hat{\bm{\theta}}=\underset{\bm{\theta}}{\operatorname{argmin}}\,\mathcal{L}^\dagger(\bm{\theta})+\lambda_{n_p^*,n_q^*}\|\bm{\theta}\|_1,
\end{align}
where $\lambda_{n_p^*,n_q^*}$ is the regularization hyper-parameter of L1 norm.

\vspace{\baselineskip}
\textit{Notation.} Given a matrix $A\in \mathbb{R}^{m\times n}$ and a parameter $q \in [1,\infty]$, $\|A\|_q$ represents the induced matrix-operator norm \cite{ravikumar2010,horn_2013}.
Two examples of particular importance in this paper are the spectral norm $\|A\|_2$, corresponding to the maximal singular value of $A$, and the max norm, given by $\|A\|_\infty = \max_{i=1,...,m}\sum_{j=1}^{n}|A_{ij}|$.
We make use of the bound $\|A\|_\infty\le \sqrt{n}\|A\|_2$.
Another important norm is the element-wise max norm as $\|A\|_{\max}=\max_{i,j}|A_{ij}|$.
For a vector $\bm{a}\in \mathbb{R}^m$, we define the max norm as $\|\bm{a}\|_{\infty}=\max_i |a_i|$.
We define the minimum and maximum eigenvalue operators of a symmetric matrix as $\Lambda_{\min}$ and $\Lambda_{\max}$, respectively, where $\Lambda_{\max}[A]=\|A\|_2$ for a symmetric matrix $A$.
Given sequences $\{a_n\}$ and $\{b_n\}$, the notation $a_n=O(b_n)$ means that there exists a positive constant $C < \infty$ such that $|a_n| \le C |b_n|$ for any $n$.
Other notations are provided in Appendix \ref{section notations}.

\subsection{Assumptions for Boundedness of Weight Function}
\label{subsection assumptions for boundedness}

We present the assumptions that the weight function $w(\bm{x})$ should satisfy.
Specifically, the weight function is required to decay to zero faster than the density ratio function $r(\bm{x};\bm{\theta},C^*_{\bm{\theta}})$ and the feature transform function $h(\bm{x})$.
These assumptions promote mathematical tractability while relaxing restrictive conditions imposed by previous studies \cite{songliu_annals_of_stats2017,songliu_neurips2017,kato2021}.

\begin{assumption}[Boundedness of Weight Function]
\label{assumption weight for normal}
\begin{align*}
    0 < w(\bm{x}) \le W_{\max}<\infty, \quad
    \mathbb{E}_{p^*}\left[w(X)\right] \ge W_{\min} > 0.
\end{align*}
\end{assumption}
The upper-boundedness of the weight function can be easily satisfied by a simple example such as $w(\bm{x})=\exp\left(-\|\bm{x}\|_4^4\right)$, which has been used in previous research \cite{nagumo_2024}.
The assumption of the lower-bounded integral implies that the weight function should not eliminate the inliers from the reference distribution too much.
Propositions based on Assumption \ref{assumption weight for normal} are provided in Appendix \ref{section discussion of assumption of weight function}.

\begin{assumption}[Boundedness of Density Ratio with Weight Function]
\label{assumption-9}
    For any $\bm{\theta}\in\Theta$,
    \begin{align*}
        \exp\left(\bm{\theta}^T h(\bm{x})\right) w(\bm{x}) &\le E_{\max}<\infty, \\
        \mathbb{E}_{q^*}\left[\exp\left(\bm{\theta}^T h(X)\right)w(X)\right] &\ge E_{\min} > 0.
    \end{align*}
\end{assumption}

Assumption \ref{assumption-9} asserts that the weighted density ratio $r(\bm{x})w(\bm{x})$ is bounded by a positive constant $E$: $r(\bm{x};\bm{\theta},C^*_{\bm{\theta}})w(\bm{x})\le E<\infty$, even when the density ratio $r(\bm{x})$ is unbounded.
This assumption contrasts with previous research, which require that $r(\bm{x})$ to be bounded or sub-Gaussian \cite{songliu_annals_of_stats2017,songliu_neurips2017,kato2021}.
This assumption can be satisfied when the weight function decays to zero faster than the density ratio function.
The lower boundedness of the integral implies that the weight function should not eliminate the inliers from the target distribution too much.
Propositions based on Assumption \ref{assumption-9} are provided in Appendix \ref{section discussion of assumption-9}.

Assuming the boundedness of the density ratio is overly restrictive \cite{songliu_annals_of_stats2017}.
A simple example involving Gaussian distributions illustrates this limitation.
Consider the density ratio of one-dimensional Gaussian distributions with the zero means and variances $\sigma_p^2$ and $\sigma_q^2$:
\begin{align}
\label{eq one-dim gauss}
    p(x)=\frac{1}{\sqrt{2\pi\sigma_p^2}}\exp\left(-\frac{x^2}{2\sigma_p^2}\right), \quad
    q(x)=\frac{1}{\sqrt{2\pi\sigma_q^2}}\exp\left(-\frac{x^2}{2\sigma_q^2}\right).
\end{align}
The density ratio of these distributions is given by
\begin{align*}
    r(x)=\frac{p(x)}{q(x)}=\frac{\sigma_q}{\sigma_p}\exp\left(-\frac{1}{2}\left(\frac{1}{\sigma_p^2}-\frac{1}{\sigma_q^2}\right)x^2\right).
\end{align*}
While $r(x)$ is bounded when $\sigma_p \le \sigma_q$, it becomes unbounded when $\sigma_p > \sigma_q$.
Consequently, the assumption of bounded density ratios is often violated in practical scenarios.

The sub-Gaussian assumption for the density ratio \cite{songliu_annals_of_stats2017,songliu_neurips2017} often assumes that the density ratio is bounded by some function.
Let $Z=Z(X)= r(X;\bm{\theta},C^*_{\bm{\theta}})-\mathbb{E}_q[r(X;\bm{\theta},C^*_{\bm{\theta}})]$, where $Z$ is a zero-mean random variable.
The sub-Gaussian assumption implies that $\mathbb{E}_q[\exp(tZ)]\le \exp(a^2t^2/2)$ for $t>0$, where $a$ is a positive constant \cite{Wainwright_2019}.
Since $\mathbb{E}_q[\exp(tZ)]=\int \exp\{tZ(\bm{x})+\log q(\bm{x})\}d\bm{x}$ should be finite, $\exp\{tZ(\bm{x})+\log q(\bm{x})\}$ should be bounded for $\bm{x} \in \mathbb{R}^m$.
However, this boundedness condition is often difficult to satisfy.
For instance, in a one-dimensional setting where $\sigma_p^2=1$ and $\sigma_q^2=1/2$ in \eqref{eq one-dim gauss}, consider a density ratio defined as $r(x)= \exp(x^2/2)/\sqrt{2}$.
In this case, $\exp\{tZ(X)+\log q(X)\} = \exp\{t\exp(X^2/2)/\sqrt{2}-X^2+{\rm const}\}$ becomes unbounded, leading $\mathbb{E}_q[\exp(tZ)]=\infty$, which violates the sub-Gaussian assumption.
This example demonstrates that the sub-Gaussian assumption is valid only when the density ratio is bounded by some function.

\begin{assumption}[Boundedness of Features with Weight Function]
\label{assumption-8}
    For any $\bm{\theta}\in\Theta$ and $t\in \mathcal{E}$,
    \begin{align*}
        \|h(\bm{x})w(\bm{x})\|_{\infty}&\le D_{\max}<\infty, \\
        \|h(\bm{x})         \exp(\bm{\theta}^Th(\bm{x}))w(\bm{x})\|_{\infty}&\le D_{\max} < \infty, \\
        \|h(\bm{x})h(\bm{x})^T\exp(\bm{\theta}^Th(\bm{x}))w(\bm{x})\|_{\max}&\le D_{\max} < \infty, \\
        \|h_t(\bm{x})h_S(\bm{x})h_S(\bm{x})^T\exp(\bm{\theta}^Th(\bm{x}))w(\bm{x})\|_{\max}&\le D_{\max} < \infty.
    \end{align*}
\end{assumption}
Assumption \ref{assumption-8} is less restrictive than the assumption used in previous research \cite{songliu_annals_of_stats2017}, which requires the feature transform function $h(\bm{x})$ to be bounded: $\|h(\bm{x})\|_{\infty}<\infty$.
The commonly used choice of $h(\bm{x})=[x_1x_1,x_1x_2,...,x_dx_d]$, justified for Gaussian distributions, violates this assumption because $\|h(\bm{x})\|_{\infty}\rightarrow\infty$ as $\bm{x}\rightarrow\infty$.
By defining the weight function as $w(\bm{x})=\exp\left(-\|\bm{x}\|_4^4\right)$, the relevant terms can be bounded because $w(\bm{x})$ decays to zero faster than $h(\bm{x})$ and $\exp(\bm{\theta}^Th(\bm{x}))$.
A proposition based on Assumption \ref{assumption-8} is provided in Appendix \ref{subsection discussion of assumption 8}.

The choice of weight function satisfies Assumptions \ref{assumption weight for normal}, \ref{assumption-9}, and \ref{assumption-8} is not so difficult.
In our settings, we suppose that we can distinguish the inliers from the outliers in \eqref{eq dataset setting} \cite{songliu_neurips2017}.
Therefore, the weight function can be set not to eliminate the inliers in the reference and target datasets to satisfy Assumptions \ref{assumption weight for normal} and \ref{assumption-9}, because these data points can be assumed to represent the population distributions.
Because we assume the existence of the true parameter $\bm{\theta}^* \in \Theta$ which satisfies $p^*(\bm{x})=r(\bm{x};\bm{\theta}^*,C^*_{\bm{\theta}^*})q^*(\bm{x})$ \cite{songliu_annals_of_stats2017,eunho_2012}, we implicitly assume that we can prepare the suitable feature transform function $h(\bm{x})$ in advance.
Then, the proper decaying rate of the weight function can be set to satisfy Assumption \ref{assumption-8}.
Overall, the setting of the weight function seems not so difficult when the same assumptions in the previous research of density ratio estimation \cite{songliu_annals_of_stats2017,songliu_neurips2017} and graphical modeling \cite{eunho_2012} are supposed.
Note that the choice of the weight function is independent from the choice of the unknown true parameter $\bm{\theta}^*$.

\subsection{Assumption for Robustness}
\label{subsection assumption for robustness}

We propose an assumption to guarantee robustness under heavy contamination.
When specific conditions are met for the uncontaminated objective function $\mathcal{L}^*$, the estimator in \eqref{eq-9} exhibits non-asymptotic properties in the absence of contamination.
Extending this result, if $\mathcal{L}^\dagger$ closely approximates $\mathcal{L}^*$, the estimator retains its non-asymptotic properties even in contaminated settings.

\begin{assumption}[Robustness Assumption]
\label{assumption weight for outlier}
Let
\begin{align*}
    & \nu_1 = \max_{\bm{x} \in {\cal X}^{(\delta_p)}} w(\bm{x}), \quad 
    \nu_2 = \max_{\bm{x} \in {\cal X}^{(\delta_q)}} \max_{\bm{\theta} \in \Theta} \exp(\bm{\theta}^T h(\bm{x}))w(\bm{x}), \quad 
    \nu_3 = \max_{t\in \mathcal{E}} \max_{\bm{x} \in {\cal X}^{(\delta_p)}}  w(\bm{x})|h_t(\bm{x})|, \\
    & \nu_4 = \max_{t\in \mathcal{E}} \max_{\bm{x} \in {\cal X}^{(\delta_q)}} \max_{\bm{\theta} \in \Theta}  \exp(\bm{\theta}^T h(\bm{x}))w(\bm{x})|h_t(\bm{x})|, \quad
    \nu_5= \max_{t,t'\in \mathcal{E}} \max_{\bm{x} \in {\cal X}^{(\delta_q)}} \max_{\bm{\theta} \in \Theta} \exp(\bm{\theta}^T h(\bm{x}))w(\bm{x})|h_t(\bm{x}) h_{t'}(\bm{x})|, \\
    & \nu_6 = \max_{t \in \mathcal{E}} \max_{t',t''\in S} \max_{\bm{x} \in {\cal X}^{(\delta_q)}}\max_{\bm{\theta} \in \Theta} \exp(\bm{\theta}^T h(\bm{x}))w(\bm{x})|h_{t}(\bm{x})h_{t'}(\bm{x})h_{t''}(\bm{x})|,
\end{align*}
where ${\cal X}^{(f)}=\{\bm{x}^{(f)}_n\}$ and $h_t(\bm{x})$ is the $t$-th element of $h(\bm{x})$.
We define the maximum values as $\nu=\max_{j=1,\ldots,6} \nu_j$ and $\varepsilon=\max\left\{\varepsilon_p, \varepsilon_q\right\}$.
We assume that $k^{3/2}\varepsilon \nu$ is sufficiently small.
\end{assumption}

Assumption \ref{assumption weight for outlier} introduces additional bounds: $\nu_1$ for Assumption \ref{assumption weight for normal}, $\nu_2$ for Assumption \ref{assumption-9}, and $\nu_{3},\ldots,\nu_{6}$ for Assumption \ref{assumption-8}.
While Assumptions \ref{assumption weight for normal}, \ref{assumption-9}, and \ref{assumption-8} define bounds for all data points $\bm{x}\in\mathcal{X}^{(p^*)}\cup \mathcal{X}^{(q^*)}\cup \mathcal{X}^{(\delta_p)}\cup \mathcal{X}^{(\delta_q)}$, Assumption \ref{assumption weight for outlier} imposes additional bounds specifically for the outliers $\bm{x}\in\mathcal{X}^{(\delta_p)}\cup \mathcal{X}^{(\delta_q)}$.

This assumption requires that at least one of the following conditions to hold: (i) the contamination ratio $\varepsilon$ is small, and (ii) outliers have small weighted values $\nu$.
Previous non-asymptotic analyses in robust estimation have considered only condition (i), more precisely, $\varepsilon \rightarrow 0$ as $n_{p,q}^* \rightarrow \infty$ \cite{nguyen_2013,liu_2020}.
While condition (ii) is typically assumed in heavily contaminated settings in the context of asymptotic analyses of strong robustness \cite{fujisawa_mult_anal_2008,nagumo_2024}, it has not been explored in non-asymptotic analyses.
By requiring at least one of the two conditions, (i) or (ii), to hold, this assumption is less restrictive than those in prior studies.
Because the settings of sparse and robust estimation usually assume that $k<n_{p,q}^*$ and $\nu$ is a constant when the outliers are given, the assumption that the effect of outliers disappears when the number of the inliers increases, more precisely, $k^{3/2}\varepsilon\nu \rightarrow 0$ as $n_{p,q}^*\rightarrow \infty$, may be not so strong.
Assumption \ref{assumption weight for outlier} enables non-asymptotic analysis of doubly strong robustness under heavy contamination, even when condition (i) is violated.

The following theorem shows the similarity of the objective functions $\mathcal{L}^\dagger(\bm{\theta})$ and $\mathcal{L}^*(\bm{\theta})$.
Let the first, second, and element-wise third derivatives of $\mathcal{L}^\dagger$ be denoted by $\nabla\mathcal{L}^\dagger(\bm{\theta})$, $\nabla^2\mathcal{L}^\dagger(\bm{\theta})$, and $\nabla_{t}\nabla^2\mathcal{L}^\dagger(\bm{\theta})$, and those of $\mathcal{L}^*$ be denoted by $\nabla\mathcal{L}^*(\bm{\theta})$, $\nabla^2\mathcal{L}^*(\bm{\theta})$, and $\nabla_{t}\nabla^2\mathcal{L}^*(\bm{\theta})$, respectively.
We define the sub-matrices of $\nabla_{t}\nabla^2\mathcal{L}^\dagger(\bm{\theta})$ and $\nabla_{t}\nabla^2\mathcal{L}^*(\bm{\theta})$ as $\nabla_{t}\nabla^2_{SS}\mathcal{L}^\dagger(\bm{\theta})$ and $\nabla_{t}\nabla^2_{SS}\mathcal{L}^*(\bm{\theta})$, respectively.

\begin{theorem}
\label{theorem of robust objective function}
We assume that $n_{p,q}^* \ge N_{\delta}$ holds, where $N_{\delta}$ is a positive constant.
Under Assumptions \ref{assumption weight for normal}, \ref{assumption-9}, \ref{assumption-8}, and \ref{assumption weight for outlier}, we have
\begin{align*}
    &\nabla\mathcal{L}^\dagger(\bm{\theta}) = (1-\varepsilon_p)\nabla\mathcal{L}^*(\bm{\theta})+O(\varepsilon \nu), \quad
    \nabla^2\mathcal{L}^\dagger(\bm{\theta}) = (1-\varepsilon_p)\nabla^2\mathcal{L}^*(\bm{\theta})+O(\varepsilon \nu), \\
    &\nabla_{t}\nabla^2_{SS}\mathcal{L}^\dagger(\bm{\theta}) =(1-\varepsilon_p)\nabla_{t}\nabla^2_{SS}\mathcal{L}^*(\bm{\theta}) + O(\varepsilon \nu),
\end{align*}
for any $\bm{\theta}\in\Theta$ and $t \in \mathcal{E}$ with probability at least $1-2\delta$, where $\delta$ is a small positive constant.
\end{theorem}

Theorem \ref{theorem of robust objective function} encapsulates the effects of outliers into the $O(\varepsilon\nu)$ terms, which are assumed to be sufficiently small under Assumption \ref{assumption weight for outlier}.
The proof outline primarily follows the approach in \cite{fujisawa_mult_anal_2008}.
Propositions regarding the sample size condition are detailed in Appendix \ref{section propositions of boundedness}, while the complete proof is provided in Appendix \ref{section proof of theorem of robust objective function}.

\subsection{Assumptions for Sparse Estimation}
\label{subsection assumptions for sparse estimation}

We introduce three common assumptions used in the non-asymptotic analysis of sparse estimation.
In these assumptions, the Hessian matrix of the UKL divergence with the weight function plays an crucial role.
We refer to this matrix as the weighted Fisher information matrix, expressed as
\begin{align*}
    \mathcal{I}^*(\bm{\theta})
    &={\mathbb{E}}_{q^*}\left[r\left(X;\bm{\theta},{C}^*_{\bm{\theta}}\right)w(X)h(X)h(X)^T\right] \\
    &\quad-\frac{1}{\mathbb{E}_{p^*}[w(X)]}{\mathbb{E}}_{q^*}\left[r\left(X;\bm{\theta},{C}^*_{\bm{\theta}}\right)w(X)h(X)\right]{\mathbb{E}}_{q^*}\left[r\left(X;\bm{\theta},{C}^*_{\bm{\theta}}\right)w(X)h(X)\right]^T.
\end{align*}
The sample version of the weighted Fisher information matrix is represented as $\hat{\mathcal{I}}^*(\bm{\theta})=\nabla^2 \mathcal{L}^*(\bm{\theta})$ in the clean setting and $\hat{\mathcal{I}}^\dagger(\bm{\theta})=\nabla^2 \mathcal{L}^\dagger(\bm{\theta})$ in the contaminated setting.
For simplicity, we denote $\mathcal{I}^*=\mathcal{I}^*(\bm{\theta}^*)$ and $\hat{\mathcal{I}}^{\dagger}=\hat{\mathcal{I}}^{\dagger}(\bm{\theta}^*)$ and define their sub-matrices as ${\mathcal{I}}^*_{SS}$ and $\hat{\mathcal{I}}^{\dagger}_{SS}$, respectively.
Note that while prior research on DRE imposes these assumptions at the sample level \cite{songliu_annals_of_stats2017}, we impose them at the population level.

\begin{assumption}[Dependency Assumption]
\label{assumption-1}
    \begin{equation*}
        \Lambda_{\min}\left[{\mathcal{I}}^*_{SS}\right] \ge \lambda_{\min} > 0.
    \end{equation*}
\end{assumption}
This dependency assumption ensures that the relevant covariates remain sufficiently independent \cite{wainwright2009,ravikumar2010,songliu_annals_of_stats2017}.
We demonstrate that this dependency assumption holds at the sample level even under contamination.

\begin{proposition}
\label{proposition robustness of minimum eignvalue}
If Assumptions \ref{assumption weight for normal}, \ref{assumption-9}, \ref{assumption-8}, \ref{assumption weight for outlier} and \ref{assumption-1} hold and $n_{p,q}^* \gtrsim k^2 \log d$ holds,
\begin{equation*}
    \Lambda_{\min}\left[\hat{\mathcal{I}}_{SS}^\dagger\right]\ge \frac{(1-\varepsilon_p) \lambda_{\min}}{4}
\end{equation*}
with probability at least $1-3\delta$.
\end{proposition}
The proof of Proposition \ref{proposition robustness of minimum eignvalue} is provided in Appendix \ref{section discussion of assumption 1}.

\begin{assumption}[Incoherence Assumption]
\label{assumption-2}
    If the sub-matrix of the weighted Fisher information matrix ${\mathcal{I}}_{SS}^{*}$ is invertible,
    \begin{equation*}
        \left\|{\mathcal{I}}_{S^cS}^* {\mathcal{I}}_{SS}^{*^{-1}} \right\|_{\infty} \le 1-\alpha,
    \end{equation*}
    where $0<\alpha\le1$.
\end{assumption}
This incoherence assumption ensures that the elements in the non-active set do not exert disproportionately strong effects on those in the active set \cite{zhao_2006,wainwright2009,ravikumar2010,songliu_annals_of_stats2017}.
We demonstrate that this incoherence assumption holds at the sample level even under contamination.

\begin{proposition}
\label{proposition incoherence of I dagger}
If Assumptions \ref{assumption weight for normal}, \ref{assumption-9}, \ref{assumption-8}, \ref{assumption weight for outlier}, \ref{assumption-1}, and \ref{assumption-2} hold and $n_{p,q}^* \gtrsim k^3 \log d$ holds,
\begin{align*}
    \left\| \hat{\mathcal{I}}_{S^cS}^\dagger\hat{\mathcal{I}}_{SS}^{\dagger^{-1}} \right\|_{\infty} \le 1 - \frac{\alpha}{2}
\end{align*}
holds with probability at least  $1-6\delta$.
\end{proposition}
The proof of Proposition \ref{proposition incoherence of I dagger} is provided in Appendix \ref{section consequences of assumption-2}.

\begin{assumption}[Smoothness Assumption]
\label{assumption smoothness}
For any $\bm{\theta}\in\Theta$ and $t\in \mathcal{E}$,
\begin{align*}
    \Lambda_{\max}\left[\nabla_{t}\mathcal{I}^*_{SS}(\bm{\theta})\right] \le \lambda_{3,\max} < \infty.
\end{align*}
\end{assumption}

This assumption requires the objective function to be smooth \cite{ravikumar2010,eunho_2012,songliu_annals_of_stats2017}.
A proposition derived from Assumption \ref{assumption smoothness} is presented in Appendix \ref{section discussion of robust objective function}.

\subsection{Main Theorem}
\label{subsection main theorem}

The following theorem presents the non-asymptotic analysis of Weighted DRE for unbounded density ratios under heavy contamination.
\begin{theorem}
\label{theorem-sparse consistency}
    Suppose that Assumptions \ref{assumption weight for normal}, \ref{assumption-9}, \ref{assumption-8}, \ref{assumption weight for outlier}, \ref{assumption-1}, \ref{assumption-2}, and \ref{assumption smoothness} hold and $n_{p,q}^* \gtrsim k^3 \log d$ holds.
    We define the regularization parameter of L1 norm as
    \begin{align*}
        \lambda_{n_p^*,n_q^*}=L(1-\varepsilon_p)\sqrt{\frac{\log(6d/\delta)}{n_{p,q}^*}}+M(1-\varepsilon_p)\sqrt{\frac{\log(2/\delta)}{n_p^*}}+N\varepsilon\nu,
    \end{align*}
    where $L$, $M$, and $N$ are some positive constants and $\delta$ is a small positive constant.
    Suppose that the minimum value of the true parameter in the active set is not too close to zero:
    \begin{align*}
        \min_{t\in S}\left|\theta_t^*\right|\ge \frac{80L}{\lambda_{\min}}\sqrt{\frac{k\log (6d/\delta)}{n_{p,q}^*}}+\frac{80M}{\lambda_{\min}}\sqrt{\frac{k\log(2/\delta)}{n_p^*}}+\frac{80N}{\lambda_{\min}}\frac{\varepsilon\nu\sqrt{k}}{1-\varepsilon_p}.
    \end{align*}
    Then, with probability at least  $1-8\delta$, the estimator of \eqref{eq-9} is unique and has sparse consistency: $\hat{{\theta}}_{t}\neq0$ for $t\in S$ and $\hat{{\theta}}_{t}=0$ for $t \in S^c$.
    The estimation error can be bounded by
    \begin{align}
    \label{eq estimation error}
        \left\|\hat{\bm{\theta}}-\bm{\theta}^*\right\|_2 \le \frac{40L}{\lambda_{\min}}\sqrt{\frac{k\log (6d/\delta)}{n_{p,q}^*}}+\frac{40M}{\lambda_{\min}}\sqrt{\frac{k\log(2/\delta)}{n_p^*}}+\frac{40N}{\lambda_{\min}}\frac{\varepsilon\nu\sqrt{k}}{1-\varepsilon_p}.
    \end{align}
\end{theorem}

The estimation error \eqref{eq estimation error} comprises three components: the naive estimation error of order $O(\sqrt{k\log d/n_{p,q}^*})$, the additional estimation error introduced by the weight function of order $O(\sqrt{k/n_p^*})$, and the contamination effect of order $O(\varepsilon\nu\sqrt{k}/(1-\varepsilon_p))$.
The naive estimation error has the same order as that of conventional DRE \cite{songliu_annals_of_stats2017}.
Furthermore, the constant in this term can be the same as that of conventional DRE \cite{songliu_annals_of_stats2017} if certain conditions are adjusted: the incoherence assumption (Assumption \ref{assumption-2}) is supposed at the sample level, and the weight function is fixed at 1.
The second term in \eqref{eq estimation error} reflects the variability introduced by the weight function.
The difference of the normalizing term in conventional DRE \eqref{eq conventional normalizing term} and Weighted DRE \eqref{eq weighted normalizing term} is $\mathbb{E}_p\left[w(X)\right]$ in the numerator.
This empirical expectation introduces the additional estimation error not to eliminate the inliers from the reference dataset, as supposed in Assumption \ref{assumption weight for normal}.
The third term in \eqref{eq estimation error} represents the contamination effect caused by outliers.
Even under heavy contamination where $\varepsilon$ is large, this term remains small under Assumption \ref{assumption weight for outlier}.
This term confirms doubly strong robustness of Weighted DRE from a non-asymptotic perspective, complementing the previous asymptotic result \cite{nagumo_2024}.

Theorem \ref{theorem-sparse consistency} demonstrates that Weighted DRE achieves sparse consistency for unbounded density ratios.
Conventional DRE attains sparse consistency only when the density ratio is bounded \cite{songliu_annals_of_stats2017}.
Weighted DRE extends this capability by requiring that the weighted density ratio is bounded, even if the density ratio is unbounded, as specified in Assumption \ref{assumption-9}.
In the absence of outliers ($\varepsilon=0$), Weighted DRE has a small additional estimation error on the order of $O(\sqrt{k/n_p^*})$ compared to conventional DRE, since the first term dominates the error order in \eqref{eq estimation error}.
This minimal error allows Weighted DRE to maintain sparse consistency for unbounded density ratios.

Although the sample complexity of Weighted DRE, $n_{p,q}^*\gtrsim k^3\log d$, is higher than that of conventional DRE, $n_{p,q}^*\gtrsim k^2\log d$ \cite{songliu_annals_of_stats2017}, this difference arises from the use of weaker assumptions.
Our analysis imposes the dependency, incoherence, and smoothness assumptions (Assumptions \ref{assumption-1}, \ref{assumption-2}, and \ref{assumption smoothness}) at the population level, whereas the prior work assumes these assumptions at the sample level with probability one \cite{songliu_annals_of_stats2017}.
Our result requires an additional sample complexity of $k$ to bridge this gap.
This result is analogous to the non-asymptotic analysis of Ising models, where the sample complexity satisfies $n \gtrsim k^3 \log d$ when these assumptions are imposed at the population level and $n \gtrsim k^2 \log d$ at the sample level \cite{ravikumar2010}.

Because the second term in \eqref{eq estimation error} is smaller than the first order, Theorem \ref{theorem-sparse consistency} can be rewritten in a simpler form.

\begin{corollary}
Suppose that Assumptions \ref{assumption weight for normal}, \ref{assumption-9}, \ref{assumption-8}, \ref{assumption weight for outlier}, \ref{assumption-1}, \ref{assumption-2}, and \ref{assumption smoothness} hold and $n_{p,q}^* \gtrsim k^3 \log d$ holds.
We define
\begin{align*}
    \lambda_{n_p^*,n_q^*}=L'(1-\varepsilon_p)\sqrt{\frac{\log(6d/\delta)}{n_{p,q}^*}}+N\varepsilon\nu,
\end{align*}
where $L'$ and $N$ are some positive constants and $\delta$ is a small positive constant, and suppose that
\begin{align*}
    \min_{t\in S}\left|\theta_t^*\right|\ge \frac{80L'}{\lambda_{\min}}\sqrt{\frac{k\log (6d/\delta)}{n_{p,q}^*}}+\frac{80N}{\lambda_{\min}}\frac{\varepsilon\nu\sqrt{k}}{1-\varepsilon_p}.
\end{align*}
Then, with probability at least  $1-8\delta$, the estimator of \eqref{eq-9} is unique and has sparse consistency.
The estimation error can be bounded by
\begin{align*}
    \left\|\hat{\bm{\theta}}-\bm{\theta}^*\right\|_2 \le \frac{40L'}{\lambda_{\min}}\sqrt{\frac{k\log (6d/\delta)}{n_{p,q}^*}}+\frac{40N}{\lambda_{\min}}\frac{\varepsilon\nu\sqrt{k}}{1-\varepsilon_p}.
\end{align*}
\end{corollary}

\section{Proof Outline of Main Theorem}
\label{section proof outline}

The main theorem establishes sparse consistency and provides the estimation error bound.
Sections \ref{subsection zero-pattern recovery} and \ref{subsection non-zero pattern recovery} present proofs for the zero and non-zero pattern recovery, respectively, which are two key components of sparse consistency.
Section \ref{subsection estimation error} presents the proof of the estimation error bound.

\subsection{Zero Pattern Recovery}
\label{subsection zero-pattern recovery}

The main proof procedure is based on the primal-dual witness method \cite{wainwright2009,ravikumar2010,songliu_annals_of_stats2017}.
Let $\hat{\bm{z}}$ be a dual variable associated with $\hat{\bm{\theta}}$, defined by the following equation:
\begin{equation}
\label{eq-14}
    \nabla \mathcal{L}^\dagger(\hat{\bm{\theta}})+\lambda_{n_p^*,n_q^*}\hat{\bm{z}}=\bm{0}.
\end{equation}
If $\hat{\bm{z}}$ is the sub-gradient of $\|\hat{\bm{\theta}}\|_1$, then \eqref{eq-14} is the optimality condition for \eqref{eq-9}, and $\hat{\bm{\theta}}$ is an optimal solution to \eqref{eq-9}.
Furthermore, the following lemma establishes the relationship between the dual variable $\hat{\bm{z}}$ and the sparsity patterns of any optimal solutions of \eqref{eq-9}.
\begin{lemma}
\label{lemma-1}
    Suppose that there exists an optimal solution $\hat{\bm{\theta}}$ of \eqref{eq-9} with the associated optimal dual variable $\hat{\bm{z}}$ such that $\|\hat{\bm{z}}_{S^c}\|_\infty<1$.
    Then, any optimal solution $\tilde{\bm{\theta}}$ of \eqref{eq-9} satisfies $\tilde{\bm{\theta}}_{S^c}=\bm{0}$.
    Furthermore, if $\hat{\mathcal{I}}_{SS}^\dagger$ is strictly positive definite, $\hat{\bm{\theta}}$ is the unique optimal solution.
\end{lemma}

The proof of Lemma \ref{lemma-1} is provided in Appendix \ref{section proof of lemmas in main proof}.
Lemma \ref{lemma-1} demonstrates the correct zero pattern recovery of the estimator.
The goal of the remainder of this section is to establish the existence of an optimal solution that satisfies the condition stated in Lemma \ref{lemma-1}, more precisely, $\hat{\bm{\theta}}$ with $\hat{\bm{z}}$ such that $\|\hat{\bm{z}}_{S^c}\|_\infty<1$.

We define an optimal solution $\hat{\bm{\theta}}=[\hat{\bm{\theta}}_S^T, \bm{0}^T]^T$, where $\hat{\bm{\theta}}_S \in \mathbb{R}^k$ is given by solving the constrained optimization problem:
\begin{align*}
    \hat{\bm{\theta}}_S=\underset{\bm{\theta}_S}{\operatorname{argmin}}\,\mathcal{L}^\dagger\left(\begin{bmatrix} \bm{\theta}_S \\ \bm{0} \end{bmatrix}\right)+\lambda_{n_p^*,n_q^*}\|\bm{\theta}_{S}\|_1.
\end{align*}
From \eqref{eq-14}, by applying the mean-value theorem, we have
\begin{equation}
\label{eq-16}
    \underbrace{\nabla^2\mathcal{L}^\dagger(\bm{\theta}^*)}_{\hat{\mathcal{I}}^\dagger}\left[\hat{\bm{\theta}}-\bm{\theta}^*\right]+\lambda_{n_p^*,n_q^*}\hat{\bm{z}}=\underbrace{-\nabla \mathcal{L}^\dagger(\bm{\theta}^*)}_{\bm{w}^\dagger}+\underbrace{\left[\nabla^2\mathcal{L}^\dagger(\bm{\theta}^*)-\overline{\nabla^2\mathcal{L}^\dagger}\right]\left[\hat{\bm{\theta}}-\bm{\theta}^*\right]}_{\bm{g}^\dagger},
\end{equation}
where $\overline{\nabla^2\mathcal{L}^\dagger}$ is a matrix whose $t$-th row is $\overline{\nabla^2\mathcal{L}^\dagger}_{t}=\nabla_{t}\nabla\mathcal{L}^\dagger\left(\bar{\bm{\theta}}^{t}\right)^T$ and $\bar{\bm{\theta}}^{t}$ is between $\bm{\theta}^*$ and $\hat{\bm{\theta}}$ in a coordinate fashion.
We can then rewrite \eqref{eq-16} in blockwise fashion:
\begin{align}
    \label{eq-17a}
    \hat{\mathcal{I}}^\dagger_{SS}[\hat{\bm{\theta}}_S-\bm{\theta}^*_S] + \lambda_{n_p^*,n_q^*}\hat{\bm{z}}_S &= \bm{w}^\dagger_S + \bm{g}^\dagger_S, \\
    \label{eq-17b}
    \hat{\mathcal{I}}^\dagger_{S^cS}[\hat{\bm{\theta}}_S-\bm{\theta}^*_S] + \lambda_{n_p^*,n_q^*}\hat{\bm{z}}_{S^c}&=\bm{w}_{S^c}^\dagger+\bm{g}_{S^c}^\dagger.
\end{align}
Because $n_{p,q}^*\gtrsim k^3\log d$ in the assumption of Theorem \ref{theorem-sparse consistency} is stronger than $n_{p,q}^*\gtrsim k^2 \log d$, $\hat{\mathcal{I}}_{SS}^{\dagger}$ is invertible with probability at least $1-3\delta$ from Proposition \ref{proposition robustness of minimum eignvalue}.
By substituting  \eqref{eq-17a} into \eqref{eq-17b}, we have
\begin{equation*}
    \hat{\mathcal{I}}^\dagger_{S^cS}\hat{\mathcal{I}}^{\dagger^{-1}}_{SS}[\bm{w}^\dagger_S+\bm{g}^\dagger_S-\lambda_{n_p^*,n_q^*}\hat{\bm{z}}_S] + \lambda_{n_p^*,n_q^*}\hat{\bm{z}}_{S^c}=\bm{w}_{S^c}^\dagger+\bm{g}_{S^c}^\dagger.
\end{equation*}
According to the triangle inequality,
\begin{align*}
    \lambda_{n_p^*,n_q^*}\|\hat{\bm{z}}_{S^c}\|_\infty \le & \|\bm{w}_{S^c}^\dagger\|_\infty+\|\bm{g}_{S^c}^\dagger\|_\infty \\
    &+\|\hat{\mathcal{I}}^\dagger_{S^cS}\hat{\mathcal{I}}^{\dagger^{-1}}_{SS}\|_{\infty} \left(\|\bm{w}_{S}^\dagger\|_\infty + \|\bm{g}_{S}^\dagger\|_\infty+\lambda_{n_p^*,n_q^*}\right).
\end{align*}
Because we assume $n_{p,q}^*\gtrsim k^3\log d$, from Proposition \ref{proposition incoherence of I dagger}, we obtain
\begin{equation}
\label{eq z bound temporal}
    \|\hat{\bm{z}}_{S^c}\|_\infty \le \frac{2-\alpha/2}{\lambda_{n_p^*,n_q^*}}\left(\|\bm{w}^\dagger\|_\infty + \|\bm{g}^\dagger\|_\infty\right) + \left(1-\frac{\alpha}{2}\right)
\end{equation}
with probability at least  $1-6\delta$.

Now we need the boundedness of $\|\bm{w}^\dagger\|_\infty$ and $\|\bm{g}^\dagger\|_\infty$ to show $\|\hat{\bm{z}}_{S^c}\|_\infty<1$.
The following lemmas show these boundedness.

\begin{lemma}
\label{lemma-5}
    If Assumptions \ref{assumption weight for normal}, \ref{assumption-9}, \ref{assumption-8}, and \ref{assumption weight for outlier} hold and $n_{p,q}^* \ge N_{\delta}$ holds,
    \begin{equation*}
        \|\bm{w}^\dagger\|_\infty \le \frac{\alpha\lambda_{n_p^*,n_q^*}}{8(2-\alpha/2)}
    \end{equation*}
    holds with probability at least $1-3\delta$, where
    \begin{align*}
        \lambda_{n_p^*,n_q^*}
        =L(1-\varepsilon_p)\sqrt{\frac{\log(6d/\delta)}{n_{p,q}^*}}+M(1-\varepsilon_p)\sqrt{\frac{\log(2/\delta)}{n_p^*}}+N\varepsilon\nu,
    \end{align*}
    and $L$, $M$, and $N$ are some positive constants.
\end{lemma}

\begin{lemma}
\label{lemma-3}
    Suppose that Assumptions \ref{assumption weight for normal}, \ref{assumption-9}, \ref{assumption-8}, \ref{assumption weight for outlier}, \ref{assumption-1}, and \ref{assumption smoothness} hold and $n_{p,q}^* \gtrsim k^2\log d$ holds.
    If $k\lambda_{n_p^*,n_q^*}\le \frac{(1-\varepsilon_p)\lambda_{\min}^2}{960\lambda_{3,\max}}$ and $\|\bm{w}^\dagger_S\|_\infty \le \frac{\lambda_{n_p^*,n_q^*}}{4}$ hold, then
    \begin{align*}
        \left\|\hat{\bm{\theta}}-\bm{\theta}^*\right\|_2 \le \frac{40}{(1-\varepsilon_p)\lambda_{\min}}\sqrt{k}\lambda_{n_p^*,n_q^*}
    \end{align*}
    holds with probability at least $1-3\delta$.
\end{lemma}

\begin{lemma}
\label{lemma-4}
    Suppose that Assumptions \ref{assumption weight for normal}, \ref{assumption-9}, \ref{assumption-8}, \ref{assumption weight for outlier}, \ref{assumption-1}, and \ref{assumption smoothness} hold and $n_{p,q}^* \gtrsim k^2\log d$ holds.
    If $k\lambda_{n_p^*,n_q^*}\le \frac{(1-\varepsilon_p)\lambda_{\min}^2}{960\lambda_{3,\max}}\frac{\alpha}{40(2-\alpha/2)}$ and $\|\bm{w}^\dagger_S\|_\infty \le \frac{\lambda_{n_p^*,n_q^*}}{4}$ hold, then
    \begin{align*}
        \|\bm{g}^\dagger\|_\infty \le \frac{\alpha \lambda_{n_p^*,n_q^*}}{8(2-\alpha/2)}
    \end{align*}
    holds with probability at least $1-3\delta$.
\end{lemma}

The proofs of Lemmas \ref{lemma-5}, \ref{lemma-3}, and \ref{lemma-4} are provided in Appendix \ref{section proof of lemmas in main proof}.
Let us now examine the assumptions of these lemmas.
The conditions on $n_{p,q}^*$ in Lemmas \ref{lemma-5}, \ref{lemma-3}, and \ref{lemma-4} are weaker than the condition of $n_{p,q}^*\gtrsim k^3 \log d$ in Theorem \ref{theorem-sparse consistency}.
The condition on $\|\bm{w}_S^\dagger\|_\infty$ in Lemmas \ref{lemma-3} and \ref{lemma-4} holds with high probability if Lemma \ref{lemma-5} holds, because $\frac{\alpha}{2(2-\alpha/2)} < 1$ holds for $0<\alpha \le 1$.
When $\lambda_{n_p^*,n_q^*}$ is set to in Lemma \ref{lemma-5} and $n_{p,q}^*\gtrsim k^2 \log d$ holds, which is milder than $n_{p,q}^*\gtrsim k^3 \log d$ in Theorem \ref{theorem-sparse consistency}, the conditions on $\lambda_{n_p^*,n_q^*}$ in Lemmas \ref{lemma-3} and \ref{lemma-4} hold.
Note that the condition on $\lambda_{n_p^*,n_q^*}$ in Lemma \ref{lemma-4} is stronger than that in Lemma \ref{lemma-3}, because $\frac{\alpha}{40(2-\alpha/2)} <1$ holds for $0<\alpha\le 1$.

Applying Lemmas \ref{lemma-5} and \ref{lemma-4} to \eqref{eq z bound temporal}, we have
\begin{align*}
    \|\hat{\bm{z}}_{S^c}\|_\infty
    &\le \frac{2-\alpha/2}{\lambda_{n_p^*,n_q^*}}\left(\frac{\alpha}{8(2-\alpha/2)}\lambda_{n_p^*,n_q^*}+\frac{\alpha}{8(2-\alpha/2)}\lambda_{n_p^*,n_q^*}\right)+ \left(1- \frac{\alpha}{2}\right) \\
    &\le 1 -\frac{1}{4}\alpha \\
    &<1
\end{align*}
with probability at least  $1-8\delta$.
By Lemma \ref{lemma-1}, this result implies that any optimal $\hat{\bm{\theta}}$ of \eqref{eq-9} recovers the correct zero pattern.
Furthermore, since $\hat{\mathcal{I}}^\dagger_{SS}$ is strictly positive definite with high probability, as stated in Proposition \ref{proposition robustness of minimum eignvalue}, we can conclude that $\hat{\bm{\theta}}$ is the unique optimal solution based on Lemma \ref{lemma-1}.

\subsection{Non-zero Pattern Recovery}
\label{subsection non-zero pattern recovery}

The correct non-zero pattern recovery is defined as $\hat{\theta}_t\neq 0$ for $t\in S$.
It suffices to show
\begin{align*}
    \min_{t\in S}|\theta^*_{t}| \ge 2 \sup_{t\in S}|\hat{\theta}_{t} - \theta^*_{t}|
\end{align*}
because, for $t' \in S$,
\begin{align*}
    2\sup_{t \in S}|\hat{\theta}_{t}-\theta^*_{t}|
    \ge 2 |\hat{\theta}_{t'} - \theta^*_{t'}|
    \ge 2 \left( |\theta^*_{t'}| - |\hat{\theta}_{t'}| \right)
    \ge 2 \min_{t \in S}|\theta^*_t| - 2|\hat{\theta}_{t'}|
\end{align*}
and then we have
\begin{align*}
    |\hat{\theta}_{t'}| \ge \frac{1}{2} \min_{t\in S}|\theta_{t}^*| > 0.
\end{align*}

Because Lemma \ref{lemma-3} shows
\begin{align*}
    \sup_{t\in S}\left|\hat{\theta}_t-\theta_t^*\right|\le\left\|\hat{\bm{\theta}}-\bm{\theta}^*\right\|_2\le\frac{40}{(1-\varepsilon_p)\lambda_{\min}}\sqrt{k}\lambda_{n_p^*,n_q^*},
\end{align*}
we need
\begin{align*}
    \min_{t\in S}\left|\theta_t^*\right|\ge\frac{80}{(1-\varepsilon_p)\lambda_{\min}}\sqrt{k}\lambda_{n_p^*,n_q^*}
\end{align*}
to ensure the correct non-zero pattern recovery.
From Lemma \ref{lemma-5}, we have
\begin{align*}
    \min_{t\in S}\left|\theta_t^*\right|\ge \frac{80L}{\lambda_{\min}} \sqrt{\frac{k\log (6d/\delta)}{n_{p,q}^*}}+\frac{80M}{\lambda_{\min}} \sqrt{\frac{k\log(2/\delta)}{n_p^*}}+ \frac{80N}{\lambda_{\min}}\frac{\varepsilon\nu\sqrt{k}}{1-\varepsilon_p}.
\end{align*}

\subsection{Estimation Error}
\label{subsection estimation error}

Finally, we obtain the bound of the estimation error.
From Lemmas \ref{lemma-5} and \ref{lemma-3}, we have
\begin{align*}
    \left\|\hat{\bm{\theta}}-\bm{\theta}^*\right\|_2
    &\le \frac{40}{(1-\varepsilon_p)\lambda_{\min}}\sqrt{k}\left(L(1-\varepsilon_p)\sqrt{\frac{\log (6d/\delta)}{n_{p,q}^*}}+M(1-\varepsilon_p)\sqrt{\frac{\log(2/\delta)}{n_p^*}} + N\varepsilon\nu\right) \\
    &=\frac{40L}{\lambda_{\min}}\sqrt{\frac{k\log (6d/\delta)}{n_{p,q}^*}}+\frac{40M}{\lambda_{\min}}\sqrt{\frac{k\log(2/\delta)}{n_p^*}}+\frac{40N}{\lambda_{\min}}\frac{\varepsilon\nu\sqrt{k}}{1-\varepsilon_p}.
\end{align*}

\section{Numerical Experiments}
\label{section experiments}

\subsection{Robustness}
\label{subsection experiment for robustness}

The first experiment investigates the robust estimation of density ratios under contamination.
Theorem \ref{theorem-sparse consistency} suggests that Weighted DRE achieves sparse consistency under heavy contamination.
We validate this theoretical result through a numerical experiment.

We conducted the experiment under the problem setting described in Section \ref{subsection problem setting}.
We considered two Gaussian distributions, $p^*(\bm{x})=N(\bm{0}, \Lambda_p^{-1})$ and $q^*(\bm{x})=N(\bm{0}, \Lambda_q^{-1})$ \cite{songliu_annals_of_stats2017,songliu_neurips2017}.
Since the true density ratio is given by $r^*(\bm{x})=p^*(\bm{x})/q^*(\bm{x})\propto\exp(\bm{x}^T(\Lambda_q-\Lambda_p)\bm{x})$, the true parameter of the density ratio is $\Theta=\Lambda_q-\Lambda_p$.
The precision matrix $\Lambda_p$ of the reference distribution was set to the identity matrix.
The difference between the two precision matrices, $\Theta_{ij}=(\Lambda_q-\Lambda_p)_{ij}$, was set to $0.4$ for $(i,j)\in S$, where the indices were randomly assigned to different row and column numbers \cite{songliu_neurips2017}.
The cardinality of the active set was $k=|S|=4$ \cite{songliu_annals_of_stats2017}.
The true density ratio is unbounded because $r^*(\bm{x})\propto \exp(\sum_{(i,j)\in S} \theta_{ij}x_i x_j)$ can take large values.
The outlier distributions were set to $\delta_p(\bm{x})=\delta_q(\bm{x})=N(100\times\bm{1}_m, I_m)$, where $\bm{1}_m=(1,\ldots,1)\in \mathbb{R}^m$ and $I_m\in \mathbb{R}^{m\times m}$ was the identity matrix \cite{nagumo_2024}.
The contamination ratios in \eqref{eq setting of p and q} were set to $\varepsilon_p = \varepsilon_q = 0$ in the clean setting and $\varepsilon_p = \varepsilon_q = 0.2$ in the contaminated setting \cite{nagumo_2024}.

We compared the success probability \cite{wainwright2009,ravikumar2010} of estimating the true parameter using conventional DRE \cite{songliu_annals_of_stats2017} and Weighted DRE.
The density ratio function was parametrized by $\bm{\theta}^Th(\bm{x})=\sum_{i\le j}^{m}\theta_{ij}x_ix_j$ \cite{songliu_neurips2017,nagumo_2024}.
The true parameter $\bm{\theta}^*\in\mathbb{R}^d$ consists of the elements in the upper triangle of $\Lambda_q-\Lambda_p\in\mathbb{R}^{m\times m}$, where $d=(m^2+m)/2$.
The dimension sizes were set to $m=50,100,200$, corresponding to parameter sizes as $d=1275,5050,20100$, respectively.
The regularization parameter $\lambda_{n_p^*,n_q^*}$ was set as $\lambda_0\sqrt{\log d / n_{p,q}^*}$ \cite{wainwright2009,ravikumar2010,songliu_annals_of_stats2017}, where $\lambda_0=5.0$.
The weight function in Weighted DRE was set as $w(\bm{x})=\exp(-\|\bm{x}\|^{4}_{4}/20m)$ \cite{nagumo_2024}.
The sizes of the reference and target datasets were set to be equal: $n_p=n_q$ \cite{songliu_annals_of_stats2017}.
The success probability for estimating the active set was calculated according to the dataset sizes $n_{p,q}^*$ for different dimension sizes $m$ by repeating the experiments 200 times \cite{ravikumar2010}.

The left column of Figure \ref{fig:robust} illustrates that DRE achieves sparse consistency only in the clean setting.
Although the true density ratio is theoretically unbounded, it appears experimentally bounded in this scenario.
This experimental boundedness occurs because the data are sampled from the standard Gaussian distributions, causing most values in the objective function \eqref{eq objective of conventional dre} to remain small.
Consequently, the theoretically unbounded density ratio behaves as if it were bounded experimentally, and the success probability increases as the dataset size increases.
In contrast, DRE fails in the contaminated setting because the outliers adversely affect the objective function \eqref{eq objective of conventional dre}.

The right column of Figure \ref{fig:robust} demonstrates that Weighted DRE achieves sparse consistency even under contamination, as indicated by Theorem \ref{theorem-sparse consistency}.
Weighted DRE maintains high success probabilities in the clean and contaminated settings when the dataset size is large.
However, when comparing success probabilities at the same dataset size, they decrease in the following order: ``DRE/clean,'' ``Weighted DRE/clean,'' and ``Weighted DRE/contaminated.''
These differences can be attributed to the second and third orders in the estimation error \eqref{eq estimation error}, which are associated with the weight function and the contamination effect, respectively.

\begin{figure*}[t]
\vskip 0.2in
\begin{center}
    \centerline{\includegraphics[width=\columnwidth]{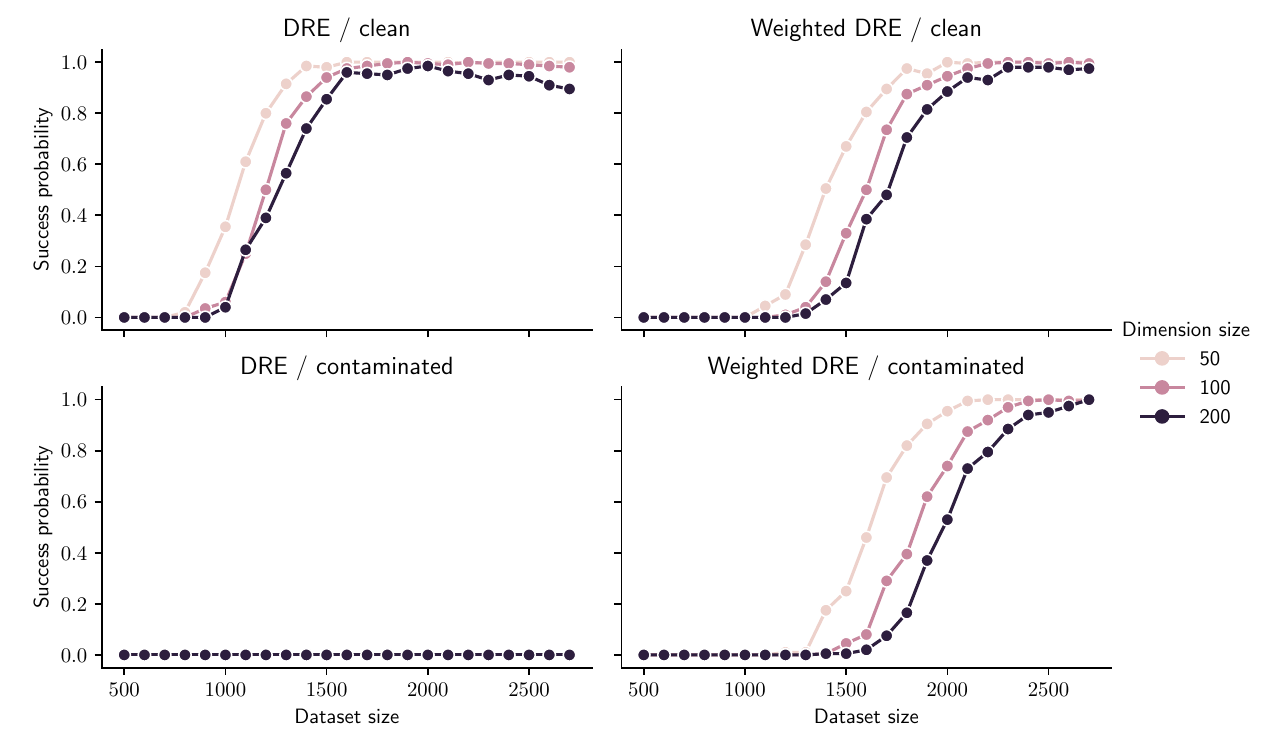}}
    \caption{
        The success probability in the estimation of the active set of DRE and Weighted DRE in the clean and contaminated settings.
        The x-axis shows the dataset sizes and each line corresponds to the different dimension size $m$.
    }
    \label{fig:robust}
\end{center}
\vskip -0.2in
\end{figure*}

\subsection{Unboundedness}
\label{subsection experiment for unboundedness}

The second experiment involves estimating unbounded density ratios in the absence of contamination.
Theorem \ref{theorem-sparse consistency} shows that Weighted DRE achieves sparse consistency for unbounded density ratios, as guaranteed by Assumption \ref{assumption-9}, whereas conventional DRE does not.

We considered two Gaussian distributions, $p^*(\bm{x})=N(\bm{0}, \Lambda_p)$ and $q^*(\bm{x})=N(\bm{0}, \Lambda_q)$, where $\Lambda_p$ and $\Lambda_q$ were diagonal matrices.
The true density ratio is given by
\begin{align*}
    r^*(\bm{x})=\frac{p^*(\bm{x})}{q^*(\bm{x})}=C \exp\left(\bm{x}^T(\Lambda_q - \Lambda_p)\bm{x}\right)
    =C \exp\left(\sum_{(i,i)\in S} \lambda_{ii}x_i^2 \right),
\end{align*}
where $\lambda_{ii}=(\Lambda_q-\Lambda_p)_{ii}$ and $C>0$ is a constant.
The true density ratio $r^*(\bm{x})$ is bounded if $\lambda_{ii}\le0$ for $(i,i)\in S$ and unbounded if $\lambda_{ii}>0$ for $(i,i)\in S$.
We set $\lambda_{ii}=-0.4$ by $((\Lambda_p)_{ii}, (\Lambda_q)_{ii})=(0.8, 0.4)$ for the bounded density ratio, and $\lambda_{ii}=0.4$ by $((\Lambda_p)_{ii}, (\Lambda_q)_{ii})=(0.4, 0.8)$ for the unbounded density ratio.
The elements of the precision matrices in the non-active set were set to $(\Lambda_p)_{ii}= (\Lambda_q)_{ii}=1.0$ for $(i,i)\in S^c$.
Therefore, the dimensions in the active set have smaller precisions, or equivalently larger variance, than those in the non-active set.
A detailed discussion of this precision setting can be found in Appendix \ref{section experimental settings}.
All other experimental settings were the same as in Section \ref{subsection experiment for robustness}, except that the coefficient of the regularization parameter was set to $\lambda_0=4.0$.

The left column of Figure \ref{fig:unbounded} demonstrates that DRE achieves sparse consistency only for bounded density ratios.
DRE exhibits low success probabilities when estimating unbounded density ratios, even with large dataset sizes.
This limitation can arise because the non-asymptotic theory of DRE assumes that the density ratio is bounded \cite{songliu_annals_of_stats2017}.
This estimation error may occur due to data sampled from the tail of the density function with the large variance, adversely affecting the objective function \eqref{eq objective of conventional dre}.

The right column of Figure \ref{fig:unbounded} demonstrates that Weighted DRE achieves sparse consistency for bounded and unbounded density ratios.
The specified density ratio and weight function settings satisfy Assumption \ref{assumption-9}, allowing Weighted DRE to achieve high success probabilities for both scenarios.
The similarity in success probabilities arises because the estimation error in \eqref{eq estimation error} is unaffected by whether the density ratio is bounded or unbounded.
However, due to the small additional estimation error introduced by the weight function in \eqref{eq estimation error}, Weighted DRE may exhibit slightly lower success probabilities than DRE when estimating bounded density ratios.

\begin{figure*}[t]
\vskip 0.2in
\begin{center}
    \centerline{\includegraphics[width=\columnwidth]{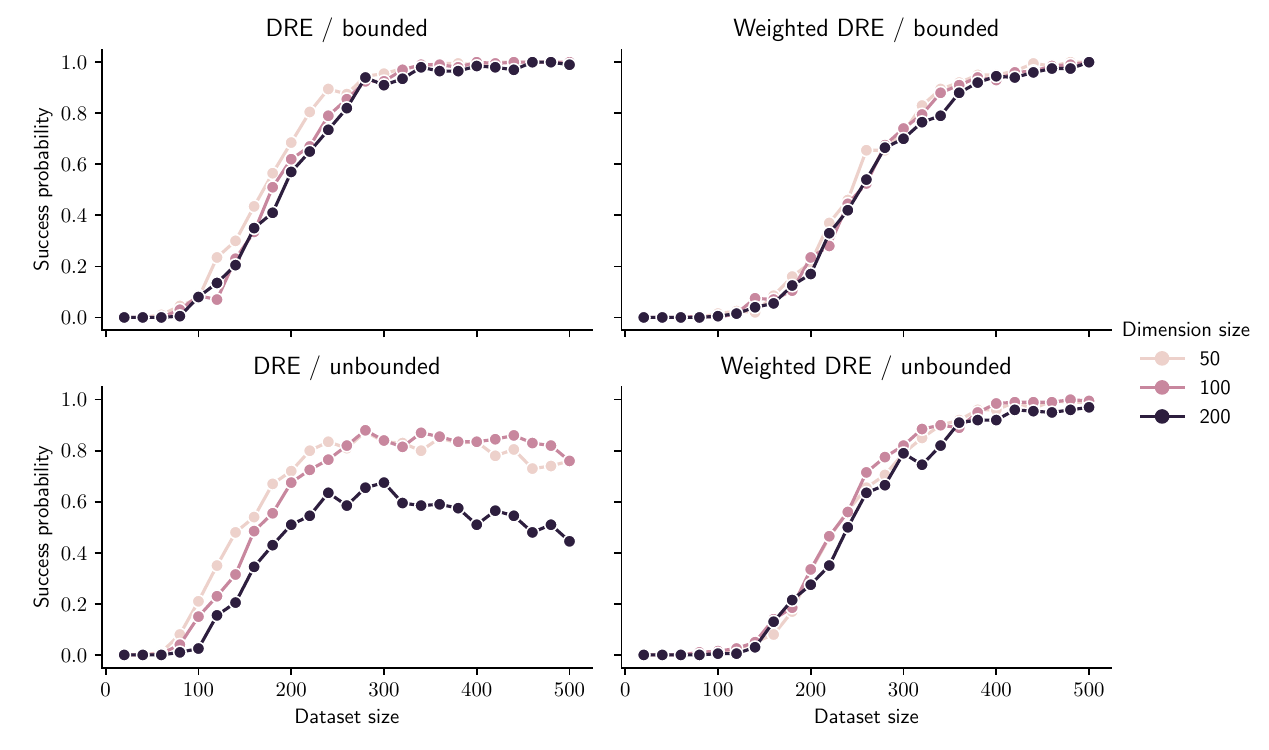}}
    \caption{
        The success probability in the estimation of the active set of the bounded and unbounded density ratio by DRE and Weighted DRE.
        The x-axis shows the dataset sizes and each line corresponds to the different dimension size $m$.
    }
    \label{fig:unbounded}
\end{center}
\vskip -0.2in
\end{figure*}

\section{Conclusion}

This study has established sparse consistency of Weighted DRE, even under heavy contamination of the reference and target datasets.
This result highlights doubly strong robustness from a non-asymptotic perspective, complementing the prior asymptotic result \cite{nagumo_2024}.
Furthermore, in scenarios without outliers, Weighted DRE achieves sparse consistency for unbounded density ratios with a slight additional estimation error compared to conventional DRE \cite{songliu_annals_of_stats2017}.
Notably, the weight function designed to eliminate outliers also ensures the boundedness of the weighted density ratio.

\bibliographystyle{abbrv}
\bibliography{references}

@inproceedings{songliu_neurips2017,
 author = {Liu, Song and Takeda, Akiko and Suzuki, Taiji and Fukumizu, Kenji},
 booktitle = {Advances in Neural Information Processing Systems},
 editor = {I. Guyon and U. Von Luxburg and S. Bengio and H. Wallach and R. Fergus and S. Vishwanathan and R. Garnett},
 pages = {},
 publisher = {Curran Associates, Inc.},
 title = {Trimmed Density Ratio Estimation},
 volume = {30},
 year = {2017}
}

@article{songliu_annals_of_stats2017,
author = {Song Liu and Taiji Suzuki and Raissa Relator and Jun Sese and Masashi Sugiyama and Kenji Fukumizu},
title = {{Support consistency of direct sparse-change learning in Markov networks}},
volume = {45},
journal = {The Annals of Statistics},
number = {3},
publisher = {Institute of Mathematical Statistics},
pages = {959 -- 990},
year = {2017},
}

@article{songliu_neuralcomp2014,
  author = {Liu, Song and Quinn, John and Gutmann, Michael and Suzuki, Taiji and Sugiyama, Masashi},
  year = {2014},
  month = {03},
  pages = {1169-1197},
  title = {Direct Learning of Sparse Changes in Markov Networks by Density Ratio Estimation},
  volume = {26},
  isbn = {978-3-642-38708-1},
  journal = {Neural computation},
}

@article{songliu_behaviormetrika2017,
  author          = {Song Liu and Kenji Fukumizu and Taiji Suzuki},
  journal         = {Behaviormetrika},
  title           = {Learning sparse structural changes in high-dimensional Markov networks: A review on methodologies and theories},
  volume          = {44},
  year            = {2017},
  pages           = {265--296},
}

@article{songliu_neuralnet2013,
title = {Change-point detection in time-series data by relative density-ratio estimation},
journal = {Neural Networks},
volume = {43},
pages = {72-83},
year = {2013},
issn = {0893-6080},
author = {Song Liu and Makoto Yamada and Nigel Collier and Masashi Sugiyama}
}

@inproceedings{kawahara_icdm2009,
author = {Yoshinobu Kawahara and Masashi Sugiyama},
  title = {Change-Point Detection in Time-Series Data by Direct Density-Ratio Estimation},
  journal = {Proceedings of the 2009 SIAM International Conference on Data Mining (SDM)},
  pages = {389-400},
  year = {2009}
}

@article{hido_kis2011,
  author          = {Shohei Hido and Yuta Tsuboi and Hisashi Kashima and Masashi Sugiyama and Takafumi Kanamori},
  journal         = {Knowledge and Information Systems},
  title           = {Statistical outlier detection using direct density ratio estimation},
  volume          = {26},
  pages           = {309-336},
  year            = {2011},
}

@article{fujisawa_mult_anal_2008,
title = {Robust parameter estimation with a small bias against heavy contamination},
journal = {Journal of Multivariate Analysis},
volume = {99},
number = {9},
pages = {2053-2081},
year = {2008},
issn = {0047-259X},
author = {Hironori Fujisawa and Shinto Eguchi}
}

@article{na_biometrika2020,
    author = {Na, S and Kolar, M and Koyejo, O},
    title = "{Estimating differential latent variable graphical models with applications to brain connectivity}",
    journal = {Biometrika},
    volume = {108},
    number = {2},
    pages = {425-442},
    year = {2020},
    month = {09},
    issn = {0006-3444},
}

@article{byol_jrssb2021,
    author = {Kim, Byol and Liu, Song and Kolar, Mladen},
    title = "{Two-Sample Inference for High-Dimensional Markov Networks}",
    journal = {Journal of the Royal Statistical Society Series B: Statistical Methodology},
    volume = {83},
    number = {5},
    pages = {939-962},
    year = {2021},
    month = {09},
    issn = {1369-7412},
}

@article{wornowizki_comp_stat_2016,
  author          = {Max Wornowizki and Roland Fried},
  journal         = {Computational Statistics},
  title           = {Two-sample homogeneity tests based on divergence measures},
  volume          = {31},
  pages           = {291-313},
  year            = {2016},
}

@article{ideker_MolSysBio2012,
  author          = {Trey Ideker and Nevan J Krogan},
  journal         = {Molecular Systems Biology},
  title           = {Differential network biology},
  volume          = {8},
  pages           = {565},
  year            = {2012}
}

@article{yamada_NeuralComp2011,
  title={Relative Density-Ratio Estimation for Robust Distribution Comparison},
  author={Makoto Yamada and Taiji Suzuki and Takafumi Kanamori and Hirotaka Hachiya and Masashi Sugiyama},
  journal={Neural Computation},
  year={2011},
  volume={25},
  pages={1324-1370},
}

@InProceedings{smola_icais2009,
  title = 	 {Relative Novelty Detection},
  author = 	 {Smola, Alex and Song, Le and Teo, Choon Hui},
  booktitle = 	 {Proceedings of the Twelth International Conference on Artificial Intelligence and Statistics},
  pages = 	 {536--543},
  year = 	 {2009},
  editor = 	 {van Dyk, David and Welling, Max},
  volume = 	 {5},
  series = 	 {Proceedings of Machine Learning Research},
  address = 	 {Hilton Clearwater Beach Resort, Clearwater Beach, Florida USA},
  month = 	 {16--18 Apr},
  publisher =    {PMLR}
}

@book{sugiyama_dre2012,
  place={Cambridge},
  title={Density Ratio Estimation in Machine Learning},
  publisher={Cambridge University Press},
  author={Sugiyama, Masashi and Suzuki, Taiji and Kanamori, Takafumi},
  year={2012}
}

@inproceedings{Nguyen_nips2007,
  author = {Nguyen, XuanLong and Wainwright, Martin J and Jordan, Michael},
  booktitle = {Advances in Neural Information Processing Systems},
  editor = {J. Platt and D. Koller and Y. Singer and S. Roweis},
  pages = {},
  publisher = {Curran Associates, Inc.},
  title = {Estimating divergence functionals and the likelihood ratio by penalized convex risk minimization},
  volume = {20},
  year = {2007}
}

@book{maronna_martin_yohai_robust_statistics,
  author         = {Ricardo A. Maronna and R. Douglas Martin and Victor J. Yohai},
  publisher      = {John Wiley \& Sons},
  title          = {Robust Statistics: Theory and Methods},
  year           = {2006}
}

@book{hampel2011robust,
  title={Robust Statistics: the approach based on influence functions},
  author={Hampel, Frank R and Ronchetti, Elvezio M and Rousseeuw, Peter J and Stahel, Werner A},
  year={2011},
  publisher={John Wiley \& Sons}
}

@book{huber2004robust,
  title={Robust Statistics},
  author={Huber, Peter J},
  year={2004},
  publisher={John Wiley \& Sons}
}

@inproceedings{rhodes2020,
  author = {Rhodes, Benjamin and Xu, Kai and Gutmann, Michael U.},
  booktitle = {Advances in Neural Information Processing Systems},
  editor = {H. Larochelle and M. Ranzato and R. Hadsell and M.F. Balcan and H. Lin},
  pages = {4905--4916},
  publisher = {Curran Associates, Inc.},
  title = {Telescoping Density-Ratio Estimation},
  volume = {33},
  year = {2020}
}

@InProceedings{kato2021,
  title = 	 {Non-Negative Bregman Divergence Minimization for Deep Direct Density Ratio Estimation},
  author =       {Kato, Masahiro and Teshima, Takeshi},
  booktitle = 	 {Proceedings of the 38th International Conference on Machine Learning},
  pages = 	 {5320--5333},
  year = 	 {2021},
  editor = 	 {Meila, Marina and Zhang, Tong},
  volume = 	 {139},
  series = 	 {Proceedings of Machine Learning Research},
  month = 	 {18--24 Jul},
  publisher =    {PMLR},
}

@article{sugiyama2008,
  author          = {Masashi Sugiyama and Taiji Suzuki and Shinichi Nakajima and Hisashi Kashima and Paul von Bünau and Motoaki Kawanabe},
  journal         = {Annals of the Institute of Statistical Mathematics},
  title           = {Direct importance estimation for covariate shift adaptation},
  volume          = {60},
  pages           = {699-746},
  year            = {2008}
}

@ARTICLE{wainwright2009,
  author={Martin J. Wainwright},
  journal={IEEE Transactions on Information Theory}, 
  title={Sharp Thresholds for High-Dimensional and Noisy Sparsity Recovery Using $\ell _{1}$-Constrained Quadratic Programming (Lasso)}, 
  year={2009},
  volume={55},
  number={5},
  pages={2183-2202},
}

@article{ravikumar2010,
  author = {Pradeep Ravikumar and Martin J. Wainwright and John D. Lafferty},
  title = {{High-dimensional Ising model selection using $\ell _{1}$-regularized logistic regression}},
  volume = {38},
  journal = {The Annals of Statistics},
  number = {3},
  publisher = {Institute of Mathematical Statistics},
  pages = {1287 -- 1319},
  year = {2010},
}

@article{sugiyama_2012,
  author          = {Sugiyama, M. and Suzuki, T. and Kanamori, T.},
  journal         = {Annals of the Institute of Statistical Mathematics},
  title           = {Density-ratio matching under the Bregman divergence: a unified framework of density-ratio estimation},
  volume          = {64},
  year            = {2012},
  pages           = {1009-1044},
}

@article{zhao_2006,
  author  = {Peng Zhao and Bin Yu},
  title   = {On Model Selection Consistency of Lasso},
  journal = {Journal of Machine Learning Research},
  year    = {2006},
  volume  = {7},
  number  = {90},
  pages   = {2541--2563},
}

@book{buhlmann_2011,
  author         = {Peter B{\"u}hlmann and Sara van de Geer},
  publisher      = {Springer Berlin, Heidelberg},
  title          = {Statistics for High-Dimensional Data},
  year           = {2011}
}

@book{Wainwright_2019,
  place={Cambridge},
  series={Cambridge Series in Statistical and Probabilistic Mathematics},
  title={High-Dimensional Statistics: A Non-Asymptotic Viewpoint},
  publisher={Cambridge University Press},
  author={Wainwright, Martin J.},
  year={2019},
  collection={Cambridge Series in Statistical and Probabilistic Mathematics}
}

@book{horn_2013,
  author         = {Roger A. Horn and Charles R. Johnson},
  editor         = {},
  publisher      = {Cambridge University Press},
  title          = {Matrix Analysis: Second Edition},
  year           = {2013}
}

@article{srivastava_2023,
  title = "Estimating the Density Ratio between Distributions with High Discrepancy using Multinomial Logistic Regression",
  author = "Akash Srivastava and Seungwook Han and Kai Xu and Benjamin Rhodes and Gutmann, {Michael U.}",
  year = "2023",
  volume = "2023",
  pages = "1--23",
  journal = "Transactions on Machine Learning Research",
  issn = "2835-8856",
  number = "3",
}

@inproceedings{choi_2022,
  author          = {Kristy Choi and Chenlin Meng and Yang Song and Stefano Ermon},
  booktitle       = {Proceedings of the 25th International Conference on Artificial Intelligence and Statistics (AISTATS) 2022},
  title           = {Density Ratio Estimation via Infinitesimal Classification},
  year            = {2022}
}

@inproceedings{choi_2021,
  author          = {Kristy Choi and Madeline Liao and Stefano Ermon},
  booktitle       = {Proceedings of the Thirty-Seventh Conference on Uncertainty in Artificial Intelligence},
  title           = {Featurized Density Ratio Estimation},
  volume          = {161},
  pages           = {172-182},
  year            = {2021}
}

@inproceedings{zhang_2023,
  author = {Zhang, Yu-Jie and Zhang, Zhen-Yu and Zhao, Peng and Sugiyama, Masashi},
  booktitle = {Advances in Neural Information Processing Systems},
  editor = {A. Oh and T. Neumann and A. Globerson and K. Saenko and M. Hardt and S. Levine},
  pages = {29074--29113},
  publisher = {Curran Associates, Inc.},
  title = {Adapting to Continuous Covariate Shift via Online Density Ratio Estimation},
  volume = {36},
  year = {2023}
}

@article{shimodaira_2000,
  title = {Improving predictive inference under covariate shift by weighting the log-likelihood function},
  journal = {Journal of Statistical Planning and Inference},
  volume = {90},
  number = {2},
  pages = {227-244},
  year = {2000},
  issn = {0378-3758},
  author = {Hidetoshi Shimodaira}
}

@article{bregman_1967,
  author          = {Bregman, L. M.},
  journal         = {USSR Computational Mathematics and Mathematical Physics},
  title           = {The relaxation method of finding the common point of convex sets and its application to the solution of problems in convex programming},
  volume          = {7},
  pages           = {200-217},
  year            = {1967}
}

@article{kanamori_2009,
    author  = {Takafumi Kanamori and Shohei Hido and Masashi Sugiyama},
    title   = {A Least-squares Approach to Direct Importance Estimation},
    journal = {Journal of Machine Learning Research},
    year    = {2009},
    volume  = {10},
    number  = {48},
    pages   = {1391--1445},
}

@article{gretton_2009,
  author          = {Arthur Gretton and Alex Smola and Jiayuan Huang and Marcel Schmittfull and Karsten Borgwardt and Bernhard Sch{\"o}lkopf},
  journal         = {Dataset Shift in Machine Learning},
  title           = {Covariate Shift by Kernel Mean Matching},
  pages           = {131-160},
  year            = {2009}
}

@book{hastie_2001,
  author         = {Trevor Hastie and Robert Tibshirani and Jerome Friedman},
  publisher      = {Springer},
  title          = {The elements of statistical learning: Data mining, inference, and prediction},
  year           = {2001},
}

@InProceedings{nagumo_2024,
  title = 	 {Density Ratio Estimation with Doubly Strong Robustness},
  author =       {Nagumo, Ryosuke and Fujisawa, Hironori},
  booktitle = 	 {Proceedings of the 41st International Conference on Machine Learning},
  pages = 	 {37260--37276},
  year = 	 {2024},
  editor = 	 {Salakhutdinov, Ruslan and Kolter, Zico and Heller, Katherine and Weller, Adrian and Oliver, Nuria and Scarlett, Jonathan and Berkenkamp, Felix},
  volume = 	 {235},
  series = 	 {Proceedings of Machine Learning Research},
  month = 	 {21--27 Jul},
  publisher =    {PMLR},
  url = 	 {https://proceedings.mlr.press/v235/nagumo24a.html}
}

@article{hoeffding_1963,
  author          = {Hoeffding, Wassily},
  journal         = {Journal of the American Statistical Association},
  number          = {301},
  title           = {Probability Inequalities for Sums of Bounded Random Variables},
  volume          = {58},
  year            = {1963},
  pages           = {13-30},
}

@article{nguyen_2013,
  author={Nguyen, Nam H. and Tran, Trac D.},
  journal={IEEE Transactions on Information Theory}, 
  title={Robust Lasso With Missing and Grossly Corrupted Observations}, 
  year={2013},
  volume={59},
  number={4},
  pages={2036-2058},
  doi={10.1109/TIT.2012.2232347}
}

@inproceedings{dalalyan_2019,
  author = {Dalalyan, Arnak and Thompson, Philip},
  booktitle = {Advances in Neural Information Processing Systems},
  editor = {H. Wallach and H. Larochelle and A. Beygelzimer and F. d\textquotesingle Alch\'{e}-Buc and E. Fox and R. Garnett},
  pages = {},
  publisher = {Curran Associates, Inc.},
  title = {Outlier-robust estimation of a sparse linear model using \textbackslash ell\_1-penalized Huber\textquotesingle s M-estimator},
  url = {https://proceedings.neurips.cc/paper_files/paper/2019/file/f0d7053396e765bf52de12133cf1afe8-Paper.pdf},
  volume = {32},
  year = {2019}
}

@article{chen_2018,
  author = {Mengjie Chen and Chao Gao and Zhao Ren},
  title = {{Robust covariance and scatter matrix estimation under Huber’s contamination model}},
  volume = {46},
  journal = {The Annals of Statistics},
  number = {5},
  publisher = {Institute of Mathematical Statistics},
  pages = {1932 -- 1960},
  year = {2018},
  doi = {10.1214/17-AOS1607},
  URL = {https://doi.org/10.1214/17-AOS1607}
}

@INPROCEEDINGS{lai_2016,
  author={Lai, Kevin A. and Rao, Anup B. and Vempala, Santosh},
  booktitle={2016 IEEE 57th Annual Symposium on Foundations of Computer Science (FOCS)}, 
  title={Agnostic Estimation of Mean and Covariance}, 
  year={2016},
  volume={},
  number={},
  pages={665-674},
  doi={10.1109/FOCS.2016.76}
}

@article{gao_2020,
  author = {Chao Gao},
  title = {{Robust regression via mutivariate regression depth}},
  volume = {26},
  journal = {Bernoulli},
  number = {2},
  publisher = {Bernoulli Society for Mathematical Statistics and Probability},
  pages = {1139 -- 1170},
  year = {2020},
  doi = {10.3150/19-BEJ1144},
  URL = {https://doi.org/10.3150/19-BEJ1144}
}

@InProceedings{chen_2013,
  title = 	 {Robust Sparse Regression under Adversarial Corruption},
  author = 	 {Chen, Yudong and Caramanis, Constantine and Mannor, Shie},
  booktitle = 	 {Proceedings of the 30th International Conference on Machine Learning},
  pages = 	 {774--782},
  year = 	 {2013},
  editor = 	 {Dasgupta, Sanjoy and McAllester, David},
  volume = 	 {28},
  number =       {3},
  series = 	 {Proceedings of Machine Learning Research},
  address = 	 {Atlanta, Georgia, USA},
  month = 	 {17--19 Jun},
  publisher =    {PMLR},
  url = 	 {https://proceedings.mlr.press/v28/chen13h.html}
}

@inproceedings{liu_2020,
  title = {High Dimensional Robust Sparse Regression},
  author = {Liu Liu and Yanyao Shen and Tianyang Li and Constantine Caramanis},
  booktitle = {Proceedings of the 23rdInternational Conference on Artificial Intelligence and Statistics (AISTATS) 2020},
  publisher = {PMLR},
  volume = {108},
  year = {2020}
}

@inproceedings{eunho_2012,
  author = {Yang, Eunho and Allen, Genevera and Liu, Zhandong and Ravikumar, Pradeep},
  booktitle = {Advances in Neural Information Processing Systems},
  editor = {F. Pereira and C.J. Burges and L. Bottou and K.Q. Weinberger},
  pages = {},
  publisher = {Curran Associates, Inc.},
  title = {Graphical Models via Generalized Linear Models},
  url = {https://proceedings.neurips.cc/paper_files/paper/2012/file/0ff8033cf9437c213ee13937b1c4c455-Paper.pdf},
  volume = {25},
  year = {2012}
}

@article{diakonikolas_2019,
  author = {Diakonikolas, Ilias and Kamath, Gautam and Kane, Daniel and Li, Jerry and Moitra, Ankur and Stewart, Alistair},
  title = {Robust Estimators in High-Dimensions Without the Computational Intractability},
  journal = {SIAM Journal on Computing},
  volume = {48},
  number = {2},
  pages = {742-864},
  year = {2019},
}

@article{oliverira_2025,
  author = {Roberto I. Oliveira and Zoraida F. Rico},
  volume = {},
  journal = {Annals of Statistics},
  title = {Improved covariance estimation: optimal robustness and sub-Gaussian guarantees under heavy tails},
  number = {},
  year = {to appear}
}

\newpage
\appendix

\section{Notations}
\label{section notations}

\subsection{Normalizing Term}

The population normalizing term ${C}^*_{\bm{\theta}}$ in the uncontaminated setting is defined as
\begin{align}
\label{eq constant of weighted dre}
    C^*_{\bm{\theta}}=\frac{\mathbb{E}_{p^*}\left[w(X)\right]}{\mathbb{E}_{q^*}\left[\exp(\bm{\theta}^T h(X))w(X)\right]}.
\end{align}
The empirical normalizing term $\hat{C}^*_{\bm{\theta}}$ in the uncontaminated setting is defined as
\begin{align}
\label{eq empirical c}
    \hat{C}^*_{\bm{\theta}}=\frac{\hat{\mathbb{E}}_{p^*}\left[w(X)\right]}{\hat{\mathbb{E}}_{q^*}\left[\exp\left(\bm{\theta}^Th(X)\right)w(X)\right]}.
\end{align}
In the contaminated setting, the empirical normalizing term $\hat{C}^\dagger_{\bm{\theta}}$ is defined as
\begin{align}
\label{eq hat c dagger}
    \hat{C}^\dagger_{\bm{\theta}}&=\frac{\hat{\mathbb{E}}_{p^\dagger} [w(X)]}{\hat{\mathbb{E}}_{q^\dagger} [\exp(\bm{\theta}^T h(X))w(X)]}.
\end{align}

\subsection{Derivatives of Objective Function}

We should consider the gradient of $\mathcal{L}^\dagger(\bm{\theta})$:
\begin{align}
\label{eq 1st derivative of objective in contaminated}
    \nabla\mathcal{L}^\dagger(\bm{\theta})
    =-\hat{\mathbb{E}}_{p^\dagger} \left[h(X)w(X)\right]
    +\hat{\mathbb{E}}_{q^\dagger}\left[r\left(X;\bm{\theta},\hat{C}^\dagger_{\bm{\theta}}\right)w(X)h(X)\right].
\end{align}
The second derivative in the contaminated setting, $\nabla^2\mathcal{L}^\dagger(\bm{\theta})\in \mathbb{R}^{p \times p}$, can be written as
\begin{equation}
\label{eq 2nd derivative of objective in comtaminated}
\begin{split}
    \nabla^2\mathcal{L}^\dagger(\bm{\theta})
    &=\hat{\mathbb{E}}_{q^\dagger}\left[r\left(X;\bm{\theta},\hat{C}^\dagger_{\bm{\theta}}\right)w(X)h(X)h(X)^T\right] \\
    &\quad-\frac{1}{\hat{\kappa}^\dagger}\hat{\mathbb{E}}_{q^\dagger}\left[r\left(X;\bm{\theta},\hat{C}^\dagger_{\bm{\theta}}\right)w(X)h(X)\right]\hat{\mathbb{E}}_{q^\dagger}\left[r\left(X;\bm{\theta},\hat{C}^\dagger_{\bm{\theta}}\right)w(X)h(X)\right]^T,
\end{split}
\end{equation}
where $\hat{\kappa}^\dagger = \hat{\mathbb{E}}_{p^\dagger}\left[w(\bm{X})\right]\in \mathbb{R}_+$.
Moreover, we should consider the element-wise third derivative of the objective function:
\begin{equation}
\label{eq 3rd derivative of objective in contaminated}
\begin{split}
    &\nabla_{t}\nabla^2\mathcal{L}^\dagger(\bm{\theta}) \\
    &=\hat{\mathbb{E}}_{q^\dagger}\left[r\left(X;\bm{\theta},\hat{C}^\dagger_{\bm{\theta}}\right)w(X)h_t(X)h(X)h(X)^T\right] \\
    &\quad-\frac{1}{\hat{\kappa}^\dagger}\hat{\mathbb{E}}_{q^\dagger}\left[r\left(X;\bm{\theta},\hat{C}^\dagger_{\bm{\theta}}\right)w(X)h_t(X)\right]\hat{\mathbb{E}}_{q^\dagger}\left[r\left(X;\bm{\theta},\hat{C}^\dagger_{\bm{\theta}}\right)w(X)h(X)h(X)^T\right] \\
    &\quad-\frac{1}{\hat{\kappa}^\dagger}\hat{\mathbb{E}}_{q^\dagger}\left[r\left(X;\bm{\theta},\hat{C}^\dagger_{\bm{\theta}}\right)w(X)h_t(X)h(X)\right]\hat{\mathbb{E}}_{q^\dagger}\left[r\left(X;\bm{\theta},\hat{C}^\dagger_{\bm{\theta}}\right)w(X)h(X)\right]^T \\
    &\quad-\frac{1}{\hat{\kappa}^\dagger}\hat{\mathbb{E}}_{q^\dagger}\left[r\left(X;\bm{\theta},\hat{C}^\dagger_{\bm{\theta}}\right)w(X)h(X)\right]\hat{\mathbb{E}}_{q^\dagger}\left[r\left(X;\bm{\theta},\hat{C}^\dagger_{\bm{\theta}}\right)w(X)h_t(X)h(X)\right]^T \\
    &\quad+\frac{2}{\hat{\kappa}^{\dagger 2}}\hat{\mathbb{E}}_{q^\dagger}\left[r\left(X;\bm{\theta},\hat{C}^\dagger_{\bm{\theta}}\right)w(X)h_t(X)\right]\hat{\mathbb{E}}_{q^\dagger}\left[r\left(X;\bm{\theta},\hat{C}^\dagger_{\bm{\theta}}\right)w(X)h(X)\right]\hat{\mathbb{E}}_{q^\dagger}\left[r\left(X;\bm{\theta},\hat{C}^\dagger_{\bm{\theta}}\right)w(X)h(X)\right]^T
\end{split}
\end{equation}

We can define the first, second, and element-wise third derivative of the uncontaminated objective function $\mathcal{L}^*(\bm{\theta})$ in the same manner.
The first derivative is
\begin{align*}
    \nabla\mathcal{L}^*(\bm{\theta})
    =-\hat{\mathbb{E}}_{p^*} \left[h(X)w(X)\right]
    +\hat{\mathbb{E}}_{q^*}\left[r\left(X;\bm{\theta},\hat{C}^*_{\bm{\theta}}\right)w(X)h(X)\right].
\end{align*}
The second derivative is
\begin{equation}
\label{eq 2nd derivative}
\begin{split}
    \nabla^2\mathcal{L}^*(\bm{\theta})
    &=\hat{\mathbb{E}}_{q^*}\left[r\left(X;\bm{\theta},\hat{C}^*_{\bm{\theta}}\right)w(X)h(X)h(X)^T\right] \\
    &\quad-\frac{1}{\hat{\kappa}^*}\hat{\mathbb{E}}_{q^*}\left[r\left(X;\bm{\theta},\hat{C}^*_{\bm{\theta}}\right)w(X)h(X)\right]\hat{\mathbb{E}}_{q^*}\left[r\left(X;\bm{\theta},\hat{C}^*_{\bm{\theta}}\right)w(X)h(X)\right]^T.
\end{split}
\end{equation}
where $\hat{\kappa}^*=\hat{\mathbb{E}}_{p^*}[w(X)]$.
The element-wise third derivative is
\begin{equation}
\label{eq 3rd derivative}
\begin{split}
    &\nabla_{t}\nabla^2\mathcal{L}^*(\bm{\theta}) \\
    &=\hat{\mathbb{E}}_{q^*}\left[r\left(X;\bm{\theta},\hat{C}^*_{\bm{\theta}}\right)w(X)h_t(X)h(X)h(X)^T\right] \\
    &\quad-\frac{1}{\hat{\kappa}^*}\hat{\mathbb{E}}_{q^*}\left[r\left(X;\bm{\theta},\hat{C}^*_{\bm{\theta}}\right)w(X)h_t(X)\right]\hat{\mathbb{E}}_{q^*}\left[r\left(X;\bm{\theta},\hat{C}^*_{\bm{\theta}}\right)w(X)h(X)h(X)^T\right] \\
    &\quad-\frac{1}{\hat{\kappa}^*}\hat{\mathbb{E}}_{q^*}\left[r\left(X;\bm{\theta},\hat{C}^*_{\bm{\theta}}\right)w(X)h_t(X)h(X)\right]\hat{\mathbb{E}}_{q^*}\left[r\left(X;\bm{\theta},\hat{C}^*_{\bm{\theta}}\right)w(X)h(X)\right]^T \\
    &\quad-\frac{1}{\hat{\kappa}^*}\hat{\mathbb{E}}_{q^*}\left[r\left(X;\bm{\theta},\hat{C}^*_{\bm{\theta}}\right)w(X)h(X)\right]\hat{\mathbb{E}}_{q^*}\left[r\left(X;\bm{\theta},\hat{C}^*_{\bm{\theta}}\right)w(X)h_t(X)h(X)\right]^T \\
    &\quad+\frac{2}{\hat{\kappa}^{* 2}}\hat{\mathbb{E}}_{q^*}\left[r\left(X;\bm{\theta},\hat{C}^*_{\bm{\theta}}\right)w(X)h_t(X)\right]\hat{\mathbb{E}}_{q^*}\left[r\left(X;\bm{\theta},\hat{C}^*_{\bm{\theta}}\right)w(X)h(X)\right]\hat{\mathbb{E}}_{q^*}\left[r\left(X;\bm{\theta},\hat{C}^*_{\bm{\theta}}\right)w(X)h(X)\right]^T
\end{split}
\end{equation}

\subsection{Sufficiently Small Value}

We prepare propositions about the eigenvalues and the max norm of a matrix with a sufficiently small value.

\begin{proposition}
\label{proposition of eigenvalues of a small matrix}
Let $A = O_{n\times n}(\chi) \in \mathbb{R}^{n\times n}$ with a sufficiently small value $\chi$.
Then,
\begin{align*}
    \Lambda_{\max}[A] = O(n \chi), \quad
    \Lambda_{\min}[A] = O(n \chi)
\end{align*}
\end{proposition}

\begin{proof}[Proof of Proposition \ref{proposition of eigenvalues of a small matrix}]
\begin{align*}
    \Lambda_{\min}[A] \le \Lambda_{\max}[A] \le {\rm tr}(A) = O(n\chi).
\end{align*}
\end{proof}

\begin{proposition}
\label{proposition of max norm of a small matrix}
Let $A = O_{m\times n}(\chi) \in \mathbb{R}^{m\times n}$ with a sufficiently small value $\chi$.
Then,
\begin{align*}
    \|A\|_\infty = O(n\chi).
\end{align*}
\end{proposition}

\begin{proof}[Proof of Proposition \ref{proposition of max norm of a small matrix}]
\begin{align*}
    \|A\|_\infty &= \max_{i=1,...,m}\sum_{j=1}^{n}|A_{ij}|=O(n\chi).
\end{align*}
\end{proof}

\section{Propositions from Assumption \ref{assumption weight for normal}}
\label{section discussion of assumption of weight function}

\subsection{Propositions}

We list some propositions that can be derived from Assumption \ref{assumption weight for normal}.
The proofs are provided in Appendix \ref{proofs of propositions from assumption weight}.

\begin{proposition}
\label{proposition of integral of weight}
If Assumption \ref{assumption weight for normal} holds, then
\begin{align*}
    0 < W'_{\max} = {\mathbb{E}}_{p^*}\left[w(X)\right] < \infty.
\end{align*}
\end{proposition}

\begin{proposition}
\label{proposition of bounded sum of weight}
If Assumption \ref{assumption weight for normal} holds, then
\begin{align*}
    W'_{\max}-\tau \le \hat{\kappa}^* = \hat{\mathbb{E}}_{p^*}\left[w(X)\right]\le W'_{\max}+\tau
\end{align*}
holds with probability at least $1-\delta_{\tau}$, where
\begin{align*}
    \tau=\sqrt{\frac{W_{\max}^2\log(2/\delta_{\tau})}{2n_p^*}}.
\end{align*}
\end{proposition}

\subsection{Proofs}
\label{proofs of propositions from assumption weight}

\begin{proof}[Proof of Proposition \ref{proposition of integral of weight}]
From Assumption \ref{assumption weight for normal}, we have
\begin{align*}
    0<{\mathbb{E}}_{p^*}\left[w(X)\right] \le {\mathbb{E}}_{p^*}\left[W_{\max}\right] = W_{\max} < \infty.
\end{align*}
\end{proof}

\begin{proof}[Proof of Proposition \ref{proposition of bounded sum of weight}]
From Assumption \ref{assumption weight for normal} and Proposition \ref{proposition of integral of weight}, $w(X)-W'_{\max}$ is a bounded zero-mean random variable.
Hoeffding's inequality provides the exponentially decaying tail behavior of the bounded zero-mean random variable \cite{Wainwright_2019}:
\begin{align*}
    P\left(\left|\hat{\mathbb{E}}_{p^*}[w(X)]-W'_{\max}\right|\ge \tau\right) \le 2\exp\left(-\frac{2n_p^*\tau^2}{W_{\max}^2}\right).
\end{align*}
By defining the right hand side as $\delta_{\tau}$, we have the formulation of $\tau$.
\end{proof}

\section{Propositions from Assumption \ref{assumption-9}}
\label{section discussion of assumption-9}

\subsection{Propositions}

We list some propositions that can be derived from Assumption \ref{assumption-9}.
The proofs are provided in Appendix \ref{subsection proof of propositions from assumption-9}.

\begin{proposition}
\label{proposition of bounded dre function}
    If Assumptions \ref{assumption weight for normal} and \ref{assumption-9} hold, for any $\theta \in \Theta$,
    \begin{align*}
        \frac{W_{\min}}{E_{\max}} \le C^*_{\bm{\theta}} \le \frac{W_{\max}}{E_{\min}}, \quad
        0<r(\bm{x}; \bm{\theta},C^*_{\bm{\theta}})w(\bm{x}) \le E'_{\max}, \quad
        \mathbb{E}_{q^*}\left[r(X; \bm{\theta},C^*_{\bm{\theta}})w(X)\right] = W'_{\max},
    \end{align*}
    where $E'_{\max}=W_{\max}E_{\max}/E_{\min}$.
\end{proposition}

\begin{proposition}
\label{proposition of lower bound of dre}
    If Assumptions \ref{assumption weight for normal} and \ref{assumption-9} hold,
    \begin{align*}
        W'_{\max}-\epsilon \le \hat{\mathbb{E}}_{q^*}\left[r\left(X;\bm{\theta},C^*_{\bm{\theta}}\right)w(X)\right] \le W'_{\max}+\epsilon
    \end{align*}
    holds for any $\bm{\theta}\in\Theta$ with probability at least $1-\delta_{\epsilon}$, where
    \begin{align*}
        \epsilon=\sqrt{\frac{E'^2_{\max}\log(2/\delta_{\epsilon})}{2n_q^*}}.
    \end{align*}
\end{proposition}

\subsection{Proofs}
\label{subsection proof of propositions from assumption-9}

\begin{proof}[Proof of Proposition \ref{proposition of bounded dre function}]
    From Assumptions \ref{assumption weight for normal} and \ref{assumption-9}, the normalizing term \eqref{eq constant of weighted dre} is bounded by
    \begin{align*}
        \frac{W_{\min}}{E_{\max}} \le C^*_{\bm{\theta}}=\frac{\mathbb{E}_{p^*}[w(X)]}{\mathbb{E}_{q^*}[\exp(\bm{\theta}^{T}X)w(X)]} \le \frac{W_{\max}}{E_{\min}}.
    \end{align*}
    Then,
    \begin{align*}
        r(\bm{x}; \bm{\theta},C^*_{\bm{\theta}})w(\bm{x})
        &=C^*_{\bm{\theta}} \exp \left( \bm{\theta}^T h(\bm{x}) \right)w(\bm{x}) \le \frac{W_{\max}E_{\max}}{E_{\min}}, \\
        \mathbb{E}_{q^*}\left[r(X; \bm{\theta},C^*_{\bm{\theta}})w(X)\right]
        &= C^*_{\bm{\theta}} \mathbb{E}_{q^*}\left[\exp \left( \bm{\theta}^T h(X) \right)w(X)\right] = \mathbb{E}_{p^*}[w(X)] = W'_{\max}.
    \end{align*}
    for any $\theta \in \Theta$.
\end{proof}

\begin{proof}[Proof of Proposition \ref{proposition of lower bound of dre}]
    From Proposition \ref{proposition of bounded dre function}, $r(X;\bm{\theta},C^*_{\bm{\theta}})w(X)-\mathbb{E}_{q^*}\left[r(X; \bm{\theta},C^*_{\bm{\theta}})w(X)\right]$ is a bounded zero-mean random variable.
    Hoeffding's inequality suggests the tail probability of such a random variable as
    \begin{align*}
        P\left(\left|\hat{\mathbb{E}}_{q^*}[r(X;\bm{\theta},C^*_{\bm{\theta}})w(X)]-W'_{\max}\right|\ge\epsilon\right)
        \le 2\exp\left(-\frac{2n_q^*}{E'^2_{\max}}\epsilon^2\right).
    \end{align*}
    By equalizing the right hand side with $\delta_{\epsilon}$, we have the formulation of $\epsilon$.
\end{proof}

\section{Proposition from Assumption \ref{assumption-8}}
\label{subsection discussion of assumption 8}

\begin{proposition}
\label{proposition bound of features and density ratio function}
    If Assumptions \ref{assumption weight for normal}, \ref{assumption-9}, and \ref{assumption-8} hold, for any $\bm{\theta}\in \Theta$ and $t\in \mathcal{E}$,
    \begin{align*}
        \left\|h(\bm{x})           r(\bm{x};\bm{\theta},C^*_{\bm{\theta}})w(\bm{x})\right\|_{\infty}&\le D'_{\max}, \\
        \left\|h(\bm{x})h(\bm{x})^Tr(\bm{x};\bm{\theta},C^*_{\bm{\theta}})w(\bm{x})\right\|_{\max}  &\le D'_{\max}, \\
        \left\|h_t(\bm{x})h_S(\bm{x})h_S(\bm{x})^Tr(\bm{x};\bm{\theta},C^*_{\bm{\theta}})w(\bm{x})\right\|_{\max}&\le D'_{\max},
    \end{align*}
    where $D'_{\max}=W_{\max}D_{\max}/E_{\min}$.
\end{proposition}

\begin{proof}[Proof of Proposition \ref{proposition bound of features and density ratio function}]
From Proposition \ref{proposition of bounded dre function} and Assumption \ref{assumption-8},
\begin{align*}
    \left\|h(\bm{x})r(\bm{x};\bm{\theta},C^*_{\bm{\theta}})w(\bm{x})\right\|_{\infty}
    = C^*_{\bm{\theta}} \left\|h(\bm{x})\exp(\bm{\theta}^Th(\bm{x}))w(\bm{x})\right\|_{\infty}
    \le \frac{W_{\max}D_{\max}}{E_{\min}}.
\end{align*}
The other upper-bounds are given by the same way.
\end{proof}

\section{Propositions of boundedness}
\label{section propositions of boundedness}

\subsection{Lemma}

The following lemma suggests the boundedness of $\tau$ and $\epsilon$ when the sample size is large.

\begin{lemma}
\label{lemma boundedness of tau and epsilon}
If $n_{p,q}^* \ge N_{\delta}$ holds, we have
\begin{align*}
    \tau \le \frac{W'_{\max}}{2} \quad{\rm and}\quad
    \epsilon \le \frac{W'_{\max}}{2},
\end{align*}
where 
\begin{align*}
    N_{\delta}=\frac{2V_{\max}\log(2/\delta_{\min})}{W'^2_{\max}}, \quad
    V_{\max}=\max\{W_{\max}^2, E'^2_{\max}\}, \quad
    \delta_{\min}=\min\{\delta_{\tau},\delta_{\epsilon}\}.
\end{align*}
\end{lemma}

\begin{proof}[Proof of Lemma \ref{lemma boundedness of tau and epsilon}]
From the definition of $\tau$ and $\epsilon$ in Propositions \ref{proposition of bounded sum of weight} and \ref{proposition of lower bound of dre}, respectively, we have
\begin{align*}
    \max\{\tau, \epsilon\}
    &=\max\left\{\sqrt{\frac{W_{\max}^2\log(2/\delta_{\tau})}{2n_p^*}}, \sqrt{\frac{E'^2_{\max}\log(2/\delta_{\epsilon})}{2n_q^*}}\right\} \\
    &\le\sqrt{\frac{\max\{W_{\max}^2, E'^2_{\max}\} \log(2/\min\{\delta_{\tau}, \delta_{\epsilon}\})}{2\min\{n_p^*, n_q^*\}}} \\
    &= \sqrt{\frac{V_{\max} \log(2/\delta_{\min})}{2n_{p,q}^*}},
\end{align*}
where $V_{\max}=\max\{W_{\max}^2, E'^2_{\max}\}$ and $\delta_{\min}=\min\{\delta_{\tau},\delta_{\epsilon}\}$.
Then, if we set
\begin{align*}
    n_{p,q}^* \ge \frac{2V_{\max} \log(2/\delta_{\min})}{W'^2_{\max}},
\end{align*}
we have
\begin{align*}
    \tau \le \frac{W'_{\max}}{2} \quad{\rm and}\quad
    \epsilon \le \frac{W'_{\max}}{2}.
\end{align*}
\end{proof}
    
\subsection{Propositions}

We list some propositions which can be derived from Lemma \ref{lemma boundedness of tau and epsilon}.
Proofs of these propositions are provided in Appendix \ref{subsection proofs of propositions about boundedness}.

\begin{proposition}
\label{proposition bound of exp(theta h(x))w(x)}
If Assumptions \ref{assumption weight for normal} and \ref{assumption-9} hold and $n_{p,q}^* \ge N_{\delta}$ holds, we have
\begin{align*}
    \frac{W'_{\max}}{2} \le \hat{\mathbb{E}}_{p^*}[w(X)]\le \frac{3W'_{\max}}{2}, \\
    \frac{W'_{\max}}{2} \le \hat{\mathbb{E}}_{q^*}\left[r\left(X;\bm{\theta},C^*_{\bm{\theta}}\right)w(X)\right] \le \frac{3W'_{\max}}{2}, \\
    \hat{\mathbb{E}}_{q^*}\left[\exp(\bm{\theta}^T h(X))w(X)\right] \ge \frac{E_{\min}W'_{\max}}{2W_{\max}},
\end{align*}
with probability at least $1-\delta_{\tau}$, $1-\delta_{\epsilon}$, and $1-\delta_{\epsilon}$, respectively.
\end{proposition}

\begin{proposition}
\label{proposition bound of sample C}
If Assumptions \ref{assumption weight for normal} and \ref{assumption-9} hold and $n_{p,q}^* \ge N_{\delta}$ holds,
\begin{align*}
    \hat{C}^*_{\bm{\theta}} \le \frac{3W_{\max}}{E_{\min}},
\end{align*}
with probability at least $1-\delta_{\tau}-\delta_{\epsilon}$.
\end{proposition}

\begin{proposition}
\label{proposition-bounded empirical moment}
    If Assumptions \ref{assumption weight for normal}, \ref{assumption-9}, and \ref{assumption-8} hold and $n_{p,q}^* \ge N_{\delta}$ holds,
    \begin{align*}
        \left\|h(\bm{x})r(\bm{x};\bm{\theta},\hat{C}^*_{\bm{\theta}})w(\bm{x})\right\|_{\infty}&\le D''_{\max}
    \end{align*}
    for any $\bm{\theta}\in\Theta$ with probability at least $1-\delta_{\tau}-\delta_{\epsilon}$, where $D''_{\max}=3D'_{\max}$.
\end{proposition}

\subsection{Proofs}
\label{subsection proofs of propositions about boundedness}

\begin{proof}[Proof of Proposition \ref{proposition bound of exp(theta h(x))w(x)}]
From Proposition \ref{proposition of bounded sum of weight} and Lemma \ref{lemma boundedness of tau and epsilon}, we have
\begin{align*}
    \frac{W'_{\max}}{2} \le W'_{\max}-\tau \le \hat{\mathbb{E}}_{p^*}[w(X)]\le W'_{\max}+\tau \le \frac{3W'_{\max}}{2}
\end{align*}
with probability at least $1-\delta_{\tau}$.
Similarly, from Proposition \ref{proposition of lower bound of dre} and Lemma \ref{lemma boundedness of tau and epsilon}, we have
\begin{align*}
    \frac{W'_{\max}}{2} \le W'_{\max}-\epsilon \le \hat{\mathbb{E}}_{q^*}\left[r\left(X;\bm{\theta},C^*_{\bm{\theta}}\right)w(X)\right] \le W'_{\max}+\epsilon  \le \frac{3W'_{\max}}{2}.
\end{align*}
with probability at least $1-\delta_{\epsilon}$.

Because we have
\begin{align*}
    \hat{\mathbb{E}}_{q^*}\left[r\left(X;\bm{\theta},C^*_{\bm{\theta}}\right)w(X)\right]
    = C^*_{\bm{\theta}} \hat{\mathbb{E}}_{q^*}\left[\exp(\bm{\theta}^T h(X))w(X)\right],
\end{align*}
from Propositions \ref{proposition of lower bound of dre} and \ref{proposition of bounded dre function} and the above proof, we have
\begin{align*}
    \hat{\mathbb{E}}_{q^*}\left[\exp(\bm{\theta}^T h(X))w(X)\right]
    = \frac{\hat{\mathbb{E}}_{q^*}\left[r\left(X;\bm{\theta},C^*_{\bm{\theta}}\right)w(X)\right]}{C^*_{\bm{\theta}}}
    \ge \frac{E_{\min}W'_{\max}}{2W_{\max}}
\end{align*}
with probability at least $1-\delta_{\epsilon}$.
\end{proof}

\begin{proof}[Proof of Proposition \ref{proposition bound of sample C}]
From Proposition \ref{proposition bound of exp(theta h(x))w(x)}, we have
\begin{align*}
    \hat{C}^*_{\bm{\theta}}
    =\frac{\hat{\mathbb{E}}_{p^*}\left[w(X)\right]}{\hat{\mathbb{E}}_{q^*}[\exp(\bm{\theta}^T h(X))w(X)]}
    \le \frac{2W_{\max}}{E_{\min}W'_{\max}}\frac{3W'_{\max}}{2}
    \le \frac{3W_{\max}}{E_{\min}}
\end{align*}
with probability at least $1-\delta_{\tau}-\delta_{\epsilon}$.
\end{proof}

\begin{proof}[Proof of Proposition \ref{proposition-bounded empirical moment}]
From \eqref{eq parametric dre} and \eqref{eq empirical c}, we have
\begin{align*}
    r(\bm{x};\bm{\theta},\hat{C}^*_{\bm{\theta}})
    =\frac{\hat{\mathbb{E}}_{p^*}\left[w(X)\right]}{\hat{\mathbb{E}}_{q^*}[r(X;\bm{\theta},C^*_{\bm{\theta}})w(X)]} r(\bm{x}; \bm{\theta}, C^*_{\bm{\theta}}).
\end{align*}
Therefore,
\begin{align*}
    \left\|h(\bm{x})r(\bm{x};\bm{\theta},\hat{C}^*_{\bm{\theta}})w(\bm{x})\right\|_{\infty}
    =\frac{\hat{\mathbb{E}}_{p^*}\left[w(X)\right]}{\hat{\mathbb{E}}_{q^*}[r(X;\bm{\theta},C^*_{\bm{\theta}})w(X)]}\left\|h(\bm{x})r(\bm{x};\bm{\theta},C^*_{\bm{\theta}})w(\bm{x})\right\|_{\infty}.
\end{align*}
Then, from Proposition \ref{proposition bound of exp(theta h(x))w(x)}, we have
\begin{align*}
    \left\|h(\bm{x})r(\bm{x};\bm{\theta},\hat{C}^*_{\bm{\theta}})w(\bm{x})\right\|_{\infty}
    \le \frac{3W'_{\max}/2}{W'_{\max}/2}D'_{\max}
    = 3D'_{\max},
\end{align*}
with probability at least $1-\delta_{\tau}-\delta_{\epsilon}$.
\end{proof}

\section{Proof of Theorem \ref{theorem of robust objective function}}
\label{section proof of theorem of robust objective function}

\subsection{Notations}
\label{subsection notation rho}
Let
\begin{align*}
    &\rho_1=\frac{1}{\varepsilon_pn_p} \sum_{n=1}^{\varepsilon_pn_p}w(\bm{x}^{(\delta_p)}_n), \quad
    \rho_2=\frac{1}{\varepsilon_qn_q} \sum_{n=1}^{\varepsilon_qn_q}\exp(\bm{\theta}^T h(\bm{x}_n^{(\delta_q)}))w(\bm{x}^{(\delta_q)}_n), \quad
    \rho_3=\frac{1}{\varepsilon_pn_p} \sum_{n=1}^{\varepsilon_pn_p}h(\bm{x}_n^{(\delta_p)})w(\bm{x}^{(\delta_p)}_n), \\
    &\rho_4=\frac{1}{\varepsilon_qn_q} \sum_{n=1}^{\varepsilon_qn_q}\exp(\bm{\theta}^T h(\bm{x}_n^{(\delta_q)}))w(\bm{x}^{(\delta_q)}_n)h(\bm{x}_n^{(\delta_q)}),\quad
    \rho_5 =\frac{1}{\varepsilon_qn_q} \sum_{n=1}^{\varepsilon_qn_q}\exp(\bm{\theta}^T h(\bm{x}_n^{(\delta_q)}))w(\bm{x}^{(\delta_q)}_n)h(\bm{x}_n^{(\delta_q)}) h(\bm{x}_n^{(\delta_q)})^T, \\
    &\rho_{6,t}=\frac{1}{\varepsilon_qn_q} \sum_{n=1}^{\varepsilon_qn_q}\exp(\bm{\theta}^T h(\bm{x}_n^{(\delta_q)}))w(\bm{x}^{(\delta_q)}_n)h_t(\bm{x}_n^{(\delta_q)})h_S(\bm{x}_n^{(\delta_q)})h_S(\bm{x}_n^{(\delta_q)})^T,
\end{align*}
and $\rho_6 = (\rho_{6,t})_{t\in \mathcal{E}}$.
Note that $|\rho_{jb}| \le \nu_j \le \nu$ for $j=1,...,6$ and $b\in \mathcal{B}_j$, where $\mathcal{B}_j$ is the index set of the elements of $\rho_j$.

\subsection{Lemma}

The following lemma shows the relationship of the empirical normalizing term in the clean and contaminated settings.

\begin{lemma}
\label{lemma normalizing term robustness}
If Assumptions \ref{assumption weight for normal}, \ref{assumption-9}, and \ref{assumption weight for outlier} hold and $n_{p,q}^* \ge N_{\delta}$ holds, we have
\begin{equation*}
    \hat{C}^\dagger_{\bm{\theta}} = \frac{1-\varepsilon_p}{1-\varepsilon_q}\hat{C}^*_{\bm{\theta}}+O(\varepsilon \nu)
\end{equation*}
with probability at least $1-\delta_{\tau}-\delta_{\epsilon}$.
\end{lemma}

\begin{proof}[Proof of Lemma \ref{lemma normalizing term robustness}]
From \eqref{eq hat c dagger}, we have
\begin{align*}
    \hat{C}^\dagger_{\bm{\theta}}&=\frac{(1-\varepsilon_p)\hat{\mathbb{E}}_{p^*}\left[w(X)\right]+\varepsilon_p\rho_1}{(1-\varepsilon_q)\hat{\mathbb{E}}_{q^*}[\exp(\bm{\theta}^T h(X))w(X)]+\varepsilon_q\rho_2}.
\end{align*}
From Assumption \ref{assumption weight for outlier}, we assume that $\varepsilon_p\rho_1$ and $\varepsilon_q\rho_2$ are bounded by $\varepsilon\nu$, which is sufficiently small.
Because we assume that $n_{p,q}^* \ge N_{\delta}$ holds, from Proposition \ref{proposition bound of exp(theta h(x))w(x)}, we can assume that $\hat{\mathbb{E}}_{p^*}\left[w(X)\right]$ and $\hat{\mathbb{E}}_{q^*}[\exp(\bm{\theta}^T h(X))w(X)]$ are bounded with probability at least $1-\delta_{\tau}$ and $1-\delta_{\epsilon}$, respectively.
When considering the Taylor expansion, we have
\begin{align*}
    &\frac{1}{(1-\varepsilon_q)\hat{\mathbb{E}}_{q^*}[\exp(\bm{\theta}^T h(X))w(X)]+\varepsilon_q\rho_2} \\
    &=\frac{1}{(1-\varepsilon_q)\hat{\mathbb{E}}_{q^*}[\exp(\bm{\theta}^T h(X))w(X)]}-\frac{1}{\left\{(1-\varepsilon_q)\hat{\mathbb{E}}_{q^*}[\exp(\bm{\theta}^T h(X))w(X)]\right\}^2}\varepsilon_q\rho_2+O\left((\varepsilon_q\rho_2)^2\right)
\end{align*}
Therefore, we have
\begin{align*}
    \hat{C}^\dagger_{\bm{\theta}} = \frac{1-\varepsilon_p}{1-\varepsilon_q}\hat{C}^*_{\bm{\theta}}+O(\varepsilon \nu)
\end{align*}
with probability at least $1-\delta_{\tau}-\delta_{\epsilon}$.
\end{proof}

\subsection{Proof of Theorem \ref{theorem of robust objective function}}

We first consider the robustness of the derivative of $\mathcal{L}^\dagger(\bm{\theta})$.
From \eqref{eq 1st derivative of objective in contaminated} and Lemma \ref{lemma normalizing term robustness}, we have
\begin{align*}
    \nabla\mathcal{L}^\dagger(\bm{\theta})
    &=-\left\{(1-\varepsilon_p)\hat{\mathbb{E}}_{p^*}[h(X)w(X)] +\varepsilon_p\rho_3\right\} \\
    &\qquad+\left\{\frac{1-\varepsilon_p}{1-\varepsilon_q}\hat{C}^*_{\bm{\theta}}+O(\varepsilon \nu)\right\}
    \left\{(1-\varepsilon_q)\hat{\mathbb{E}}_{q^*} [\exp(\bm{\theta}^T h(X))w(X)h(X)]+\varepsilon_q\rho_4\right\}
\end{align*}
with probability at least $1-\delta_{\tau}-\delta_{\epsilon}$.
Proposition \ref{proposition bound of sample C} suggests that $\hat{C}_{\bm{\theta}}^*$ is bounded with probability at least $1-\delta_{\tau}-\delta_{\epsilon}$, where the condition of $n_{p,q}^*\ge N_{\delta}$ in Proposition \ref{proposition bound of sample C} holds from the assumption of this theorem.
Assumption \ref{assumption-8} suggests that $h(X)w(X)$ and $\exp(\bm{\theta}^T h(X))w(X)h(X)$ are bounded.
Then, because $\varepsilon_p\rho_3$ and $\varepsilon_q\rho_4$ are bounded by $\varepsilon\nu$ from Assumption \ref{assumption weight for outlier}, we have
\begin{align*}
    \nabla\mathcal{L}^\dagger(\bm{\theta})
    =(1-\varepsilon_p)\nabla\mathcal{L}^*(\bm{\theta})+O\left(\varepsilon \nu\right)
\end{align*}
with probability at least $1-\delta_{\tau}-\delta_{\epsilon}$.

The second part is robust estimation of the second derivative.
From Lemma \ref{lemma normalizing term robustness}, we have
\begin{align*}
    &\hat{\kappa}^\dagger=(1-\varepsilon_p)\hat{\mathbb{E}}_{p^*}[w(X)]+\varepsilon_p\rho_1, \\
    &\hat{\mathbb{E}}_{q^\dagger}\left[r\left(X;\bm{\theta},\hat{C}^\dagger_{\bm{\theta}}\right)w(X)h(X)\right] \\
    &\quad=\left\{\frac{1-\varepsilon_p}{1-\varepsilon_q}\hat{C}^*_{\bm{\theta}}+O(\varepsilon \nu)\right\}\left\{(1-\varepsilon_q)\hat{\mathbb{E}}_{q^*}\left[\exp\left(\bm{\theta}^Th(X)\right)w(X)h(X)\right] +\varepsilon_q\rho_4\right\}, \\
    &\hat{\mathbb{E}}_{q^\dagger}\left[r\left(X;\bm{\theta},\hat{C}^\dagger_{\bm{\theta}}\right)w(X)h(X)h(X)^T\right] \\
    &\quad=\left\{\frac{1-\varepsilon_p}{1-\varepsilon_q}\hat{C}^*_{\bm{\theta}}+O(\varepsilon \nu)\right\}\left\{(1-\varepsilon_q)\hat{\mathbb{E}}_{q^*}\left[\exp\left(\bm{\theta}^Th(X)\right)w(X)h(X)h(X)^T\right]+\varepsilon_q\rho_5\right\},
\end{align*}
where each empirical integral is bounded from Assumptions \ref{assumption weight for normal} and \ref{assumption-8}.
Then, similar to the above discussion, from \eqref{eq 2nd derivative of objective in comtaminated}, we have
\begin{align*}
    \nabla^2\mathcal{L}^\dagger(\bm{\theta})=(1-\varepsilon_p)\nabla^2\mathcal{L}^*(\bm{\theta})+O\left(\varepsilon \nu\right)
\end{align*}
with probability at least $1-\delta_{\tau}-\delta_{\epsilon}$.

Lastly, we consider the third derivative.
From \eqref{eq 3rd derivative of objective in contaminated}, we have
\begin{align*}
    &\hat{\mathbb{E}}_{q^\dagger}\left[r\left(X;\bm{\theta},\hat{C}^\dagger_{\bm{\theta}}\right)w(X)h_t(X)\right] \\
    &\quad=\left\{\frac{1-\varepsilon_p}{1-\varepsilon_q}\hat{C}^*_{\bm{\theta}}+O(\varepsilon \nu)\right\}\left\{(1-\varepsilon_q)\hat{\mathbb{E}}_{q^*}[\exp\left(\bm{\theta}^Th(X)\right)w(X)h_t(X)]+(\varepsilon_q\rho_4)_t\right\}, \\
    &\hat{\mathbb{E}}_{q^\dagger}\left[r\left(X;\bm{\theta},\hat{C}^\dagger_{\bm{\theta}}\right)w(X)h_t(X)h_S(X)\right] \\
    &\quad=\left\{\frac{1-\varepsilon_p}{1-\varepsilon_q}\hat{C}^*_{\bm{\theta}}+O(\varepsilon \nu)\right\}\left\{(1-\varepsilon_q)\hat{\mathbb{E}}_{q^*}[\exp\left(\bm{\theta}^Th(X)\right)w(X)h_t(X)h_S(X)]+(\varepsilon_q\rho_5)_t\right\}, \\
    &\hat{\mathbb{E}}_{q^\dagger}\left[r\left(X;\bm{\theta},\hat{C}^\dagger_{\bm{\theta}}\right)w(X)h_t(X)h_S(X)h_S(X)^T\right] \\
    &\quad=\left\{\frac{1-\varepsilon_p}{1-\varepsilon_q}\hat{C}^*_{\bm{\theta}}+O(\varepsilon \nu)\right\}\left\{(1-\varepsilon_q)\hat{\mathbb{E}}_{q^*}[\exp\left(\bm{\theta}^Th(X)\right)w(X)h_t(X)h_S(X)h_S(X)^T]+\varepsilon_q\rho_{6,t}\right\},
\end{align*}
for $t \in \mathcal{E}$, where $(\varepsilon_q\rho_4)_t$ is the $t$-th element of $\varepsilon_q\rho_4$ and $(\varepsilon_q\rho_5)_t$ is the $t$-th column of $\varepsilon_q\rho_5$.
Then, similar to the above discussion, we have
\begin{align*}
    \nabla_{t}\nabla^2_{SS}\mathcal{L}^\dagger(\bm{\theta})
    =(1-\varepsilon_p)\nabla_{t}\nabla^2_{SS}\mathcal{L}^*(\bm{\theta})+O\left(\varepsilon \nu\right)
\end{align*}
for $t \in \mathcal{E}$ with probability at least $1-\delta_{\tau}-\delta_{\epsilon}$.
We define $\delta = \delta_{\tau} = \delta_{\epsilon}$.

\section{Proof of Proposition \ref{proposition robustness of minimum eignvalue}}
\label{section discussion of assumption 1}

\subsection{Notations}

We define the sample weighted Fisher information matrix in the uncontaminated setting as
\begin{align*}
    \hat{\mathcal{I}}^*
    &=\hat{\mathbb{E}}_{q^*}\left[r\left(X;\bm{\theta}^*,\hat{C}^*_{\bm{\theta}}\right)w(X)h(X)h(X)^T\right] \\
    &\quad-\frac{1}{\hat{\mathbb{E}}_{p^*}[w(X)]}\hat{\mathbb{E}}_{q^*}\left[r\left(X;\bm{\theta}^*,\hat{C}^*_{\bm{\theta}}\right)w(X)h(X)\right]\hat{\mathbb{E}}_{q^*}\left[r\left(X;\bm{\theta}^*,\hat{C}^*_{\bm{\theta}}\right)w(X)h(X)\right]^T.
\end{align*}
When considering the contaminated setting, the contaminated sample weighted Fisher information matrix can be written as
\begin{align*}
    \hat{\mathcal{I}}^{\dagger}
    &=\hat{\mathbb{E}}_{q^{\dagger}}\left[r\left(X;\bm{\theta}^*,\hat{C}^{\dagger}_{\bm{\theta}}\right)w(X)h(X)h(X)^T\right] \\
    &\quad-\frac{1}{\hat{\mathbb{E}}_{p^\dagger}[w(X)]}\hat{\mathbb{E}}_{q^{\dagger}}\left[r\left(X;\bm{\theta}^*,\hat{C}^{\dagger}_{\bm{\theta}}\right)w(X)h(X)\right]\hat{\mathbb{E}}_{q^{\dagger}}\left[r\left(X;\bm{\theta}^*,\hat{C}^{\dagger}_{\bm{\theta}}\right)w(X)h(X)\right]^T.
\end{align*}

\subsection{Lemma}

The following lemma shows the relationship between the population and sample weighted Fisher information matrices.

\begin{lemma}
\label{lemma relationship of I-star and hatI-star}
If Assumptions \ref{assumption weight for normal}, \ref{assumption-9}, and \ref{assumption-8} hold and $n_{p,q}^* \gtrsim \log d$ holds, we have
\begin{align*}
    \left|\hat{\mathcal{I}}_{ij}^* - \mathcal{I}_{ij}^*\right| \le \eta
\end{align*}
for $i,j \in \mathcal{E}$ with probability at least $1-\delta_\eta$, where
\begin{align*}
    \eta = \sqrt{\frac{18D_0^2\log(4/\delta_\eta)}{n_{p,q}^*}}
\end{align*}
and $D_0$ is a positive constant.
\end{lemma}

\begin{proof}[Proof of Lemma \ref{lemma relationship of I-star and hatI-star}]
We define random variables as
\begin{align*}
    Z_n^{(w)} &= \frac{w(X_n)}{\mathbb{E}_{p^*}[w(X)]} - 1, \\
    Z_n^{(r)} &= \frac{r_{\bm{\theta}}^*(X_n)w(X_n)}{\mathbb{E}_{q^*}[r_{\bm{\theta}}^*(X)w(X)]} - 1, \\
    Z_n^{(h_i)} &= r_{\bm{\theta}}^*(X_n)w(X_n)h_i(X_n) - \mathbb{E}_{q^*}[r_{\bm{\theta}}^*(X)w(X)h_i(X)], \\
    Z_n^{(h_i,h_j)} &= r_{\bm{\theta}}^*(X_n)w(X_n)h_i(X_n)h_j(X_n) - \mathbb{E}_{q^*}[r_{\bm{\theta}}^*(X)w(X)h_i(X)h_j(X)],
\end{align*}
where $r_{\bm{\theta}}^*(X)=r(X;\bm{\theta},C_{\bm{\theta}}^*)$ and $i,j\in \mathcal{E}$.
From Assumption \ref{assumption weight for normal} and Propositions \ref{proposition of bounded dre function} and \ref{proposition bound of features and density ratio function}, we have
\begin{align*}
    \left|Z_n^{(w)}\right| &\le \left|\frac{w(X_n)}{\mathbb{E}_{p^*}[w(X)]}\right| + 1 \le \frac{W_{\max}}{W_{\min}}+1, \\
    \left|Z_n^{(r)}\right| &\le \left|\frac{r_{\bm{\theta}}^*(X_n)w(X_n)}{\mathbb{E}_{q^*}[r_{\bm{\theta}}^*(X)w(X)]}\right|+1 \le \frac{E'_{\max}}{W'_{\max}}+1, \\
    \left|Z_n^{(h_i)}\right| &\le \left|r_{\bm{\theta}}^*(X_n)w(X_n)h_i(X_n)\right|+\left|\mathbb{E}_{q^*}[r_{\bm{\theta}}^*(X)w(X)h_i(X)]\right| \le 2D'_{\max}, \\
    \left|Z_n^{(h_i,h_j)}\right| &\le \left|r_{\bm{\theta}}^*(X_n)w(X_n)h_i(X_n)h_j(X_n)\right|+\left|\mathbb{E}_{q^*}[r_{\bm{\theta}}^*(X)w(X)h_i(X)h_j(X)]\right| \le 2D'_{\max}.
\end{align*}
Therefore, the variables $Z_n^{(w)}$, $Z_n^{(r)}$, $Z_n^{(h_i)}$, and $Z_n^{(h_i,h_j)}$ are zero-mean bounded random variables.
We define the bound of these random variables as
\begin{align*}
    D_Z = \max\left\{\frac{W_{\max}}{W_{\min}}+1, \frac{E'_{\max}}{W'_{\max}}+1, 2D'_{\max}\right\}.
\end{align*}
We define the i.i.d. sample means of the above random variables as
\begin{align*}
    \frac{\hat{\mathbb{E}}_{p^*}[w(X)]}{\mathbb{E}_{p^*}[w(X)]} &= 1 + \bar{Z}^{(w)}, \quad
    \bar{Z}^{(w)} = \frac{1}{n_p^*}\sum_{n=1}^{n_p^*}Z_n^{(w)}, \\
    \frac{\hat{\mathbb{E}}_{q^*}[r_{\bm{\theta}}^*(X)w(X)]}{\mathbb{E}_{q^*}[r_{\bm{\theta}}^*(X)w(X)]} &= 1 + \bar{Z}^{(r)}, \quad
    \bar{Z}^{(r)} = \frac{1}{n_q^*}\sum_{n=1}^{n_q^*}Z_n^{(r)}, \\
    \hat{\mathbb{E}}_{q^*}[r_{\bm{\theta}}^*(X)w(X)h_i(X)] &= \mathbb{E}_{q^*}[r_{\bm{\theta}}^*(X)w(X)h_i(X)] + \bar{Z}^{(h_i)}, \quad
    \bar{Z}^{(h_i)} = \frac{1}{n_q^*}\sum_{n=1}^{n_q^*}Z_n^{(h_i)}, \\
    \hat{\mathbb{E}}_{q^*}[r_{\bm{\theta}}^*(X)w(X)h_i(X)h_j(X)] &= \mathbb{E}_{q^*}[r_{\bm{\theta}}^*(X)w(X)h_i(X)h_j(X)] + \bar{Z}^{(h_i,h_j)}, \quad
    \bar{Z}^{(h_i,h_j)} = \frac{1}{n_q^*}\sum_{n=1}^{n_q^*}Z_n^{(h_i,h_j)}. \\
\end{align*}
By Hoeffding's inequality \cite{hoeffding_1963} and $n_{p,q}^*=\min\{n_p^*, n_q^*\}$, we have
\begin{align*}
    P\left(\left|\bar{Z}^{(w)}\right| \ge \xi\right) &\le 2\exp\left(-\frac{2n_p^*\xi^2}{(2D_Z)^2}\right) \le 2\exp\left(-\frac{n_{p,q}^*\xi^2}{2D_Z^2}\right), \\
    P\left(\left|\bar{Z}^{(r)}\right| \ge \xi\right) &\le 2\exp\left(-\frac{2n_q^*\xi^2}{(2D_Z)^2}\right) \le 2\exp\left(-\frac{n_{p,q}^*\xi^2}{2D_Z^2}\right), \\
    P\left(\left|\bar{Z}^{(h_i)}\right| \ge \xi\right) &\le 2\exp\left(-\frac{2n_q^*\xi^2}{(2D_Z)^2}\right) \le 2\exp\left(-\frac{n_{p,q}^*\xi^2}{2D_Z^2}\right), \\
    P\left(\left|\bar{Z}^{(h_i,h_j)}\right| \ge \xi\right) &\le 2\exp\left(-\frac{2n_q^*\xi^2}{(2D_Z)^2}\right) \le 2\exp\left(-\frac{n_{p,q}^*\xi^2}{2D_Z^2}\right).
\end{align*}
Let
\begin{align*}
    \xi = \sqrt{\frac{4D_Z^2\log(4d/\delta_\xi)}{n_{p,q}^*}},
\end{align*}
where $\delta_\xi$ is a small positive constant and $\delta_\xi<1$.
Then, because $2\log(4d/\delta_\xi)=\log(16d^2/\delta_\xi^2)\ge \log(8d^2/\delta_\xi)$, we have
\begin{align*}
    2\exp\left(-\frac{n_{p,q}^*\xi^2}{2D_Z^2}\right) \le \frac{\delta_\xi}{4d^2}.
\end{align*}
Then, we suppose $|\bar{Z}^{(w)}|\le \xi$, $|\bar{Z}^{(r)}|\le \xi$, $|\bar{Z}^{(h_i)}|\le \xi$, and $|\bar{Z}^{(h_i,h_j)}|\le \xi$ for $i,j\in \mathcal{E}$ with probability at least $1-\delta_{\xi}$ because
\begin{align*}
    P\left(\left|\bar{Z}^{(w)}\right| \ge \xi\right) &\le \frac{\delta_{\xi}}{4d^2} \le \frac{\delta_{\xi}}{4}, \\
    P\left(\left|\bar{Z}^{(r)}\right| \ge \xi\right) &\le \frac{\delta_{\xi}}{4d^2} \le \frac{\delta_{\xi}}{4}, \\
    P\left(\bigcup_{1\le i\le d}\left\{ \left|\bar{Z}^{(h_i)}\right|\ge \xi \right\} \right)
    &\le \sum_{1\le i \le d}P\left(\left|\bar{Z}^{(h_i)}\right|\ge \xi\right)
    \le d \frac{\delta_\xi}{4d^2} \le \frac{\delta_\xi}{4}, \\
    P\left(\bigcup_{1\le i,j\le d}\left\{ \left|\bar{Z}^{(h_i,h_j)}\right|\ge \xi \right\} \right)
    &\le \sum_{1\le i,j \le d}P\left(\left|\bar{Z}^{(h_i,h_j)}\right|\ge \xi\right)
    \le d^2 \frac{\delta_\xi}{4d^2} = \frac{\delta_\xi}{4}.
\end{align*}

Because
\begin{align*}
    \frac{\hat{C}^*_{\bm{\theta}}}{C^*_{\bm{\theta}}}
    = \frac{\hat{\mathbb{E}}_{p^*}\left[w(X)\right]}{{\mathbb{E}}_{p^*}\left[w(X)\right]} \frac{{\mathbb{E}}_{q^*}\left[r\left(X;\bm{\theta},C^*_{\bm{\theta}}\right)w(X)\right]}{\hat{\mathbb{E}}_{q^*}\left[r\left(X;\bm{\theta},C^*_{\bm{\theta}}\right)w(X)\right]}
    = \frac{1 + \bar{Z}^{(w)}}{1+\bar{Z}^{(r)}},
\end{align*}
we have
\begin{align*}
    \hat{\mathcal{I}}^*
    &=\frac{\hat{C}^*_{\bm{\theta}}}{C^*_{\bm{\theta}}} \hat{\mathbb{E}}_{q^*}\left[r\left(X;\bm{\theta}^*,{C}^*_{\bm{\theta}}\right)w(X)h(X)h(X)^T\right] \\
    &\quad-\left(\frac{\hat{C}^*_{\bm{\theta}}}{C^*_{\bm{\theta}}}\right)^2 \frac{1}{\hat{\mathbb{E}}_{p^*}[w(X)]}\hat{\mathbb{E}}_{q^*}\left[r\left(X;\bm{\theta}^*,{C}^*_{\bm{\theta}}\right)w(X)h(X)\right]\hat{\mathbb{E}}_{q^*}\left[r\left(X;\bm{\theta}^*,{C}^*_{\bm{\theta}}\right)w(X)h(X)\right]^T \\
    &= \frac{1 + \bar{Z}^{(w)}}{1+\bar{Z}^{(r)}} \left[{\mathbb{E}}_{q^*}\left[r\left(X;\bm{\theta}^*,{C}^*_{\bm{\theta}}\right)w(X)h(X)h(X)^T\right] + \left(\bar{Z}^{(h_i,h_j)}\right)_{i,j=1}^d\right] \\
    &\quad - \left(\frac{1 + \bar{Z}^{(w)}}{1+\bar{Z}^{(r)}}\right)^2  \frac{1}{{\mathbb{E}}_{p^*}[w(X)]}\frac{1}{1+\bar{Z}^{(w)}} \left[{\mathbb{E}}_{q^*}\left[r\left(X;\bm{\theta}^*,{C}^*_{\bm{\theta}}\right)w(X)h(X)\right] + \left(\bar{Z}^{(h_i)}\right)_{i=1}^d\right] \\
    &\quad\quad\quad \left[{\mathbb{E}}_{q^*}\left[r\left(X;\bm{\theta}^*,{C}^*_{\bm{\theta}}\right)w(X)h(X)\right] + \left(\bar{Z}^{(h_i)}\right)_{i=1}^d\right]^T.
\end{align*}
From $n_{p,q}^* \gtrsim \log d$, we can assume that $\xi$ is sufficiently small.
Using the Taylor expansion with $\bar{Z}^{(w)}$ and $\bar{Z}^{(r)}$, we have
\begin{align*}
    \frac{1 + \bar{Z}^{(w)}}{1+\bar{Z}^{(r)}} = (1+\bar{Z}^{(w)})(1-\bar{Z}^{(r)}+O(\xi^2))
    =1+\bar{Z}^{(w)}-\bar{Z}^{(r)}+O(\xi^2).
\end{align*}
Applying the same technique, the $(i, j)$-th element of $\hat{\mathcal{I}}^*$ can be expressed by
\begin{align*}
    \hat{\mathcal{I}}^*_{ij} = {\mathcal{I}}^*_{ij} + \bar{Z}^{(p^*)} + \bar{Z}^{(q^*)}_{ij} + O\left(\xi^2\right),
\end{align*}
where
\begin{align*}
    \bar{Z}^{(p^*)} = \frac{1}{n_p^*}\sum_{n=1}^{n_p^*}Z^{(p^*)}_n \quad{\rm and}\quad
    \bar{Z}^{(q^*)}_{ij} = \frac{1}{n_q^*}\sum_{n=1}^{n_q^*}Z^{(q^*)}_{ij,n}.
\end{align*}
The concrete formulas of $Z^{(p^*)}_n$ and $Z^{(q^*)}_{ij,n}$ are omitted because those are very complicated, but we can easily see that $Z^{(p^*)}_n$ and $Z^{(q^*)}_{ij,n}$ are zero-mean independent random variables with some bound $D_0$, where $D_0$ is a positive constant.

When considering the tail behavior, we have
\begin{align*}
    P\left(\left| \hat{\mathcal{I}}^*_{ij} - {\mathcal{I}}^*_{ij}\right| \ge \eta\right)
    &\le P\left(\left|\bar{Z}^{(p^*)}\right|\ge \eta/3\right) + P\left(\left|\bar{Z}^{(q^*)}_{ij}\right|\ge \eta/3\right) + P\left(|O\left(\xi^2\right)|\ge \eta/3\right).
\end{align*}
Since $\xi$ is sufficiently small, $O(\xi^2)$ is bounded by $\eta/6$ with probability one.
Then, by Hoeffding's inequality, we have
\begin{align*}
    P\left(\left| \hat{\mathcal{I}}^*_{ij} - {\mathcal{I}}^*_{ij}\right| \ge \eta\right)
    &\le 2\exp\left(-\frac{2n_p^*(\eta/3)^2}{(2D_0)^2}\right) + 2\exp\left(-\frac{2n_q^*(\eta/3)^2}{(2D_0)^2}\right) \\
    &\le 4\exp\left(-\frac{n_{p,q}^*\eta^2}{18D_0^2}\right) =: \delta_{\eta}.
\end{align*}
\end{proof}

\subsection{Proposition}

\begin{proposition}
\label{proposition min eigen of sample fisher}
If Assumptions \ref{assumption weight for normal}, \ref{assumption-9}, \ref{assumption-8}, and \ref{assumption-1} hold and $n_{p,q}^* \gtrsim k^2\log d$ holds, we have
\begin{align*}
    \Lambda_{\min}\left[\hat{\mathcal{I}}^*_{SS}\right]
    \ge \frac{1}{2}\lambda_{\min}
\end{align*}
with probability at least $1-\delta_{\eta}$.
\end{proposition}

\begin{proof}[Proof of Proposition \ref{proposition min eigen of sample fisher}]
The proof outline is mainly owing to \cite{ravikumar2010}.
When considering $\mathcal{I}^*_{SS}$ and $\hat{\mathcal{I}}^*_{SS}$, we have
\begin{align*}
    \Lambda_{\min}\left[{\mathcal{I}}^*_{SS}\right]
    &= \min_{\|x\|_2=1} x^T {\mathcal{I}}^*_{SS} x \\
    &= \min_{\|x\|_2=1} \left\{x^T \hat{\mathcal{I}}^*_{SS} x + x^T \left({\mathcal{I}}^*_{SS}-\hat{\mathcal{I}}^*_{SS}\right) x\right\} \\
    &\le y^T\hat{\mathcal{I}}^*_{SS}y + y^T \left({\mathcal{I}}^*_{SS}-\hat{\mathcal{I}}^*_{SS}\right)y,
\end{align*}
where $y\in \mathbb{R}^k$ is a unit-norm minimal eigenvector of $\hat{\mathcal{I}}^*_{SS}$.
Therefore, from Assumption \ref{assumption-1}, we have
\begin{align*}
    \Lambda_{\min}\left[\hat{\mathcal{I}}^*_{SS}\right]
    \ge \Lambda_{\min}\left[{\mathcal{I}}^*_{SS}\right] - \left\|{\mathcal{I}}^*_{SS}-\hat{\mathcal{I}}^*_{SS}\right\|_2
    \ge \lambda_{\min}- \left\|{\mathcal{I}}^*_{SS}-\hat{\mathcal{I}}^*_{SS}\right\|_2.
\end{align*}
From Lemma \ref{lemma relationship of I-star and hatI-star} and Proposition \ref{proposition of eigenvalues of a small matrix}, with probability at least $1-\delta_{\eta}$, we have
\begin{align*}
    \left\|{\mathcal{I}}^*_{SS}- \hat{\mathcal{I}}^*_{SS}\right\|_2
    = \left\|O_{k \times k}(\eta)\right\|_2
    = O(k\eta).
\end{align*}
Note that $n_{p,q}^*\gtrsim k^2\log d$ supposed in this proposition is stronger than $n_{p,q}^*\gtrsim \log d$ supposed in Lemma \ref{lemma relationship of I-star and hatI-star}, and $\eta$ is sufficiently small.
Because $n_{p,q}^* \gtrsim k^2 \log d$ holds from the assumption of this proposition, $k\eta$ is assumed to be sufficiently small from the definition of $\eta$ in Lemma \ref{lemma relationship of I-star and hatI-star}.
Hence, we have
\begin{align*}
    \left\|{\mathcal{I}}^*_{SS}- \hat{\mathcal{I}}^*_{SS}\right\|_2
    \le \frac{\lambda_{\min}}{2}
\end{align*}
and
\begin{align*}
    \Lambda_{\min}\left[\hat{\mathcal{I}}^*_{SS}\right]
    \ge \lambda_{\min}- \frac{\lambda_{\min}}{2}
    = \frac{\lambda_{\min}}{2}
\end{align*}
with probability at least $1-\delta_{\eta}$.
\end{proof}

\subsection{Proof of Proposition \ref{proposition robustness of minimum eignvalue}}

From Theorem \ref{theorem of robust objective function} and its proof, we have
\begin{align*}
    \hat{\mathcal{I}}_{SS}^\dagger=(1-\varepsilon_p)\hat{\mathcal{I}}_{SS}^*+O_{k\times k}(\varepsilon \nu)
\end{align*}
with probability at least $1-\delta_{\tau}-\delta_{\epsilon}$.
Note that $n_{p,q}^* \ge N_{\delta}$ in the assumption of Theorem \ref{theorem of robust objective function} holds when $n_{p,q}^* \gtrsim k^2 \log d$ holds.
Because $\hat{\mathcal{I}}_{SS}^\dagger$ and $(1-\varepsilon_p)\hat{\mathcal{I}}_{SS}^*$ are symmetric, $O_{k\times k}(\varepsilon \nu)$ is also a symmetric matrix.
Then, by taking the minimum eigenvalue, we have
\begin{align*}
    \Lambda_{\min}\left[\hat{\mathcal{I}}_{SS}^\dagger\right]
    \ge (1-\varepsilon_p)\Lambda_{\min}\left[\hat{\mathcal{I}}_{SS}^*\right] + \Lambda_{\min}\left[O_{k\times k}(\varepsilon \nu)\right]
\end{align*}
with probability at least $1-\delta_{\tau}-\delta_{\epsilon}$.
Because $k\varepsilon\nu$ is assumed to be sufficiently small in Assumption \ref{assumption weight for outlier}, from Proposition \ref{proposition of eigenvalues of a small matrix}, we have
\begin{align*}
    \left|\Lambda_{\min}\left[O_{k\times k}(\varepsilon \nu)\right]\right|
    = \left|O(k\varepsilon \nu)\right|
    \le \frac{(1-\varepsilon_p)\lambda_{\min}}{4}.
\end{align*}
From Proposition \ref{proposition min eigen of sample fisher}, we have
\begin{equation*}
    \Lambda_{\min}\left[\hat{\mathcal{I}}_{SS}^\dagger\right]
    \ge (1-\varepsilon_p)\frac{\lambda_{\min}}{2} - \frac{(1-\varepsilon_p)\lambda_{\min}}{4}
    \ge \frac{(1-\varepsilon_p)\lambda_{\min}}{4}
\end{equation*}
with probability at least $1-\delta_{\tau}-\delta_{\epsilon}-\delta_{\eta}$.
We define $\delta=\delta_{\tau}=\delta_{\epsilon}=\delta_{\eta}$.

\section{Proof of Proposition \ref{proposition incoherence of I dagger}}
\label{section consequences of assumption-2}

\subsection{Lemma}

We show a lemma about the tail probability of the max-norm.

\begin{lemma}
\label{lemma control of bound of fisher matrix}
    If Assumptions \ref{assumption weight for normal}, \ref{assumption-9}, \ref{assumption-8}, and \ref{assumption-1} hold and if $n_{p,q}^*\gtrsim k^2\log d$ holds, we have
    \begin{align*}
        \sqrt{k}\left\|\hat{\mathcal{I}}_{S^cS}^{*}-{\mathcal{I}}_{S^cS}^{*}\right\|_{\infty} &\le \zeta
        = \sqrt{\frac{18D_0^2}{n_{p,q}^*}k^3\left(\log(4/\delta_{\zeta})+\log k +\log(d-k)\right)}, \\
        \sqrt{k}\left\|\hat{\mathcal{I}}_{SS}^{*}-{\mathcal{I}}_{SS}^{*}\right\|_{\infty} &\le \zeta'
        = \sqrt{\frac{18D_0^2}{n_{p,q}^*}k^3\left(\log(4/\delta_{\zeta'})+2\log k\right)}, \\
        \left\|\hat{\mathcal{I}}_{SS}^{*^{-1}}-{\mathcal{I}}_{SS}^{*^{-1}}\right\|_\infty &\le \zeta''
        = \sqrt{\frac{72D_0^2}{\lambda_{\min}^4n_{p,q}^*}k^3\left(\log(4/\delta_{\zeta''})+2\log k\right)},
    \end{align*}
    with probability at least $1-\delta_{\zeta}$, $1-\delta_{\zeta'}$, and $1-\delta_{\eta}-\delta_{\zeta''}$, respectively.
    Furthermore, if $n_{p,q}^* \gtrsim k^3 \log d$ holds, we have
    \begin{align*}
        \zeta \le \lambda_{\min}\frac{\alpha}{12},
        \quad \zeta' \le \lambda_{\min}\frac{\alpha}{24(1-\alpha)},
        \quad \zeta'' \le \frac{1}{\lambda_{\min}}.
    \end{align*}    
\end{lemma}

\begin{proof}[Proof of Lemma \ref{lemma control of bound of fisher matrix}]
The proof outline is mainly owing to \cite{ravikumar2010}.
Because the sample condition in Lemma \ref{lemma relationship of I-star and hatI-star} is satisfied by the condition in this lemma, Lemma \ref{lemma relationship of I-star and hatI-star} suggests that the $(i,j)$-th element of $\hat{\mathcal{I}}^*_{S^cS}-{\mathcal{I}}^*_{S^cS}$ for $i\in S^c$ and $j\in S$ is bounded by
\begin{align*}
    P\left(\left| \hat{\mathcal{I}}^*_{ij} - {\mathcal{I}}^*_{ij}\right| \ge \eta\right)
    &\le 4\exp\left(-\frac{n_{p,q}^*\eta^2}{18D_0^2}\right).
\end{align*}
When we define $Z_{ij}=\hat{\mathcal{I}}^*_{ij} - {\mathcal{I}}^*_{ij}$ and $\zeta=k^{3/2}\eta$, we have
\begin{align*}
    P\left[ \left|Z_{ij}\right| \ge \frac{\zeta}{k^{3/2}} \right] \le 4\exp\left(-\frac{n_{p,q}^*}{18D_0^2}\left(\frac{\zeta}{k^{3/2}}\right)^2\right)=:\frac{\delta_{\zeta}}{k(d-k)}.
\end{align*}
By the definition of the $L_{\infty}$-matrix norm, we have
\begin{align*}
    P\left[\sqrt{k}\left\|\hat{\mathcal{I}}^*_{S^cS}-{\mathcal{I}}^*_{S^cS}\right\|_{\infty}\ge \zeta\right]
    &= P\left[\sqrt{k} \max_{i\in S^c} \sum_{j\in S} \left|Z_{ij}\right| \ge \zeta\right] \\
    &\le (d-k) P\left[\sqrt{k}\sum_{j\in S}\left|Z_{ij}\right|\ge \zeta\right],
\end{align*}
where the final inequality uses a union bound.
Via the another union bound over the row elements, we have
\begin{align*}
    P\left[\sqrt{k}\sum_{j\in S}\left|Z_{ij}\right|\ge \zeta\right]
    &\le P\left[\exists j \in S \,|\, \sqrt{k}\left|Z_{ij}\right|\ge \frac{\zeta}{k}\right] \\
    &\le k P\left[\sqrt{k}\left|Z_{ij}\right|\ge \frac{\zeta}{k}\right].
\end{align*}
Then,
\begin{align*}
    P\left[\sqrt{k}\left|Z_{ij}\right|\ge \frac{\zeta}{k}\right]
    = P\left[\left|Z_{ij}\right|\ge \frac{\zeta}{k^{3/2}}\right]
    \le \frac{\delta_\zeta}{k(d-k)}.
\end{align*}
Putting all together, we have
\begin{align*}
    P\left[\sqrt{k}\left\|\hat{\mathcal{I}}^*_{S^cS}-{\mathcal{I}}^*_{S^cS}\right\|_{\infty}\ge \zeta\right]
    \le k(d-k)\frac{\delta_\zeta}{k(d-k)}
    = \delta_\zeta.
\end{align*}
The bound of $\sqrt{k}\left\|\hat{\mathcal{I}}_{SS}^{*}-{\mathcal{I}}_{SS}^{*}\right\|_{\infty}$ is analogous with the pre-factor $d-k$ replaced by $k$.

When considering the last term, we write
\begin{align*}
    \left\|\hat{\mathcal{I}}_{SS}^{*^{-1}}-{\mathcal{I}}_{SS}^{*^{-1}}\right\|_\infty
    &=\left\|{\mathcal{I}}_{SS}^{*^{-1}}\left[ {\mathcal{I}}_{SS}^{*} - \hat{\mathcal{I}}_{SS}^{*} \right]\hat{\mathcal{I}}_{SS}^{*^{-1}}\right\|_\infty \\
    &\le \sqrt{k} \left\|{\mathcal{I}}_{SS}^{*^{-1}}\left[ {\mathcal{I}}_{SS}^{*} - \hat{\mathcal{I}}_{SS}^{*} \right]\hat{\mathcal{I}}_{SS}^{*^{-1}}\right\|_2 \\
    &\le \sqrt{k} \left\|{\mathcal{I}}_{SS}^{*^{-1}}\right\|_2 \left\|{\mathcal{I}}_{SS}^{*} - \hat{\mathcal{I}}_{SS}^{*}\right\|_2 \left\|\hat{\mathcal{I}}_{SS}^{*^{-1}}\right\|_2.
\end{align*}
Assumption \ref{assumption-1} suggests that
\begin{align*}
    \left\|{\mathcal{I}}_{SS}^{*^{-1}}\right\|_2
    = \left(\Lambda_{\min}\left[{\mathcal{I}}_{SS}^{*}\right]\right)^{-1}
    \le \frac{1}{\lambda_{\min}}
\end{align*}
and Proposition \ref{proposition min eigen of sample fisher} suggests that
\begin{align*}
    \left\|\hat{\mathcal{I}}_{SS}^{*^{-1}}\right\|_2
    = \left(\Lambda_{\min}\left[\hat{\mathcal{I}}_{SS}^{*}\right]\right)^{-1}
    \le \frac{2}{\lambda_{\min}}
\end{align*}
with probability at least $1-\delta_{\eta}$.
Note that the sample condition in Proposition \ref{proposition min eigen of sample fisher} is satisfied by the condition of this lemma.
Then,
\begin{align*}
    \left\|\hat{\mathcal{I}}_{SS}^{*^{-1}}-{\mathcal{I}}_{SS}^{*^{-1}}\right\|_\infty
    \le \frac{2}{\lambda_{\min}^2} \sqrt{k} \left\|{\mathcal{I}}_{SS}^{*} - \hat{\mathcal{I}}_{SS}^{*}\right\|_2
\end{align*}
holds with probability at least $1-\delta_{\eta}$.
From Lemma \ref{lemma relationship of I-star and hatI-star}, when defining $Z_{ij}=\hat{\mathcal{I}}_{ij}^{*} - {\mathcal{I}}_{ij}^{*}$ and $\eta=\lambda_{\min}^2\zeta''/2k^{3/2}$, we consider
\begin{align*}
    P\left[ \left|Z_{ij}\right| \ge \frac{\lambda_{\min}^2}{2k^{3/2}}\zeta'' \right] \le 4\exp\left(-\frac{n_{p,q}^*}{18D_0^2}\left(\frac{\lambda_{\min}^2}{2k^{3/2}}\zeta''\right)^2\right)=:\frac{\delta''_{\zeta}}{k^2}
\end{align*}
for $i, j\in S$.
Applying the union bound over the $k^2$ index pairs $(i,j)$ then yields
\begin{align*}
    P\left[\left\|\hat{\mathcal{I}}_{SS}^{*^{-1}}-{\mathcal{I}}_{SS}^{*^{-1}}\right\|_\infty \ge \zeta''\right]
    &\le P\left[\left\|{\mathcal{I}}_{SS}^{*} - \hat{\mathcal{I}}_{SS}^{*}\right\|_2 \ge \zeta'' \frac{\lambda_{\min}^2}{2\sqrt{k}}\right] \\
    &\le P\left[\left(\sum_{i,j\in S}\left(Z_{ij}\right)^2 \right)^{1/2} \ge \zeta'' \frac{\lambda_{\min}^2}{2\sqrt{k}}\right] \\
    &\le \sum_{i,j\in S} P\left[\left|Z_{ij}\right| \ge \frac{\zeta''}{k} \frac{\lambda_{\min}^2}{2\sqrt{k}}\right] \\
    &\le k^2 \frac{\delta''_{\zeta}}{k^2} \\
    &= \delta_{\zeta''}.
\end{align*}

Finally, if we assume that $n_{p,q}^* \gtrsim k^3 \log d$ holds, we can assume that $\zeta$, $\zeta'$, and $\zeta''$ are sufficiently small because $k \le d-k \le d$ holds in the common sparse estimation setting.
Therefore, we have
\begin{align*}
    \zeta   \le \lambda_{\min}\frac{\alpha}{12}, \quad
    \zeta'  \le \lambda_{\min}\frac{\alpha}{24(1-\alpha)}, \quad
    \zeta'' \le \frac{1}{\lambda_{\min}}.
\end{align*}
\end{proof}

\subsection{Proposition}

The following proposition shows the incoherence assumption of the sample weighted Fisher information matrix in the uncontaminated setting.

\begin{proposition}
\label{proposition incoherence of sample}
If Assumptions \ref{assumption weight for normal}, \ref{assumption-9}, \ref{assumption-8}, \ref{assumption-1}, and \ref{assumption-2} hold and $n_{p,q}^* \gtrsim k^3 \log d$ holds,
\begin{align*}
    \left\| \hat{\mathcal{I}}_{S^cS}^* \hat{\mathcal{I}}_{SS}^{*^{-1}}  \right\|_{\infty} \le 1 - \frac{3\alpha}{4}
\end{align*}
holds with probability at least  $1-\delta_{\eta}-\delta_{\zeta}-\delta_{\zeta'}-\delta_{\zeta''}$.
\end{proposition}

\begin{proof}[Proof of Proposition \ref{proposition incoherence of sample}]
The proof outline is mainly owing to \cite{ravikumar2010}.
We begin by decomposing $\hat{\mathcal{I}}_{S^cS}^* \hat{\mathcal{I}}_{SS}^{*^{-1}}$ as the sum $\hat{\mathcal{I}}_{S^cS}^* \hat{\mathcal{I}}_{SS}^{*^{-1}}={T}_1+{T}_2+{T}_3+{T}_4$, where
\begin{align*}
    {T}_1 &:= \mathcal{I}_{S^cS}^*\left[ \hat{\mathcal{I}}_{SS}^{*^{-1}} - {\mathcal{I}}_{SS}^{*^{-1}} \right], \\
    {T}_2 &:= \left[\hat{\mathcal{I}}_{S^cS}^{*}-{\mathcal{I}}_{S^cS}^{*}\right]{\mathcal{I}}_{SS}^{*^{-1}}, \\
    {T}_3 &:= \left[\hat{\mathcal{I}}_{S^cS}^{*}-{\mathcal{I}}_{S^cS}^{*}\right]\left[\hat{\mathcal{I}}_{SS}^{*^{-1}}-{\mathcal{I}}_{SS}^{*^{-1}}\right], \\
    {T}_4 &:= {\mathcal{I}}_{S^cS}^{*}{\mathcal{I}}_{SS}^{*^{-1}}.
\end{align*}
Based on Lemma \ref{lemma control of bound of fisher matrix}, we can control the tail bound probability of ${T}_{1},...,T_4$.

\textit{Control of ${T}_1$}:
We can re-factorize ${T}_1$ as
\begin{align*}
    {T}_1 = \mathcal{I}_{S^cS}^*{\mathcal{I}}_{SS}^{*^{-1}}\left[ {\mathcal{I}}_{SS}^{*} - \hat{\mathcal{I}}_{SS}^{*} \right]\hat{\mathcal{I}}_{SS}^{*^{-1}}
\end{align*}
and then bound it using the sub-multiplicative property $\|AB\|_{\infty}\le \|A\|_\infty \|B\|_\infty$ as follows:
\begin{align*}
    \left\|{T}_1\right\|_{\infty}
    &\le \left\|\mathcal{I}_{S^cS}^*{\mathcal{I}}_{SS}^{*^{-1}}\right\|_\infty \left\|{\mathcal{I}}_{SS}^{*} - \hat{\mathcal{I}}_{SS}^{*}\right\|_\infty \left\|\hat{\mathcal{I}}_{SS}^{*^{-1}}\right\|_\infty \\
    &\le (1-\alpha) \left\|{\mathcal{I}}_{SS}^{*} - \hat{\mathcal{I}}_{SS}^{*}\right\|_\infty \left(\sqrt{k} \left\|\hat{\mathcal{I}}_{SS}^{*^{-1}}\right\|_2\right),
\end{align*}
from Assumption \ref{assumption-2}. From Proposition \ref{proposition min eigen of sample fisher}, we have
\begin{align*}
    \left\|\hat{\mathcal{I}}_{SS}^{*^{-1}}\right\|_2
    = \left(\Lambda_{\min}\left[\hat{\mathcal{I}}_{SS}^*\right]\right)^{-1}
    \le \frac{2}{\lambda_{\min}}
\end{align*}
with probability at least $1-\delta_{\eta}$.
Then, from Lemma \ref{lemma control of bound of fisher matrix}, we have
\begin{align*}
    \left\|{T}_1\right\|_{\infty}\le \frac{2(1-\alpha)}{\lambda_{\min}}\zeta' \le \frac{\alpha}{12}
\end{align*}
with probability at least $1-\delta_{\eta}-\delta_{\zeta'}$.

\textit{Control of ${T}_2$}:
From Assumption \ref{assumption-1} and Lemma \ref{lemma control of bound of fisher matrix},
\begin{align*}
    \left\|{T}_2\right\|_\infty
    \le \left\|\hat{\mathcal{I}}_{S^cS}^{*}-{\mathcal{I}}_{S^cS}^{*}\right\|_{\infty} \left\|{\mathcal{I}}_{SS}^{*^{-1}}\right\|_\infty
    \le \left\|\hat{\mathcal{I}}_{S^cS}^{*}-{\mathcal{I}}_{S^cS}^{*}\right\|_\infty \sqrt{k} \left\|{\mathcal{I}}_{SS}^{*^{-1}}\right\|_2
    \le \frac{1}{\lambda_{\min}}\zeta
    \le \frac{\alpha}{12}
\end{align*}
holds with probability at least $1-\delta_{\zeta}$.

\textit{Control of ${T}_3$}:
Lemma \ref{lemma control of bound of fisher matrix} shows that
\begin{align*}
    \left\|{T}_3\right\|_\infty
    \le \left\|\hat{\mathcal{I}}_{S^cS}^{*}-{\mathcal{I}}_{S^cS}^{*}\right\|_\infty \left\|\hat{\mathcal{I}}_{SS}^{*^{-1}}-{\mathcal{I}}_{SS}^{*^{-1}}\right\|_\infty 
    \le  \frac{\zeta\zeta''}{\sqrt{k}}
    \le \frac{\alpha}{12}
\end{align*}
with probability as least $1-\delta_{\eta}-\delta_{\zeta}-\delta_{\zeta''}$.

\textit{Control of ${T}_4$}:
From Assumption \ref{assumption-2}, we have
\begin{align*}
    \left\|{T}_4\right\|_\infty = \left\|{\mathcal{I}}_{S^cS}^{*}{\mathcal{I}}_{SS}^{*^{-1}}\right\|_\infty \le 1-\alpha.
\end{align*}

Putting together all of the pieces, we conclude that
\begin{align*}
    \left\|\hat{\mathcal{I}}_{S^cS}^* \hat{\mathcal{I}}_{SS}^{*^{-1}}\right\|_{\infty}
    &= \left\|{T}_1+{T}_2+{T}_3+{T}_4\right\|_\infty \\
    &\le \left\|{T}_1\right\|_\infty + \left\|{T}_2\right\|_\infty + \left\|{T}_3\right\|_\infty + \left\|{T}_4\right\|_\infty \\
    &\le \frac{\alpha}{12} + \frac{\alpha}{12} + \frac{\alpha}{12} + (1-\alpha) \\
    &= 1 - \frac{3\alpha}{4}
\end{align*}
with probability at least  $1-\delta_{\eta}-\delta_{\zeta}-\delta_{\zeta'}-\delta_{\zeta''}$.
\end{proof}

\subsection{Proof of Proposition \ref{proposition incoherence of I dagger}}

From Theorem \ref{theorem of robust objective function}, we have
\begin{align*}
    \hat{\mathcal{I}}^\dagger_{SS}=(1-\varepsilon_p)\hat{\mathcal{I}}^*_{SS}+O_{k\times k}(\varepsilon \nu), \quad
    \hat{\mathcal{I}}^\dagger_{S^cS}=(1-\varepsilon_p)\hat{\mathcal{I}}^*_{S^cS}+O_{(d-k)\times k}(\varepsilon \nu)
\end{align*}
with probability at least $1-\delta_{\tau}-\delta_{\epsilon}$.
Note that $n_{p,q}^* \ge N_{\delta}$ in the assumption of Theorem \ref{theorem of robust objective function} is weaker than $n_{p,q}^*\gtrsim k^3\log d$ in the assumption of this proposition.
Because $n_{p,q}^*\gtrsim k^2\log d$ in the assumption of Proposition \ref{proposition robustness of minimum eignvalue} is weaker than $n_{p,q}^*\gtrsim k^3\log d$ in the assumption of this proposition, $\hat{\mathcal{I}}_{SS}^\dagger$ is invertible with probability at least $1-\delta_{\tau}-\delta_{\epsilon}-\delta_{\eta}$.
Using the relationship of $(A+B)^{-1}=A^{-1}-A^{-1}B(A+B)^{-1}$ for $A, B\in \mathbb{R}^{k\times k}$, we have
\begin{align*}
    \hat{\mathcal{I}}^{\dagger^{-1}}_{SS}
    =\left\{(1-\varepsilon_p)\hat{\mathcal{I}}^*_{SS}+O_{k\times k}(\varepsilon \nu)\right\}^{-1}
    =(1-\varepsilon_p)^{-1}\hat{\mathcal{I}}^{*^{-1}}_{SS} - (1-\varepsilon_p)^{-1}\hat{\mathcal{I}}^{*^{-1}}_{SS}O_{k\times k}(\varepsilon \nu)\hat{\mathcal{I}}^{\dagger^{-1}}_{SS}.
\end{align*}
with probability at least $1-\delta_{\tau}-\delta_{\epsilon}-\delta_{\eta}$.
Then, we have
\begin{align*}
    &\left\|\hat{\mathcal{I}}^\dagger_{S^cS}\hat{\mathcal{I}}^{\dagger^{-1}}_{SS}\right\|_{\infty} \\
    &= \left\|\left\{(1-\varepsilon_p)\hat{\mathcal{I}}^*_{S^cS}+O_{(d-k)\times k}(\varepsilon \nu)\right\}\left\{(1-\varepsilon_p)^{-1}\hat{\mathcal{I}}^{*^{-1}}_{SS} - (1-\varepsilon_p)^{-1}\hat{\mathcal{I}}^{*^{-1}}_{SS}O_{k\times k}(\varepsilon \nu)\hat{\mathcal{I}}^{\dagger^{-1}}_{SS}\right\}\right\|_{\infty} \\
    &= \Big\|
        \hat{\mathcal{I}}^*_{S^cS}\hat{\mathcal{I}}^{*^{-1}}_{SS}
        +\hat{\mathcal{I}}^*_{S^cS}\hat{\mathcal{I}}^{*^{-1}}_{SS}O_{k\times k}(\varepsilon \nu)\hat{\mathcal{I}}^{\dagger^{-1}}_{SS}
        +(1-\varepsilon_p)^{-1}O_{(d-k)\times k}(\varepsilon\nu) \hat{\mathcal{I}}^{*^{-1}}_{SS} \\
        &\qquad+(1-\varepsilon_p)^{-1}O_{(d-k)\times k}(\varepsilon \nu)\hat{\mathcal{I}}^{*^{-1}}_{SS}O_{k\times k}(\varepsilon \nu)\hat{\mathcal{I}}^{\dagger^{-1}}_{SS}
    \Big\|_{\infty} \\
    &\le \left\|\hat{\mathcal{I}}^*_{S^cS}\hat{\mathcal{I}}^{*^{-1}}_{SS}\right\|_{\infty}
        +\left\|\hat{\mathcal{I}}^*_{S^cS}\hat{\mathcal{I}}^{*^{-1}}_{SS}\right\|_{\infty} \left\|O_{k\times k}(\varepsilon \nu)\right\|_{\infty} \left\|\hat{\mathcal{I}}^{\dagger^{-1}}_{SS}\right\|_{\infty}
        +(1-\varepsilon_p)^{-1} \left\|O_{(d-k)\times k}(\varepsilon\nu)\right\|_{\infty}\left\|\hat{\mathcal{I}}^{*^{-1}}_{SS}\right\|_{\infty} \\
        &\qquad+(1-\varepsilon_p)^{-1}\left\|O_{(d-k)\times k}(\varepsilon \nu)\right\|_{\infty}\left\|\hat{\mathcal{I}}^{*^{-1}}_{SS}\right\|_{\infty}\left\|O_{k\times k}(\varepsilon \nu)\right\|_{\infty} \left\|\hat{\mathcal{I}}^{\dagger^{-1}}_{SS}\right\|_{\infty}.
\end{align*}
with probability at least $1-\delta_{\tau}-\delta_{\epsilon}-\delta_{\eta}$.
Proposition \ref{proposition min eigen of sample fisher} suggests that
\begin{align*}
    \left\|\hat{\mathcal{I}}_{SS}^{*^{-1}}\right\|_\infty
    \le \sqrt{k} \left\|\hat{\mathcal{I}}_{SS}^{*^{-1}}\right\|_2
    = \sqrt{k} \left(\Lambda_{\min}\left[\hat{\mathcal{I}}_{SS}^*\right]\right)^{-1}
    \le \frac{2\sqrt{k}}{\lambda_{\min}}
\end{align*}
with probability at least $1-\delta_\eta$, and Proposition \ref{proposition robustness of minimum eignvalue} suggests that
\begin{align*}
    \left\|\hat{\mathcal{I}}^{\dagger^{-1}}_{SS}\right\|_{\infty}
    \le \sqrt{k} \left\|\hat{\mathcal{I}}^{\dagger^{-1}}_{SS}\right\|_{2}
    = \sqrt{k} \left(\Lambda_{\min}\left[\hat{\mathcal{I}}_{SS}^\dagger\right]\right)^{-1}
    \le \frac{4\sqrt{k}}{(1-\varepsilon_p) \lambda_{\min}}
\end{align*}
with probability at least $1-\delta_{\tau}-\delta_{\epsilon}-\delta_{\eta}$.
Because $\varepsilon\nu$ is assumed to be sufficiently small from Assumption \ref{assumption weight for outlier}, from Proposition \ref{proposition of max norm of a small matrix}, we have
\begin{align*}
    \left\|O_{k\times k}(\varepsilon \nu)\right\|_{\infty} = O(k\varepsilon\nu), \quad
    \left\|O_{(d-k)\times k}(\varepsilon \nu)\right\|_{\infty} = O(k\varepsilon\nu).
\end{align*}
From Assumption \ref{assumption weight for outlier}, we can assume that $k^{3/2}\varepsilon\nu$ is sufficiently small, which indicates that $\left|O(k^{3/2}\varepsilon\nu)\right| \le \alpha/4$.
Then, from Proposition \ref{proposition incoherence of sample}, we have
\begin{align*}
    \left\|\hat{\mathcal{I}}^\dagger_{S^cS}\hat{\mathcal{I}}^{\dagger^{-1}}_{SS}\right\|_{\infty}
    &\le \left\|\hat{\mathcal{I}}^*_{S^cS}\hat{\mathcal{I}}^{*^{-1}}_{SS}\right\|_{\infty} \left(1+O(k\varepsilon\nu)\frac{4\sqrt{k}}{(1-\varepsilon_p)\lambda_{\min}}\right)
        + (1-\varepsilon_p)^{-1}O(k\varepsilon\nu) \frac{2\sqrt{k}}{\lambda_{\min}} \\
        &\quad\qquad+ (1-\varepsilon_p)^{-1}O(k\varepsilon\nu) \frac{2\sqrt{k}}{\lambda_{\min}} O(k\varepsilon\nu) \frac{4\sqrt{k}}{(1-\varepsilon_p) \lambda_{\min}} \\
    &\le \left(1-\frac{3\alpha}{4}\right) + O(k^{3/2}\varepsilon\nu) \\
    &\le 1-\frac{\alpha}{2},
\end{align*}
with probability at least $1-\delta_{\tau}-\delta_{\epsilon}-\delta_{\eta}-\delta_{\zeta}-\delta_{\zeta'}-\delta_{\zeta''}$.
We define $\delta=\delta_{\tau}=\delta_{\epsilon}=\delta_{\eta}=\delta_{\zeta}=\delta_{\zeta'}=\delta_{\zeta''}$.

\section{Proposition from Assumption \ref{assumption smoothness}}
\label{section discussion of robust objective function}

\subsection{Lemma}

We show a lemma about the derivative of the sample weighted Fisher information matrix.

\begin{lemma}
\label{lemma sample-population third derivative}
If Assumptions \ref{assumption weight for normal}, \ref{assumption-9}, and \ref{assumption-8} hold and $n_{p,q}^*\gtrsim \log d$ holds, we have
\begin{align*}
    \left| \nabla_t\hat{\mathcal{I}}^*_{ij}(\bm{\theta}) - \nabla_t{\mathcal{I}}^*_{ij}(\bm{\theta})\right| \le \eta'
\end{align*}
for any $t\in \mathcal{E}$, $i,j\in S$, and $\bm{\theta} \in \Theta$, with probability at least $1-\delta_{\eta'}$.
\end{lemma}

\begin{proof}[Proof of Lemma \ref{lemma sample-population third derivative}]
The proof outline is the same as the proof of Lemma \ref{lemma relationship of I-star and hatI-star}.
We define a random variable as 
\begin{align*}
    Z_n^{(h_i,h_j,h_t)} = r_{\bm{\theta}}^*(X_n)w(X_n)h_i(X_n)h_j(X_n)h_t(X_n) - \mathbb{E}_{q^*}[r_{\bm{\theta}}^*(X)w(X)h_i(X)h_j(X)h_t(X)],
\end{align*}
for $i, j\in S$ and $t \in \mathcal{E}$.
$Z_n^{(h_i,h_j,h_t)}$ is a zero-mean bounded random variable, where its bound is given by $2D'_{\max}$ from Proposition \ref{proposition bound of features and density ratio function}.
By transforming the above equation, we have
\begin{align*}
    \hat{\mathbb{E}}_{q^*}[r_{\bm{\theta}}^*(X)w(X)h_i(X)h_j(X)h_t(X)] &= \mathbb{E}_{q^*}[r_{\bm{\theta}}^*(X)w(X)h_i(X)h_j(X)h_t(X)] + \bar{Z}^{(h_i,h_j,h_t)}, \\
    \bar{Z}^{(h_i,h_j,h_t)} &= \frac{1}{n_q^*}\sum_{n=1}^{n_q^*} Z_n^{(h_i,h_j,h_t)}.
\end{align*}
Because $\bar{Z}^{(h_i,h_j,h_t)}$ is an i.i.d. sample mean of the zero-mean bounded random variable, from Hoeffding's inequality, we have
\begin{align*}
    P\left(\left|\bar{Z}^{(h_i,h_j,h_t)}\right| \ge \xi\right) &\le 2\exp\left(-\frac{2n_q^*\xi^2}{(2D_Z)^2}\right) \le 2\exp\left(-\frac{n_{p,q}^*\xi^2}{2D_Z^2}\right).
\end{align*}
When we define
\begin{align*}
    \xi = \sqrt{\frac{4D_Z^2\log(4d/\delta_\xi)}{n_{p,q}^*}},
\end{align*}
then, because $2\log (4d/\delta_\xi)=\log(16d^2/\delta_\xi^2)\ge\log(10d^2/\delta_\xi)$, we have
\begin{align*}
    2\exp\left(-\frac{n_{p,q}^*\xi^2}{2D_Z^2}\right) \le \frac{\delta_\xi}{5d^2}.
\end{align*}
Then, we suppose $|\bar{Z}^{(w)}|\le \xi$, $|\bar{Z}^{(r)}|\le \xi$, $|\bar{Z}^{(h_i)}|\le \xi$, $|\bar{Z}^{(h_i,h_j)}|\le \xi$, and $|\bar{Z}^{(h_i,h_j,h_t)}|\le \xi$ for any $i,j\in \mathcal{S}$ with probability at least $1-\delta_{\xi}$ because
\begin{align*}
    P\left(\left|\bar{Z}^{(w)}\right| \ge \xi\right) &\le \frac{\delta_{\xi}}{5d^2} \le \frac{\delta_{\xi}}{5}, \\
    P\left(\left|\bar{Z}^{(r)}\right| \ge \xi\right) &\le \frac{\delta_{\xi}}{5d^2} \le \frac{\delta_{\xi}}{5}, \\
    P\left(\bigcup_{1\le i\le k}\left\{ \left|\bar{Z}^{(h_i)}\right|\ge \xi \right\} \right)
    &\le \sum_{1\le i \le k}P\left(\left|\bar{Z}^{(h_i)}\right|\ge \xi\right)
    \le k \frac{\delta_\xi}{5d^2} \le \frac{\delta_\xi}{5}, \\
    P\left(\bigcup_{1\le i,j\le k}\left\{ \left|\bar{Z}^{(h_i,h_j)}\right|\ge \xi \right\} \right)
    &\le \sum_{1\le i,j \le k}P\left(\left|\bar{Z}^{(h_i,h_j)}\right|\ge \xi\right)
    \le k^2 \frac{\delta_\xi}{5d^2} \le \frac{\delta_\xi}{5}, \\
    P\left(\bigcup_{1\le i,j\le k}\left\{ \left|\bar{Z}^{(h_i,h_j,h_t)}\right|\ge \xi \right\} \right)
    &\le \sum_{1\le i,j \le k}P\left(\left|\bar{Z}^{(h_i,h_j,h_t)}\right|\ge \xi\right)
    \le k^2 \frac{\delta_\xi}{5d^2} \le \frac{\delta_\xi}{5},
\end{align*}
for $t \in \mathcal{E}$.

From \eqref{eq 3rd derivative}, we have
\begin{align*}
    &\nabla_t\hat{\mathcal{I}}^*_{SS}(\bm{\theta}) \\
    &=\frac{\hat{C}^*_{\bm{\theta}}}{{C}^*_{\bm{\theta}}} \hat{\mathbb{E}}_{q^*}\left[r\left(X;\bm{\theta},{C}^*_{\bm{\theta}}\right)w(X)h_t(X)h_S(X)h_S(X)^T\right] \\
    &\quad-\left(\frac{\hat{C}^*_{\bm{\theta}}}{{C}^*_{\bm{\theta}}}\right)^2 \frac{1}{\hat{\kappa}^*}\hat{\mathbb{E}}_{q^*}\left[r\left(X;\bm{\theta},{C}^*_{\bm{\theta}}\right)w(X)h_t(X)\right]\hat{\mathbb{E}}_{q^*}\left[r\left(X;\bm{\theta},{C}^*_{\bm{\theta}}\right)w(X)h_S(X)h_S(X)^T\right] \\
    &\quad-\left(\frac{\hat{C}^*_{\bm{\theta}}}{{C}^*_{\bm{\theta}}}\right)^2 \frac{1}{\hat{\kappa}^*}\hat{\mathbb{E}}_{q^*}\left[r\left(X;\bm{\theta},{C}^*_{\bm{\theta}}\right)w(X)h_t(X)h_S(X)\right]\hat{\mathbb{E}}_{q^*}\left[r\left(X;\bm{\theta},{C}^*_{\bm{\theta}}\right)w(X)h_S(X)\right]^T \\
    &\quad-\left(\frac{\hat{C}^*_{\bm{\theta}}}{{C}^*_{\bm{\theta}}}\right)^2 \frac{1}{\hat{\kappa}^*}\hat{\mathbb{E}}_{q^*}\left[r\left(X;\bm{\theta},{C}^*_{\bm{\theta}}\right)w(X)h_S(X)\right]\hat{\mathbb{E}}_{q^*}\left[r\left(X;\bm{\theta},{C}^*_{\bm{\theta}}\right)w(X)h_t(X)h_S(X)\right]^T \\
    &\quad+\left(\frac{\hat{C}^*_{\bm{\theta}}}{{C}^*_{\bm{\theta}}}\right)^3 \frac{2}{\hat{\kappa}^{* 2}}\hat{\mathbb{E}}_{q^*}\left[r\left(X;\bm{\theta},{C}^*_{\bm{\theta}}\right)w(X)h_t(X)\right]\hat{\mathbb{E}}_{q^*}\left[r\left(X;\bm{\theta},{C}^*_{\bm{\theta}}\right)w(X)h_S(X)\right]\hat{\mathbb{E}}_{q^*}\left[r\left(X;\bm{\theta},{C}^*_{\bm{\theta}}\right)w(X)h_S(X)\right]^T \\
    &=\frac{1 + \bar{Z}^{(w)}}{1+\bar{Z}^{(r)}} \left[{\mathbb{E}}_{q^*}\left[r\left(X;\bm{\theta},{C}^*_{\bm{\theta}}\right)w(X)h_t(X)h_S(X)h_S(X)^T\right] + \left(\bar{Z}^{(h_i,h_j,h_t)}\right)_{i,j\in S}\right] \\
    &\quad-\left(\frac{1 + \bar{Z}^{(w)}}{1+\bar{Z}^{(r)}}\right)^2 \frac{1}{{\mathbb{E}}_{p^*}[w(X)](1+\bar{Z}^{(w)})} \\
        &\quad\quad\quad\left[{\mathbb{E}}_{q^*}\left[r\left(X;\bm{\theta},{C}^*_{\bm{\theta}}\right)w(X)h_t(X)\right] + \bar{Z}^{(h_t)}\right] \\
        &\quad\quad\quad \left[{\mathbb{E}}_{q^*}\left[r\left(X;\bm{\theta},{C}^*_{\bm{\theta}}\right)w(X)h_S(X)h_S(X)^T\right] + \left(\bar{Z}^{(h_i,h_j)}\right)_{i,j\in S}\right] \\
    &\quad-\left(\frac{1 + \bar{Z}^{(w)}}{1+\bar{Z}^{(r)}}\right)^2 \frac{1}{{\mathbb{E}}_{p^*}[w(X)](1+\bar{Z}^{(w)})} \\
        &\quad\quad\quad\left[{\mathbb{E}}_{q^*}\left[r\left(X;\bm{\theta},{C}^*_{\bm{\theta}}\right)w(X)h_t(X)h_S(X)\right]+\left(\bar{Z}^{(h_t,h_j)}\right)_{j\in S}\right] \\
        &\quad\quad\quad \left[{\mathbb{E}}_{q^*}\left[r\left(X;\bm{\theta},{C}^*_{\bm{\theta}}\right)w(X)h_S(X)\right] + \left(\bar{Z}^{(h_i)}\right)_{i\in S}\right]^T \\
    &\quad-\left(\frac{1 + \bar{Z}^{(w)}}{1+\bar{Z}^{(r)}}\right)^2 \frac{1}{{\mathbb{E}}_{p^*}[w(X)](1+\bar{Z}^{(w)})} \\
        &\quad\quad\quad\left[{\mathbb{E}}_{q^*}\left[r\left(X;\bm{\theta},{C}^*_{\bm{\theta}}\right)w(X)h_S(X)\right] + \left( \bar{Z}^{(h_i)}\right)_{i\in S}\right] \\
        &\quad\quad\quad \left[{\mathbb{E}}_{q^*}\left[r\left(X;\bm{\theta},{C}^*_{\bm{\theta}}\right)w(X)h_t(X)h_S(X)\right] + \left(\bar{Z}^{(h_t,h_j)}\right)_{j\in S}\right]^T \\
    &\quad+\left(\frac{1 + \bar{Z}^{(w)}}{1+\bar{Z}^{(r)}}\right)^3 \frac{2}{\{{\mathbb{E}}_{p^*}[w(X)](1+\bar{Z}^{(w)})\}^2} \\
        &\quad\quad\quad \left[{\mathbb{E}}_{q^*}\left[r\left(X;\bm{\theta},{C}^*_{\bm{\theta}}\right)w(X)h_t(X)\right] + \bar{Z}^{(h_t)}\right] \\
        &\quad\quad\quad \left[{\mathbb{E}}_{q^*}\left[r\left(X;\bm{\theta},{C}^*_{\bm{\theta}}\right)w(X)h_S(X)
    \right] + \left(\bar{Z}^{(h_i)}\right)_{i\in S}\right] \\
        &\quad\quad\quad \left[{\mathbb{E}}_{q^*}\left[r\left(X;\bm{\theta},{C}^*_{\bm{\theta}}\right)w(X)h_S(X)\right] + \left(\bar{Z}^{(h_i)}\right)_{i\in S}\right]^T,
\end{align*}
for any $t\in \mathcal{E}$ and $\bm{\theta} \in \Theta$.
Then, using the Taylor expansion, the $(i,j)$-th element of $\nabla_t\hat{\mathcal{I}}^*_{SS}(\bm{\theta})$ can be expressed by
\begin{align*}
    \nabla_t\hat{\mathcal{I}}^*_{ij}(\bm{\theta}) = \nabla_t{\mathcal{I}}^*_{ij}(\bm{\theta}) + \bar{V}^{(p^*)}_t + \bar{V}^{(q^*)}_{tij} + O\left(\xi^2\right),
\end{align*}
where
\begin{align*}
    \bar{V}^{(p^*)}_t = \frac{1}{n_p^*}\sum_{n=1}^{n_p^*}V^{(p^*)}_{t,n}, \quad
    \bar{V}^{(q^*)}_{tij} = \frac{1}{n_q^*}\sum_{n=1}^{n_q^*}V^{(q^*)}_{tij,n}.
\end{align*}
The concrete formulas of $V^{(p^*)}_{t,n}$ and $V^{(q^*)}_{tij,n}$ are omitted because those are very complicated, but we can easily see that $V^{(p^*)}_{t,n}$ and $V^{(q^*)}_{tij,n}$ are zero-mean independent random variables with some bound $D_1$.

When considering the tail behavior, we have
\begin{align*}
    P\left(\left| \nabla_t\hat{\mathcal{I}}^*_{ij}(\bm{\theta}) - \nabla_t{\mathcal{I}}^*_{ij}(\bm{\theta})\right| \ge \eta'\right)
    &\le P\left(\left|\bar{V}^{(p^*)}_t\right|\ge \eta'/3\right) + P\left(\left|\bar{V}^{(q^*)}_{tij}\right|\ge \eta'/3\right) + P\left(\left|O\left(\xi^2\right)\right|\ge \eta'/3\right).
\end{align*}
If $n_{p,q}^*\gtrsim \log d$ holds, we can assume that $\xi$ is sufficiently small and $O(\xi^2)$ is bounded by $\eta/6$ with probability one.
Then, by Hoeffding's inequality, we have
\begin{align*}
    P\left(\left| \nabla_t\hat{\mathcal{I}}^*_{ij}(\bm{\theta}) - \nabla_t{\mathcal{I}}^*_{ij}(\bm{\theta})\right| \ge \eta'\right)
    &\le 2\exp\left(-\frac{2n_p^*(\eta'/3)^2}{(2D_1)^2}\right) + 2\exp\left(-\frac{2n_q^*(\eta'/3)^2}{(2D_1)^2}\right) \\
    &\le 4\exp\left(-\frac{n_{p,q}^*\eta'^2}{18D_1^2}\right) = \delta_{\eta'}.
\end{align*}
\end{proof}

\subsection{Proposition}

\begin{proposition}
\label{proposition contaminated function boundedness}
If Assumptions \ref{assumption weight for normal}, \ref{assumption-9}, \ref{assumption-8}, \ref{assumption weight for outlier}, and \ref{assumption smoothness} hold and $n_{p,q}^* \gtrsim k^2\log d$ holds, the contaminated third derivative is bounded by
\begin{align*}
    \left\|\nabla_{t} \nabla^2_{SS} \mathcal{L}^\dagger(\bm{\theta})\right\|_{2}
    \le 3(1-\varepsilon_p)\lambda_{3,\max}
\end{align*}
for any $t\in \mathcal{E}$ and $\bm{\theta}\in\Theta$, with probability at least $1-\delta_{\tau}-\delta_{\epsilon}-\delta_{\eta'}$.
\end{proposition}

\begin{proof}[Proof of Proposition \ref{proposition contaminated function boundedness}]
From Assumption \ref{assumption smoothness} and Lemma \ref{lemma sample-population third derivative} and Proposition \ref{proposition of eigenvalues of a small matrix}, we have
\begin{align*}
    \Lambda_{\max}\left[\nabla_t\hat{\mathcal{I}}^*_{SS}(\bm{\theta})\right]
    &\le \left\|\nabla_t{\mathcal{I}}^*_{SS}(\bm{\theta})\right\|_2 + \left\|\nabla_t\hat{\mathcal{I}}^*_{SS}(\bm{\theta}) - \nabla_t{\mathcal{I}}^*_{SS}(\bm{\theta})\right\|_2 \\
    &= \lambda_{3,\max}+ \left\|O_{k\times k}(\eta)\right\|_2 \\
    &= \lambda_{3,\max}+O(k\eta)
\end{align*}
with probability at least $1-\delta_{\eta'}$.
Note that $n_{p,q}^*\gtrsim k^2 \log d$ in the assumption of this proposition is stronger than $n_{p,q}^*\ge \log d$ in the assumption of Lemma \ref{lemma sample-population third derivative}.
Because we can assume that $k\eta$ is sufficiently small from the assumption of $n_{p,q}^* \gtrsim k^2 \log d$, we have $|O(k\eta)| \le \lambda_{3,\max}$.
Then, we have
\begin{align*}
    \left\|\nabla_{t} \nabla^2_{SS} \mathcal{L}^*(\bm{\theta})\right\|_{2}
    =\Lambda_{\max}\left[\nabla_t\hat{\mathcal{I}}^*_{SS}(\bm{\theta})\right]
    \le 2\lambda_{3,\max}
\end{align*}
for any $t\in \mathcal{E}$ and $\bm{\theta}\in\Theta$, with probability at least $1-\delta_{\eta'}$.

Then, from Theorem \ref{theorem of robust objective function}, we have
\begin{align*}
    \left\|\nabla_{t} \nabla^2_{SS} \mathcal{L}^\dagger(\bm{\theta})\right\|_{2}
    &= \left\|(1-\varepsilon_p)\nabla_{t}\nabla^2_{SS}\mathcal{L}^*(\bm{\theta})+O_{k\times k}(\varepsilon \nu)\right\|_2 \\
    &\le \left\|(1-\varepsilon_p)\nabla_{t}\nabla^2_{SS}\mathcal{L}^*(\bm{\theta})\right\|_2 + \left\|O_{k\times k}(\varepsilon \nu)\right\|_2 \\
    &\le 2(1-\varepsilon_p)\lambda_{3,\max}+O(k\varepsilon\nu) \\
    &\le 3(1-\varepsilon_p)\lambda_{3,\max}
\end{align*}
with probability at least $1-\delta_{\tau}-\delta_{\epsilon}-\delta_{\eta'}$.
The last inequality holds because $k\varepsilon\nu$ is assumed to be sufficiently small in Assumption \ref{assumption weight for outlier}.
Note that $n_{p,q}^*\gtrsim k^2 \log d$ in the assumption of this proposition is stronger than $n_{p,q}^*\ge N_{\delta}$ in the assumption of Theorem \ref{theorem of robust objective function}.
\end{proof}

\section{Proofs of Lemmas in Section \ref{subsection zero-pattern recovery}}
\label{section proof of lemmas in main proof}

\subsection{Proof of Lemma \ref{lemma-1}}
Because $\hat{\bm{\theta}}$ and $\tilde{\bm{\theta}}$ are optimal, we have
\begin{align*}
    \mathcal{L}^\dagger(\hat{\bm{\theta}})+\lambda_{n_p^*,n_q^*}\|\hat{\bm{\theta}}\|_1
    = \mathcal{L}^\dagger(\tilde{\bm{\theta}})+\lambda_{n_p^*,n_q^*}\|\tilde{\bm{\theta}}\|_1.
\end{align*}
Due to the convexity of $\mathcal{L}^\dagger({\bm{\theta}})$ and the definition of $\nabla\mathcal{L}^\dagger(\hat{\bm{\theta}})=-\lambda_{n_p^*,n_q^*}\hat{\bm{z}}$, we have
\begin{align*}
    \mathcal{L}^\dagger(\tilde{\bm{\theta}})
    \ge \mathcal{L}^\dagger(\hat{\bm{\theta}}) + \lambda_{n_p^*,n_q^*} \langle \tilde{\bm{\theta}}-\hat{\bm{\theta}}, -\hat{\bm{z}} \rangle.
\end{align*}
Similarly, due to the convexity of $\|{\bm{\theta}}\|_1$ and the definition of $\hat{\bm{z}}\in \nabla\|\hat{\bm{\theta}}\|_1$, we have
\begin{align*}
    \|\tilde{\bm{\theta}}\|_1
    \ge \|\hat{\bm{\theta}}\|_1 + \langle \tilde{\bm{\theta}}-\hat{\bm{\theta}}, \hat{\bm{z}} \rangle.
\end{align*}
Therefore,
\begin{align*}
    \mathcal{L}^\dagger(\hat{\bm{\theta}})+\lambda_{n_p^*,n_q^*}\|\hat{\bm{\theta}}\|_1
    &\ge \mathcal{L}^\dagger(\hat{\bm{\theta}}) + \lambda_{n_p^*,n_q^*} \langle \tilde{\bm{\theta}}-\hat{\bm{\theta}}, -\hat{\bm{z}} \rangle + \lambda_{n_p^*,n_q^*} \langle \tilde{\bm{\theta}}-\hat{\bm{\theta}}, \hat{\bm{z}} \rangle + \lambda_{n_p^*,n_q^*}\|\hat{\bm{\theta}}\|_1 \\
    &\ge \mathcal{L}^\dagger(\hat{\bm{\theta}}) + \lambda_{n_p^*,n_q^*}\|\hat{\bm{\theta}}\|_1.
\end{align*}
The above inequality suggests that all the inequality we have used should take the exact equality.
Therefore, we have
\begin{align*}
    \|\tilde{\bm{\theta}}\|_1
    = \|\hat{\bm{\theta}}\|_1 + \langle \tilde{\bm{\theta}}-\hat{\bm{\theta}}, \hat{\bm{z}} \rangle
    = \langle \tilde{\bm{\theta}}, \hat{\bm{z}} \rangle.
\end{align*}
Because we assume that $\|\hat{\bm{z}}_{S^c}\|_{\infty}<1$, the above equality implies that $\tilde{\bm{\theta}}_{S^c}=\bm{0}$.

If $\hat{\mathcal{I}}_{SS}^\dagger$ is strictly positive definite, \eqref{eq-9} is strictly convex. Then, $\hat{\bm{\theta}}$ is the unique optimal solution of \eqref{eq-9}.

\subsection{Proof of Lemma \ref{lemma-5}}
\label{subsection proof of lemma-5}

We have
\begin{align*}
    w_t^\dagger=-\nabla_{t}\mathcal{L}^\dagger(\bm{\theta}^*)
    =\hat{\mathbb{E}}_{p^\dagger}\left[h_t(X)w(X)\right]-\hat{\mathbb{E}}_{q^\dagger}\left[r\left(X;\bm{\theta}^*, \hat{C}^\dagger_{\bm{\theta}^*}\right)h_t(X)w(X)\right]
\end{align*}
for any $t\in \mathcal{E}$.
From Theorem \ref{theorem of robust objective function}, we have
\begin{equation}
\label{eq w_t^dagger}
    w_t^\dagger=(1-\varepsilon_p)w_t^*+O(\varepsilon \nu)
\end{equation}
with probability at least $1-\delta_{\tau}-\delta_{\epsilon}$, where
\begin{align*}
    w_t^*&=\hat{\mathbb{E}}_{p^*}\left[h_t(X)w(X)\right]
    -\hat{\mathbb{E}}_{q^*}\left[r\left(X;\bm{\theta}^*, \hat{C}^*_{\bm{\theta}^*}\right)h_t(X)w(X)\right].
\end{align*}
Because we assume $p^*(\bm{x})=r(\bm{x};\bm{\theta}^*,C^*_{\bm{\theta}^*})q^*(\bm{x})$, we can show that $|w_t^*|$ is upper bounded by
\begin{equation}
\label{eq w_t^*}
\begin{split}
    |w_t^*|
    &\le \Big|\underbrace{\hat{\mathbb{E}}_{p^*}[h_t(X)w(X)]-\mathbb{E}_{p^*}\left[h_t(X)w(X)\right]}_{\bar{Z}^{(1)}}\Big| \\
    &\quad+ \Big|\underbrace{\hat{\mathbb{E}}_{q^*}\left[\left\{r(X;\bm{\theta}^*,\hat{C}^*_{\bm{\theta}^*})-r(X;\bm{\theta}^*,C^*_{\bm{\theta}^*})\right\}h_t(X)w(X)\right]}_{\bar{Z}^{(2)}}\Big| \\
    &\quad+\Big|\underbrace{\hat{\mathbb{E}}_{q^*}\left[r(X;\bm{\theta}^*,C^*_{\bm{\theta}^*})h_t(X)w(X)\right]-\mathbb{E}_{q^*}\left[r(X;\bm{\theta}^*,C^*_{\bm{\theta}^*})h_t(X)w(X)\right]}_{\bar{Z}^{(3)}}\Big|.
\end{split}
\end{equation}
Because $h_t(\bm{x})w(\bm{x})$ and $r(\bm{x};\bm{\theta}^*,C^*_{\bm{\theta}^*})h_t(\bm{x})w(\bm{x})$ are bounded from Assumption \ref{assumption-8} and Proposition \ref{proposition bound of features and density ratio function}, respectively, Hoeffding's inequality suggests that
\begin{align}
\label{eq-39}
    P\left(\left|\bar{Z}^{(1)}\right|\ge u\right)\le2\exp\left(-\frac{2n_p^* u^2}{(2D_{\max})^2}\right), \quad
    P\left(\left|\bar{Z}^{(3)}\right|\ge u\right)\le2\exp\left(-\frac{2n_q^* u^2}{(2D'_{\max})^2}\right).
\end{align}
From Proposition \ref{proposition-bounded empirical moment}, $\bar{Z}^{(2)}$ can be bounded by                                                                                             
\begin{align*}
    \left|\bar{Z}^{(2)}\right|
    &=\left| \hat{\mathbb{E}}_{q^*}\left[\left\{r(X;\bm{\theta}^*,\hat{C}^*_{\bm{\theta}^*})-r(X;\bm{\theta}^*,C^*_{\bm{\theta}^*})\right\}w(X)h_t(X)\right]\right| \\
    &=\left| \hat{\mathbb{E}}_{q^*}\left[\left\{1-\frac{C_{\bm{\theta}^*}^*}{\hat{C}_{\bm{\theta}^*}^*}\right\}r(X;\bm{\theta}^*,\hat{C}^*_{\bm{\theta}^*})w(X)h_t(X)\right]\right| \\
    &\le\left| \hat{\mathbb{E}}_{q^*}\left[r(X;\bm{\theta}^*,\hat{C}^*_{\bm{\theta}^*})w(X)h_t(X)\right]\right|\cdot\left|1-\frac{C_{\bm{\theta}^*}^*}{\hat{C}_{\bm{\theta}^*}^*}\right| \\
    &\le D''_{\max}\left|\frac{C_{\bm{\theta}^*}^*}{\hat{C}_{\bm{\theta}^*}^*}-1\right|
\end{align*}
with probability at least $1-\delta_{\tau}-\delta_{\epsilon}$.
Note that the condition of $n_{p,q}^*\ge N_{\delta}$ in this lemma satisfies the condition of Proposition \ref{proposition-bounded empirical moment}.
From \eqref{eq constant of weighted dre}, \eqref{eq empirical c}, and Proposition \ref{proposition of bounded sum of weight}, we have
\begin{align*}
    \frac{C_{\bm{\theta}^*}^*}{\hat{C}_{\bm{\theta}^*}^*}-1
    &=\frac{1}{\hat{\kappa}^*}\hat{\mathbb{E}}_{q^*}\left[r\left(X;\bm{\theta}^*,C^*_{\bm{\theta}^*}\right)w(X)\right]-1 \\
    &\le\frac{1}{W'_{\max}-\tau}\left\{\hat{\mathbb{E}}_{q^*}[r\left(X;\bm{\theta}^*,C^*_{\bm{\theta}^*}\right)w(X)]-W'_{\max}+\tau\right\},
\end{align*}
with probability at least $1-\delta_{\tau}$.
Because we assume that $n_{p,q}^* \ge N_{\delta}$ holds, $\tau\le W'_{\max}/2$ holds from Lemma \ref{lemma boundedness of tau and epsilon}.
Then, we have
\begin{equation}
\label{eq b_n_q}
\begin{split}
    \left|\bar{Z}^{(2)}\right|
    &\le\frac{D''_{\max}}{W'_{\max}-\tau}\left|\hat{\mathbb{E}}_{q^*}[r\left(X;\bm{\theta}^*,C^*_{\bm{\theta}^*}\right)w(X)]-W'_{\max}+\tau\right| \\
    &\le \Bigg|\underbrace{\frac{2D''_{\max}}{W'_{\max}}\left\{\hat{\mathbb{E}}_{q^*}[r\left(X;\bm{\theta}^*,C^*_{\bm{\theta}^*}\right)w(X)]-W'_{\max}\right\}}_{\bar{Z}'^{(2)}}\Bigg| + \frac{2D''_{\max}}{W'_{\max}}\tau.
\end{split}
\end{equation}
Because $r(X;\bm{\theta}^*,C^*_{\bm{\theta}^*})w(X)-W'_{\max}$ is a bounded zero-mean random variable from Proposition \ref{proposition of bounded dre function}, using the Hoeffding's inequality, we obtain
\begin{align}
\label{eq-41}
    P\left(\left|\bar{Z}'^{(2)}\right|\ge u\right)\le2\exp\left(-\frac{2n_q^*}{E'^2_{\max}}\left(\frac{W'_{\max}}{2D''_{\max}}u\right)^2\right).
\end{align}

Then, combining \eqref{eq w_t^dagger}, \eqref{eq w_t^*}, and \eqref{eq b_n_q}, we have
\begin{align*}
    \left|w_t^\dagger\right|
    &\le (1-\varepsilon_p)\left|w_t^*\right| + O(\varepsilon \nu) \\
    &\le (1-\varepsilon_p)\left\{\left|\bar{Z}^{(1)}\right|+\left|\bar{Z}'^{(2)}\right|+\left|\bar{Z}^{(3)}\right|\right\} +\frac{2D''_{\max}}{W'_{\max}}(1-\varepsilon_p)\tau+O(\varepsilon \nu).
\end{align*}
From the definition of $\tau$ in Proposition \ref{proposition of bounded sum of weight}, we have
\begin{align*}
    \left|w_t^\dagger\right|
    &\le (1-\varepsilon_p)\left\{\left|\bar{Z}^{(1)}\right|+\left|\bar{Z}'^{(2)}\right|+\left|\bar{Z}^{(3)}\right|\right\} +M'(1-\varepsilon_p)\sqrt{\frac{\log(2/\delta_{\tau})}{n_p^*}}+O(\varepsilon \nu),
\end{align*}
where $M'=2\sqrt{2}D''_{\max}W_{\max}/W'_{\max}$.
From the definition of $O(\varepsilon\nu)$, there exists a positive constant $N'$ which satisfies
\begin{align*}
    \left|w_t^\dagger\right|
    &\le (1-\varepsilon_p)\left\{\left|\bar{Z}^{(1)}\right|+\left|\bar{Z}'^{(2)}\right|+\left|\bar{Z}^{(3)}\right|\right\} +M'(1-\varepsilon_p)\sqrt{\frac{\log(2/\delta_{\tau})}{n_p^*}}+N'\varepsilon \nu.
\end{align*}
Therefore, combining \eqref{eq-39} and \eqref{eq-41},
\begin{align*}
    &P\left(|w_t^\dagger|\ge 3(1-\varepsilon_p) u +M'(1-\varepsilon_p)\sqrt{\frac{\log(2/\delta_{\tau})}{n_p^*}}+N'\varepsilon \nu\right) \\
    &\le P\Bigg((1-\varepsilon_p)\left\{\left|\bar{Z}^{(1)}\right|+\left|\bar{Z}'^{(2)}\right|+\left|\bar{Z}^{(3)}\right|\right\} +M'(1-\varepsilon_p)\sqrt{\frac{\log(2/\delta_{\tau})}{n_p^*}}+N'\varepsilon \nu \\
    &\qquad\qquad \ge 3(1-\varepsilon_p) u +M'(1-\varepsilon_p)\sqrt{\frac{\log(2/\delta_{\tau})}{n_p^*}}+N'\varepsilon \nu\Bigg) \\
    &\le 6\exp\left(-\frac{n_{p,q}^*}{L'} u^2\right),
\end{align*}
where $L'= \max\left\{2D_{\max}^2,\, 2E'^2_{\max}D''^2_{\max}/W'^2_{\max},\,2D'^2_{\max}\right\}$ and $n_{p,q}^*=\min\{n_p^*,n_q^*\}$.
Applying the union-bound for all $t\in S \cup S^c$, we have
\begin{align*}
    P\left(\|\bm{w}^\dagger\|_\infty\ge3(1-\varepsilon_p) u+M'(1-\varepsilon_p)\sqrt{\frac{\log(2/\delta_{\tau})}{n_p^*}}+N'\varepsilon \nu\right)\le 6d\exp\left(-\frac{n_{p,q}^*}{L'} u^2\right).
\end{align*}
Then, by setting
\begin{align*}
    u=\frac{1}{3(1-\varepsilon_p)}u_0,\quad
    u_0 := \frac{\alpha}{8(2-\alpha/2)}\lambda_{n_p^*,n_q^*}-M'(1-\varepsilon_p)\sqrt{\frac{\log(2/\delta_{\tau})}{n_p^*}}-N'\varepsilon\nu,
\end{align*}
we have
\begin{align*}
    P\left(\|\bm{w}^\dagger\|_\infty \ge \frac{\alpha\lambda_{n_p^*,n_q^*}}{8(2-\alpha/2)}\right)
    \le 6d\exp\left(-\frac{n_{p,q}^*}{9(1-\varepsilon_p)^2L'}u_0^2\right)=:\delta_{\lambda}.
\end{align*}
Therefore, we have
\begin{align*}
    \|\bm{w}^\dagger\|_\infty \le \frac{\alpha\lambda_{n_p^*,n_q^*}}{8(2-\alpha/2)}
\end{align*}
with probability at least $1-\delta_{\tau}-\delta_{\epsilon}-\delta_{\lambda}$.
We have
\begin{align*}
    \lambda_{n_p^*,n_q^*}
    =L(1-\varepsilon_p)\sqrt{\frac{\log(6d/\delta_{\lambda})}{n_{p,q}^*}}+M(1-\varepsilon_p)\sqrt{\frac{\log(2/\delta_{\tau})}{n_p^*}}+N\varepsilon\nu,
\end{align*}
where
\begin{align*}
    L=\frac{24(2-\alpha/2)}{\alpha}\sqrt{L'}, \quad
    M = \frac{8(2-\alpha/2)}{\alpha}M', \quad
    N = \frac{8(2-\alpha/2)}{\alpha}N'.
\end{align*}
We define that $\delta=\delta_{\tau}=\delta_{\epsilon}=\delta_{\lambda}$.

\subsection{Proof of Lemma \ref{lemma-3}}
\label{subsection proof of lemma-3}

We consider the following function \cite{ravikumar2010}:
\begin{equation}
\label{eq-a8}
    G(\bm{\delta}_S)=\mathcal{L}^\dagger(\bm{\theta}^*+\bm{\delta})-\mathcal{L}^\dagger(\bm{\theta}^*) + \lambda_{n_p^*,n_q^*}\left(\|\bm{\theta}^*_{S}+\bm{\delta}_{S}\|_1-\|\bm{\theta}^*_{S}\|_1\right),
\end{equation}
where $\bm{\delta}=[\bm{\delta}_S^T, \bm{0}^T]^T\in\mathbb{R}^d$, $\bm{\delta}_S\in\mathbb{R}^k$, and $\bm{\theta}^*+\bm{\delta}\in\Theta$.
Note that $G$ is convex and $G(\bm{0})=0$.
Also note that $G$ reaches the minimal at $\bm{\delta}^*_S=\hat{\bm{\theta}}_S-\bm{\theta}^*_S$ and $G(\bm{\delta}^*_S)\le0$.
We can easily show that if there exists $B>0$ which satisfies $G(\tilde{\bm{\delta}}_S)>0$ for all $\tilde{\bm{\delta}}_S$ such that $\|\tilde{\bm{\delta}}_S\|_2=B$, then $\|\bm{\delta}^*_S\|_2=\|\hat{\bm{\theta}}_S-\bm{\theta}^*_S\|_2\le B$ holds.
Indeed, if $\|\bm{\delta}^*_S\|>B$, then the convex combination $c \bm{\delta}^*_S + (1-c)\bm{0}$ would satisfy $\|c\bm{\delta}^*_S + (1-c)\bm{0}\|_2=B$ for an appropriately chosen $c\in (0,1)$.
By convexity of $G$,
\begin{align*}
    G(c\bm{\delta}^*_S+(1-c)\bm{0}) \le c G(\bm{\delta}^*_S)+(1-c)G(\bm{0}) \le 0,
\end{align*}
contradicting the assumed strict positivity of $G$ on $\tilde{\bm{\delta}}_S$ such that $\|\tilde{\bm{\delta}}_S\|_2=B$.

When considering the Taylor expansion of $\mathcal{L}^\dagger(\bm{\theta}^*+\bm{\delta})$ in \eqref{eq-a8}, we have
\begin{equation}
\label{eq-a9}
    G(\bm{\delta}_S)=\bm{\delta}^T_S \nabla_S\mathcal{L}^\dagger(\bm{\theta}^*) + \frac{1}{2}\bm{\delta}^T_S \nabla^2_{SS} \mathcal{L}^\dagger(\bm{\theta}^*+\bar{\bm{\delta}}) \bm{\delta}_S + \lambda_{n_p^*,n_q^*}\left(\|\bm{\theta}^*_{S}+\bm{\delta}_{S}\|_1-\|\bm{\theta}^*_{S}\|_1\right),
\end{equation}
where $\bar{\bm{\delta}}$ exists between $\bm{0}$ and $\bm{\delta}$ in a coordinate fashion.
For the first term of \eqref{eq-a9}, from $\|\bm{\delta}_S\|_1\le \sqrt{k}\|\bm{\delta}_S\|_2$ and $\|\bm{w}^\dagger_S\|_\infty \le \frac{\lambda_{n_p^*,n_q^*}}{4}$ by the assumption of this lemma,
\begin{equation}
\label{eq-a10}
    \left|\bm{\delta}^T_S \nabla_S\mathcal{L}^\dagger(\bm{\theta}^*)\right|=|\langle\bm{w}^\dagger_S, \bm{\delta}_S\rangle|\le \|\bm{w}^\dagger_S\|_\infty\|\bm{\delta}_S\|_1\le \|\bm{w}^\dagger_S\|_\infty \sqrt{k}\|\bm{\delta}_S\|_2 \le \frac{\sqrt{k}\lambda_{n_p^*,n_q^*}}{4}\|\bm{\delta}_S\|_2.
\end{equation}
For the last term of \eqref{eq-a9},
\begin{equation}
\label{eq-a11}
    \lambda_{n_p^*,n_q^*}\left(\|\bm{\theta}^*_{S}+\bm{\delta}_{S}\|_1-\|\bm{\theta}^*_{S}\|_1\right)
    \ge -\lambda_{n_p^*,n_q^*}\|\bm{\delta}_{S}\|_1
    \ge -\sqrt{k}\lambda_{n_p^*,n_q^*}\|\bm{\delta}_S\|_2.
\end{equation}
Therefore, we only need to lower-bound the middle term of \eqref{eq-a9} to show $G(\tilde{\bm{\delta}}_S)>0$.

Let $\phi(\bar{\bm{\delta}})=\bm{\delta}^T_S \nabla^2_{SS} \mathcal{L}^\dagger(\bm{\theta}^*+\bar{\bm{\delta}}) \bm{\delta}_S$.
Obviously, we need to lower-bound $\phi(\bar{\bm{\delta}})$.
By applying the mean-value theorem,
\begin{align*}
    \phi(\bar{\bm{\delta}})
    &=\bm{\delta}^T_S \nabla^2_{SS} \mathcal{L}^\dagger(\bm{\theta}^*+\bar{\bm{\delta}}) \bm{\delta}_S \\
    &=\bm{\delta}^T_S \left[\nabla^2_{SS}\mathcal{L}^\dagger\left(\bm{\theta}^*\right)+\sum_{t=1}^{k} \bar{{\delta}}_t \nabla_t\nabla^2_{SS}\mathcal{L}^\dagger(\bm{\theta}^*+\bar{\bar{\bm{\delta}}})\right] \bm{\delta}_S,
\end{align*}
where $\bar{\bar{\bm{\delta}}}\in\mathbb{R}^p$ is between $\bm{0}$ and $\bar{\bm{\delta}}$ in a coordinate fashion.
From Weyl's inequality \cite{horn_2013}, we have
\begin{align*}
    \Lambda_{\min}\left[\nabla^2_{SS} \mathcal{L}^\dagger(\bm{\theta}^*+\bar{\bm{\delta}})\right]
    &\ge\Lambda_{\min}\left[\nabla^2_{SS}\mathcal{L}^\dagger(\bm{\theta}^*)\right]-\left\|\sum_{t=1}^{k} \bar{{\delta}}_t \nabla_t\nabla^2_{SS}\mathcal{L}^\dagger(\bm{\theta}^*+\bar{\bar{\bm{\delta}}})\right\|_2.
\end{align*}
Here,
\begin{align*}
    \left\|\sum_{t=1}^{k} \bar{{\delta}}_t \nabla_t\nabla^2_{SS}\mathcal{L}^\dagger(\bm{\theta}^*+\bar{\bar{\bm{\delta}}})\right\|_2
    &\le \sum_{t=1}^{k} \left\|\bar{{\delta}}_t \nabla_t\nabla^2_{SS}\mathcal{L}^\dagger\left(\bm{\theta}^*+\bar{\bar{\bm{\delta}}}\right)\right\|_2 \\
    &\le \sup_{t\in S} \left\|\nabla_t\nabla^2_{SS}\mathcal{L}^\dagger\left(\bm{\theta}^*+\bar{\bar{\bm{\delta}}}\right)\right\|_2 \sum_{t=1}^{k} \left|\bar{{\delta}}_t\right| \\
    &\le \sqrt{k} \sup_{t\in S} \left\|\nabla_t\nabla^2_{SS}\mathcal{L}^\dagger\left(\bm{\theta}^*+\bar{\bar{\bm{\delta}}}\right)\right\|_2 \left\|\bm{{\delta}}_S\right\|_2.
\end{align*}
The last inequality comes from $\|\bar{\bm{\delta}}_S\|_1 \le \sqrt{k} \|\bar{\bm{\delta}}_S\|_2$ and $\|\bar{\bm{\delta}}_S\|_2\le \|\bm{\delta}_S\|_2$.
Thus, from Propositions \ref{proposition robustness of minimum eignvalue} and \ref{proposition contaminated function boundedness}, we have
\begin{align*}
    \Lambda_{\min}\left[\nabla^2_{SS} \mathcal{L}^\dagger(\bm{\theta}^*+\bar{\bm{\delta}})\right]
    &\ge\Lambda_{\min}\left[\nabla^2_{SS}\mathcal{L}^\dagger(\bm{\theta}^*)\right]-\sqrt{k}\sup_{t\in S}\left\|\nabla_{t}\nabla^2_{SS}\mathcal{L}^\dagger\left(\bm{\theta}^*+\bar{\bar{\bm{\delta}}}\right)\right\|_2 \left\|\bm{\delta}_S\right\|_2 \\
    &\ge (1-\varepsilon_p)\left\{\frac{\lambda_{\min}}{4}-3\sqrt{k}\lambda_{3,\max}\|\bm{\delta}_S\|_2\right\}
\end{align*}
with probability at least $1-\delta_{\tau}-\delta_{\epsilon}-\delta_{\eta'}$.
When considering the case where
\begin{equation}
\label{eq-a14}
    \frac{1}{8}\lambda_{\min}\ge 3\sqrt{k}\lambda_{3,\max}\|\bm{\delta}_S\|_2,
\end{equation}
we have
\begin{equation}
\label{eq-a12}
    \Lambda_{\min}\left[\nabla^2_{SS}\mathcal{L}^\dagger(\bm{\theta}^*+\bar{\bm{\delta}})\right]
    \ge \frac{(1-\varepsilon_p)\lambda_{\min}}{8}.
\end{equation}

Combining \eqref{eq-a10}, \eqref{eq-a11}, and \eqref{eq-a12} with \eqref{eq-a9}, we can get
\begin{equation*}
    G(\bm{\delta}_S)\ge \frac{(1-\varepsilon_p)\lambda_{\min}}{16}\|\bm{\delta}_S\|^2_2 -\frac{5}{4}\sqrt{k}\lambda_{n_p^*,n_q^*}\|\bm{\delta}_S\|_2,
\end{equation*}
with probability at least $1-\delta_{\tau}-\delta_{\epsilon}-\delta_{\eta'}$.
Let $\|\bm{\delta}_S\|_2=\sqrt{k}\lambda_{n_p^*,n_q^*}V$, where $V>0$, then we have
\begin{align*}
    G(\bm{\delta}_S)
    &\ge k\lambda_{n_p^*,n_q^*}^2\left(\frac{(1-\varepsilon_p)\lambda_{\min}}{16}V^2 -\frac{5}{4}V\right) \\
    &=k\lambda_{n_p^*,n_q^*}^2\frac{(1-\varepsilon_p)\lambda_{\min}}{16}\left(V^2-\frac{20}{(1-\varepsilon_p)\lambda_{\min}}V\right).
\end{align*}
When we define $\tilde{\bm{\delta}}_S$ such that $V=\frac{40}{(1-\varepsilon_p)\lambda_{\min}}$, more precisely,
\begin{align}
\label{eq tilde delta}
    \left\|\tilde{\bm{\delta}}_S\right\|_2=\frac{40}{(1-\varepsilon_p)\lambda_{\min}}\sqrt{k}\lambda_{n_p^*,n_q^*}\quad (=B),
\end{align}
we have $G(\tilde{\bm{\delta}}_S)>0$.
Substituting \eqref{eq tilde delta} to \eqref{eq-a14}, we have
\begin{equation}
\label{eq-a15}
    k\lambda_{n_p^*,n_q^*}\le \frac{(1-\varepsilon_p)\lambda_{\min}^2}{960\lambda_{3,\max}}.
\end{equation}
Therefore, if \eqref{eq-a15} holds, from \eqref{eq tilde delta}, we have
\begin{equation*}
    \left\|\hat{\bm{\theta}}-\bm{\theta}^*\right\|_2
    = \left\|\hat{\bm{\theta}}_S-\bm{\theta}^*_S\right\|_2
    \le B
    = \frac{40}{(1-\varepsilon_p)\lambda_{\min}}\sqrt{k}\lambda_{n_p^*,n_q^*}.
\end{equation*}
with probability at least $1-\delta_{\tau}-\delta_{\epsilon}-\delta_{\eta'}$.
Finally we set $\delta= \delta_{\tau}=\delta_{\epsilon}=\delta_{\eta'}$.

\subsection{Proof of Lemma \ref{lemma-4}}
\label{subsection proof of lemma-4}

We need the upper bound of $|g_t^\dagger|$ for $t\in \mathcal{E}$, where
\begin{align*}
    g_t^\dagger
    &= \left[\left[\nabla^2\mathcal{L}^\dagger(\bm{\theta}^*)-\overline{\nabla^2\mathcal{L}^\dagger}\right]\left[\hat{\bm{\theta}}-\bm{\theta}^*\right]\right]_t \\
    &= \left[\nabla_{t}\nabla_S\mathcal{L}^\dagger(\bm{\theta}^*) - \nabla_{t}\nabla_S\mathcal{L}^\dagger\left(\bar{\bm{\theta}}^t\right)\right]^T (\hat{\bm{\theta}}_S-\bm{\theta}^*_S).
\end{align*}
By applying the mean-value theorem, we have
\begin{align*}
    g_t^\dagger = \left(\bar{\bm{\theta}}_S^t-\bm{\theta}_S^*\right)^T \nabla_{t}\nabla^2_{SS}\mathcal{L}^\dagger(\bar{\bar{\bm{\theta}}}^t) (\hat{\bm{\theta}}_S-\bm{\theta}_S^*),
\end{align*}
where $\bar{\bar{\bm{\theta}}}^t$ exists between $\bm{\theta}^*$ and $\bar{\bm{\theta}}^t$.
Then,
\begin{align*}
    \left|g_t^\dagger\right|
    &= \left|\left(\bar{\bm{\theta}}_S^t-\bm{\theta}_S^*\right)^T \nabla_{t}\nabla^2_{SS}\mathcal{L}^\dagger(\bar{\bar{\bm{\theta}}}^t) (\hat{\bm{\theta}}_S-\bm{\theta}_S^*)\right| \\
    &\le \left\|\bar{\bm{\theta}}_S^t-\bm{\theta}_S^*\right\|_2 \left\|\nabla_{t}\nabla^2_{SS}\mathcal{L}^\dagger(\bar{\bar{\bm{\theta}}}^t)\right\|_2 \left\|\hat{\bm{\theta}}_S-\bm{\theta}_S^*\right\|_2 \\
    &\le \left\|\nabla_{t}\nabla^2_{SS}\mathcal{L}^\dagger(\bar{\bar{\bm{\theta}}}^t)\right\|_2 \left\|\hat{\bm{\theta}}_S-\bm{\theta}_S^*\right\|^2_2.
\end{align*}
The last inequality holds because $\bar{\bm{\theta}}^t_S$ is between $\bm{\theta}^*_S$ and $\hat{\bm{\theta}}_S$ in a coordinate fashion.
Then, from Proposition \ref{proposition contaminated function boundedness} and Lemma \ref{lemma-3},
\begin{align*}
    \left|g_t^\dagger\right|
    &\le 3(1-\varepsilon_p)\lambda_{3,\max} \times \frac{1600}{(1-\varepsilon_p)^2\lambda_{\min}^2}k\lambda_{n_p^*,n_q^*}^2 \\
    &= \frac{4800\lambda_{3,\max}}{(1-\varepsilon_p)\lambda_{\min}^2}k\lambda_{n_p^*,n_q^*}^2
\end{align*}
holds with probability at least $1-\delta_{\tau}-\delta_{\epsilon}-\delta_{\eta'}$ for $t \in \mathcal{E}$.
Therefore, when 
\begin{align*}
    k\lambda_{n_p^*,n_q^*}\le \frac{(1-\varepsilon_p)\lambda_{\min}^2}{4800\lambda_{3,\max}}\frac{\alpha}{8(2-\alpha/2)},
\end{align*}
we have
\begin{align*}
    \left\|\bm{g}^\dagger\right\|_\infty \le \frac{\alpha\lambda_{n_p^*,n_q^*}}{8(2-\alpha/2)}
\end{align*}
with probability at least $1-\delta_{\tau}-\delta_{\epsilon}-\delta_{\eta'}$.
We define $\delta=\delta_{\tau}=\delta_{\epsilon}=\delta_{\eta'}$.

\section{Experimental Settings in Section \ref{subsection experiment for unboundedness}}
\label{section experimental settings}

In the experiment in Section \ref{subsection experiment for unboundedness}, we used three Gaussian distributions with precisions parameters $\lambda$ set to $1.0$ for the non-active set and $0.8$ and $0.4$ for the active set.
Figure \ref{fig:precision} displays the probability density functions of these Gaussian distributions in one dimension.
Besides, we plot the weight function $w(x)=0.5\exp(-\|x\|^4_4/20)$.

The weight function is designed not to eliminate data from the distribution with $\lambda=1.0$ in the non-active set.
In regions where the probability density function takes large values, the weight function also assign large values.
This setting satisfies Assumption \ref{assumption weight for normal}, ensuring the weight function does not eliminate the majority of the inlier data.

The weight function, however, eliminates data sampled from the tail of the distributions with $\lambda=0.8$ and $0.4$ in the active set.
In unbounded density ratio estimation, data with large values adversely affect the parameter estimation in the objective function.
For example, the density ratio value is calculated as $\exp(\lambda_i x_i^2)\approx 600$ at $x_i=4$ with $\lambda_i=0.4$, which is considerably larger than other data points, such as $\exp(\lambda_i x_i^2)\approx1$ at $x_i=1$.
The weight function eliminates such data to satisfy Assumption \ref{assumption-9}, ensuring that the weighted density ratio remains bounded.

\begin{figure*}[ht]
\vskip 0.2in
\begin{center}
    \centerline{\includegraphics[width=0.8\columnwidth]{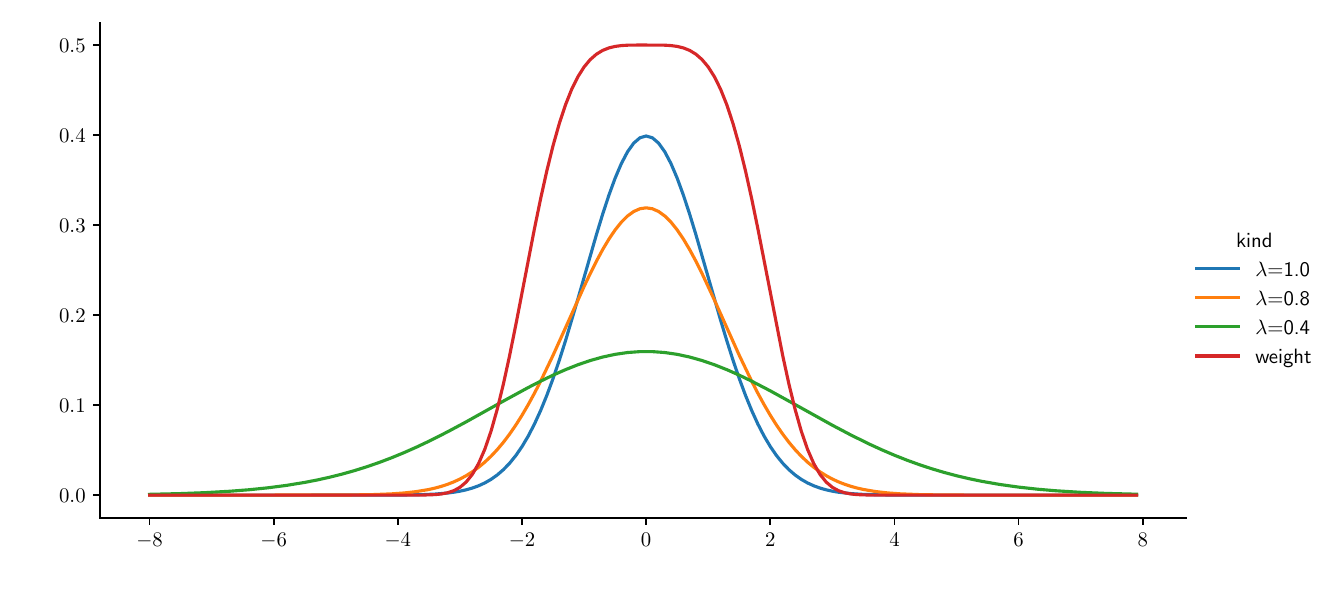}}
    \caption{
        The probability density functions of Gaussian distributions with different precisions and the weight function.
    }
    \label{fig:precision}
\end{center}
\vskip -0.2in
\end{figure*}

\end{document}